\documentclass[twoside,11pt]{article}

\usepackage{jmlr2e}

\usepackage{mathrsfs}

\usepackage[margin=1in]{geometry}
\usepackage{ulem}
\usepackage{amsthm,amsmath,amssymb}
\usepackage{color}
\usepackage[linesnumbered,ruled,vlined]{algorithm2e}

\usepackage{savesym}
\savesymbol{listofalgorithms}

\usepackage{graphicx}
\usepackage{natbib}
\usepackage{epsfig}
\usepackage{algorithmic}
\usepackage{algorithm}
\usepackage{subfig}

\newcommand{\be}{\begin{equation}}
\newcommand{\ee}{\end{equation}}
\newcommand{\bes}{\begin{equation*}}
\newcommand{\ees}{\end{equation*}}
\newcommand{\beqn}{\begin{eqnarray}}
\newcommand{\eeqn}{\end{eqnarray}}
\newcommand{\beqns}{\begin{eqnarray*}}
\newcommand{\eeqns}{\end{eqnarray*}}
\newcommand{\lkr}{\left(}

\newcommand{\rkr}{\right)}
\newcommand{\lfi}{\left\{}
\newcommand{\rfi}{\right\}}

\newcommand{\del}{\delta}

\newcommand{\eps}{\epsilon}

\newcommand{\sig}{\sigma}

\newcommand{\EE}{\ensuremath{{\mathbb E}}}

\newcommand{\real}{{\mathbb R}}

\newcommand{\iter}{\mathrm{iter}}

\newcommand{\vect}{\mbox{vec}}

\newcommand{\Span}{\mbox{Span}}

\newcommand{\diag}{\mbox{diag}}

\newcommand{\Tr}{\mbox{Tr}}

\newcommand{\Skew}{\mathrm{Skew}}
\newcommand{\Kmax}{K_{\max}}
\newcommand{\Kmin}{K_{\min}}

\newtheorem{thm}{Theorem}
\newtheorem{lem}{Lemma}
\newtheorem{cor}{Corollary}
\newtheorem{rem}{Remark}

\newtheorem{prop}{Proposition}

\newcommand{\bbA}{\boldsymbol{A}}
\newcommand{\bbB}{\boldsymbol{B}}
\newcommand{\bbC}{\boldsymbol{C}}
\newcommand{\bL}{\mathbf{L}}
\newcommand{\bfe}{\mathbf{e}}
\newcommand{\bbP}{\boldsymbol{P}}
\newcommand{\bbQ}{\boldsymbol{Q}}
\newcommand{\bbX}{\boldsymbol{X}}
\newcommand{\bbY}{\boldsymbol{Y}}

\newcommand{\bbDelta}{\boldsymbol{\Delta}}

\newcommand{\rank}{\mathrm{rank}}

\newcommand{\bDelta}{\mathbf{\Delta}}

\newcommand{\bu}{\mathbf{u}}
\newcommand{\bv}{\mathbf{v}}

\newcommand{\bx}{\mathbf{x}}

\newcommand{\bA}{\mathbf{A}}
\newcommand{\bB}{\mathbf{B}}
\newcommand{\bC}{\mathbf{C}}
\newcommand{\bD}{\mathbf{D}}

\newcommand{\bI}{\mathbf{I}}

\newcommand{\bK}{\mathbf{K}}

\newcommand{\bO}{\mathbf{O}}
\newcommand{\bP}{\mathbf{P}}
\newcommand{\bQ}{\mathbf{Q}}
\newcommand{\bR}{\mathbf{R}}
\newcommand{\bS}{\mathbf{S}}
\newcommand{\bU}{\mathbf{U}}
\newcommand{\bV}{\mathbf{V}}
\newcommand{\bW}{\mathbf{W}}
\newcommand{\bX}{\mathbf{X}}
\newcommand{\bY}{\mathbf{Y}}
\newcommand{\bZ}{\mathbf{Z}}

\newcommand{\bzero}{\mathbf{0}}
\newcommand{\bone}{\mathbf{1}}

\newcommand{\bLam}{\mbox{\mathversion{bold}$\Lambda$}}

\newcommand{\bTe}{\mbox{\mathversion{bold}$\Theta$}}

\newcommand{\calA}{{\mathcal{A}}}

\newcommand{\calF}{{\mathcal{F}}}

\newcommand{\calM}{{\mathcal M}}

\newcommand{\calP}{{\mathcal{P}}}

\newcommand{\calS}{{\mathcal{S}}}

\newcommand{\calV}{{\cal{V}}}

 \newcommand{\minM}{\displaystyle \min_{m=1, ....M}\ }
 \newcommand{\maxM}{\displaystyle \max_{m=1, ....M}\ }
 
\newcommand {\colred}[1] {\textcolor{red}{#1}}

\definecolor{mycolor1}{rgb}{0.1, 0.5, 0.7}
 
\long\def\ignore#1{}

\newcommand{\reals}{\mathbb{R}}

\SetKwInput{KwInput}{Input}                
\SetKwInput{KwOutput}{Output}              

\firstpageno{1}

\graphicspath{%
    {converted_graphics/}
    {/}
}
\begin{document}

\title{ALMA: Alternating Minimization Algorithm For  Clustering Mixture Multilayer Network }

\author{\name  Xing Fan \email fanxing@knights.ucf.edu
  \\
       \addr Department of Mathematics\\
       University of Central Florida\\
     Orlando, FL 32816, USA
      \AND
\name  Marianna Pensky  \email marianna.pensky@ucf.edu \\
       \addr Department of Mathematics\\
       University of Central Florida\\
     Orlando, FL 32816, USA
     \AND
    \name   Feng Yu  \email yfeng@knights@ucf.edu \\
       \addr Department of Mathematics\\
       University of Central Florida\\
     Orlando, FL 32816, USA
     \AND
    \name  Teng Zhang \footnote{Corresponding Author}
    \email Teng.Zhang@ucf.edu \\
       \addr Department of Mathematics\\
       University of Central Florida\\
     Orlando, FL 32816, USA
          }

\editor{}

\maketitle


\begin{abstract}%

The paper considers a   Mixture Multilayer Stochastic Block Model (MMLSBM), where layers can be partitioned into groups of similar networks, and networks in each group are  equipped with a  distinct
 Stochastic Block Model. The goal is to partition the multilayer network into clusters of similar layers, and to
 identify communities in those layers.   Jing {\it et al.} (2020) introduced the MMLSBM  and developed a clustering methodology, TWIST,  based on regularized tensor decomposition.

The present paper proposes a different technique,  an alternating minimization algorithm (ALMA), that aims at simultaneous recovery of the layer partition, together with estimation  of the matrices of connection probabilities  of the distinct layers. Compared to  TWIST, ALMA achieves   higher accuracy, both theoretically and numerically.
\\
\end{abstract}



\section{Introduction}
\label{sec:introduction}

Stochastic networks arise in many areas of research and applications and are
used, for example, 
to study brain connectivity or gene regulatory mechanisms, 
to monitor cyber and homeland security, and to evaluate and predict social relationships within groups or between groups, such as countries.
\\

While in the early years of the field of stochastic networks, research  mainly focused  on studying a
single network, in recent years the frontier moved to investigation of collection of networks, the
so called {\it multilayer network}, which allows to model relationships between nodes
with respect to various modalities (e.g., relationships between species based on food or space),
or consists of  network data collected from different individuals (e.g., brain networks).
\\

Although there are many different ways of modeling a multilayer network (see, e.g.,
an excellent review article of \cite{10.1093/comnet/cnu016}), in this paper we consider the case where
all layers have the same set of  nodes, and all edges between nodes are drawn within  layers, i.e.,
there are no edges connecting the nodes in different layers. 
 \cite{macdonald2021latent} called this type of networks the {\it multiplex} networks and argued that they appear 
 in a variety of applications. Indeed, consider brain networks of several individuals that are drawn on the basis 
of some imaging modality. The nodes in the networks are associated with brain regions,
and the brain regions are considered to be connected if the signals in those regions exhibit some kind of similarity.   
In this setting, the nodes are the same for each individual network, and there is no connection between 
brain regions of different individuals. For this reason, one  can consider a multiplex
network constituted by brain networks of several individuals, with common nodes but possibly different
community structures in different layers (individuals). It is known that brain disorders are associated with
changes in brain network organizations (see, e.g.,  \cite{Buckner2019TheBD}), and that alterations in the community structure of
the brain have been observed in several neuropsychiatric conditions, including Alzheimer disease (see, e.g.,  \cite{doi:10.1002/hbm.23240}),
 schizophrenia (see, e.g.,  \cite{pub.1037745277}) and epilepsy disease (see, e.g.,  \cite{munsell_2015}). Hence, assessment of the brain
modular organization may provide a key to understanding the relation between aberrant connectivity and
brain disease.
\\

The multiplex networks have been studied  by many authors  who work  in a variety of research fields. 
(see, e.g.,  \cite{JMLR:v18:16-391}, \cite{han2018multiresolution},
\cite{Aleta_2019}, \cite{Kao_2017} among others).
In this paper, we consider a multilayer network where all layers are equipped with the Stochastic Block Models (SBM).
In this case, the   problems of interest include finding groups of layers that are similar in some sense, finding the
communities in those groups of layers and estimation of the tensor of connection probabilities.
While the scientific community attacked   all three of those problems, often in a somewhat ad-hoc manner
(see e.g., \cite{doi:10.1098/rsos.171747}, \cite{Kao_2017}, \cite{mercado2018power} among others),
the theoretically inclined papers in the field of statistics   mainly been  investigated the case where
communities persist throughout all layers of the network. This includes studying the so called ``checker board model''  in
\cite{JMLR:v21:18-155},  where the matrices of block probabilities take only finite number of values,
and communities persist in all layers. The tensor block models of \cite{NEURIPS2019_9be40cee} and  \cite{han2021exact} 
belong to the same category. In recent years, statistics publications extended this type of research   
to the  case, where community structure persists but the matrix of probabilities of connections can take
arbitrary values (see, e.g., \cite{bhattacharyya2020general}, \cite{paul2020}, \cite{10.1093/biomet/asz068},
\cite{lei2020tail}, \cite{paul2016}  and references therein). The authors studied  precision of community detection 
and provided comparison between various techniques that can be employed in this case. 
\\

In many practical situations, however, the assumption of common community structures in all layers of
the network may not be justified. Indeed, as we have stated above, some psychiatric or neurological conditions 
may be due to the alteration in the brain networks community structures rather than modifications
in the strength of connections. For this reason, it is of interest to study a multiplex network 
with distinct community structures in groups of layers. 
Recently,  \cite{jing2020community}  investigated the
so called ``{\bf M}ixture {\bf M}ulti{\bf L}ayer {\bf S}tochastic
{\bf B}lock {\bf M}odel'' (MMLSBM),
where there are $L$ layers can be partition into $M$ different types, with $M$ being a small number.
In MMLSBM,  each   class $m$ of   layers is equipped with its own  community structure and a
distinct matrix of connection probabilities $\bbB_m$, $m=1,...,M$. The methodology of \cite{jing2020community}
is  based on a regularized  tensor decomposition, where all tensor dimensions are treated in the same way.
The theory is developed under the assumption that the number of layers does not exceed the number of nodes.
Note that the latter may not be true, for example, for brain networks, where the number of nodes 
is in hundreds (and is fixed) while the number of individuals, whose brain images are available, 
can grow indefinitely.
\\

In this paper, we  also consider  the MMLSBM   and suggest a new algorithm for the layer partition and local communities recovery.   
While  the methodology of \cite{jing2020community} is  based on a regularized  tensor decomposition, our technique
is centered around finding the groups of layers. Indeed, the ``naive'' approach to the problem would be 
to vectorize all adjacency matrices and cluster them using the k-means procedure. The major difference between 
our paper and \cite{jing2020community}  is that we recognize that it is advantageous to treat   within-layer and between-layer dimensions 
of the adjacency tensor in a different manner.
Specifically,  we propose a novel {\bf AL}ternating {\bf M}inimization {\bf A}lgorithm  (ALMA) which utilizes the fact that,
for each layer of the network, the matrix of probabilities of connections can be approximated by
a low-rank matrix. As a result, for the MMLSBM, our algorithm  consistently recovers the layer labels and the memberships of nodes.
\\

The present paper makes several contributions. 
First, it introduces the idea that the key to the inference in the MMLSBM is identification of the groups of layers: 
as soon as networks in each of $M$ layers are discovered, the communities can be found by the spectral algorithm of \cite{lei2015}, applied to
the averages of the adjacency matrices.  In addition, it uses the information that all layers are approximately low-rank. 
In comparison, the algorithm  of \cite{jing2020community} only uses the information that the underlying tensor is approximately low-rank, 
which ignores the low-rankness within each layer. Due to this idea,  as it follows from our theoretical analysis, ALMA achieves
higher accuracy in the between-layer clustering. Also, as our numerical studies show, the latter 
leads to  smaller between-layer and within-layer clustering errors,   than for the algorithm  of \cite{jing2020community}.
In addition, unlike  the technique in \cite{jing2020community}, ALMA does not require the assumption 
that the number of layers in the network is smaller than the number of nodes.
\\

In this paper, we are not interested in the case of $M=1$, where   communities are the same in all layers. 
Indeed, if one know that $M=1$, then, under the assumption that there are only $M=1$ types of matrices 
of connection probabilities, one can just find communities by spectral clustering after averaging.
For this reason, one should not apply ALMA to the ``checker board'' or tensor block model,
and ALMA should not be compared with techniques designed for this type of models.
%
\\

Also, we assume that both the number of distinct layers $M$ and the number of communities in each group 
of layers are fixed and  known in advance. While this is usually not true in practice, this is a very common 
assumption for theoretical investigations. When the algorithm is used in a real data setting, one needs to
obtain solutions for several different values of $M$ and then choose the one that agrees with data.
Since the probability tensor of the  MMLSBM has sets of identical layers, we can borrow the idea from 
the problem of determining the number of clusters in a data set, when the  $K$-means algorithm is used.
One of the most popular heuristic methods is the so called ``elbow method''. In our setting, we can run 
the algorithm with an increasing number of clusters  $M$, and plot an error measure of the model 
as a function of $M$. This function would decrease as $M$ increases since models with larger $M$ explain more variations. 
Then, the elbow methods evaluate the curve of the function and find the ``elbow of the curve'', i.e., the point where the 
function is no longer decreasing rapidly, as the  number of distinct layers $M$ grow
(see, e.g., \cite{doi.org/10.1111/1467-9868.00293}, \cite{Zhang2012,Le2015}). 
Other methods of  choosing $M$ include  cross-validation~\citep{10.1093/biomet/asq061} and information 
criterion~\citep{10.1007/978-3-540-45080-1_27}. 
After the number of groups of layers has been  determined by one of the above mentioned techniques
and the between-layers clustering  has been implemented, one can identify  the number of communities within 
each group of layers using common techniques employed in the Stochastic Block Models~(SBMs)~\citep{Zhang2012,Le2015,10.1214/19-EJS1533}.
\\

Note that dynamic network models can be viewed as a particular case of the multilayer network model where
there are no edges connecting the nodes in different layers.  The difference between those models and the multilayer network
is that, in a dynamic network, the layers are ordered according to time instances, while in a multilayer network the enumeration
of layers is completely arbitrary. That is why, although there is a multitude of papers that study the change point detection
in the dynamic SBMs (see, e.g., \cite{bhattacharjee2018change}, \cite{gangrade2018testing} and \cite{wang2017optimal} among others),
the techniques and error bounds in those papers are not applicable in the situation of the MMLSBM.
\\

The rest of the paper is organized as follows.
Section~\ref{sec:model} describes the MMLSBM and presents the necessary concepts and notations.
Section~\ref{sec:methodology} introduces the Alternating Minimization Algorithm  (ALMA).
Section~\ref{sec:Theor_guarantees}  provides theoretical guarantees for between-layer and within-layer clustering errors.
Specifically, the section starts with  Section~\ref{sec:convergence_true} that investigates the situation where ALMA
is applied to the true probability tensor. Based on the results of this analysis, Section~\ref{sec:assumptions}  provides assumptions which,
as it is confirmed in Section~\ref{sec:convergence}, guarantee
 convergence of our algorithm. Finally, Section~\ref{sec:layer_clust}
 produces   upper bounds for    between-layer   and   within-layer
 clustering errors. Section~\ref{sec:Theor_guarantees}  is concluded
 by a discussion of various aspects of ALMA in
 Section~\ref{sec:discuss_theor}.
Section~\ref{sec:comparison} brings up theoretical and numerical comparisons between the ALMA and the TWIST algorithm, proposed in
\cite{jing2020community}.
The proofs of all statements in the paper are deferred to Section~\ref{sec:Proofs}, Appendix.


\section{Model framework}\label{sec:model}
\setcounter{equation}{0}

This work considers an $L$-layer network on the same set of $n$ vertices $\calV=\{1,\cdots,n\}$.
For any $1\leq l \leq L$, the observed data is the adjacency matrix $\bA_l\in\reals^{n\times n}$
of the $l$-th network, where $\bA_l(i,j)=\bA_l (j,i) = 1$ if a connection between nodes $i$
and $j$ is observed at the $l$-th network, and  $\bA_l (i,j)=\bA_l (j,i)  = 0$  otherwise.
Assume that for all $1 \leq i < j \leq n$ and $1\leq l\leq L$, $\bA_l (i,j)$ are the Bernoulli
random variables with $\Pr(\bA_l(i,j)=1)=\bP_{*l}(i,j)$, and they are independent
 with each other. The probability matrices $\{\bP_{*l}\}_{l=1}^L$ take $M$
different values ($M<L$), that is, there exists a partition of $[L] = \{1,\cdots, L\}=\cup_{m=1}^{M}\calS_m$ such that
$\bP_{*l} = \tilde{\bQ}_{*m}$  for all $l\in\calS_m$.
This means that there exists a clustering function $z: [L] \to [M]$ such that $z(l) = m$ if the $l$-th
network is of the type $m$, or, equivalently,  $l\in\calS_{m}$. Consider a set $\calF_{L,M}$ of the {\it clustering }  matrices
\bes
\calF_{L,M} = \lfi \bZ \in \{0,1\}^{L \times M},\quad \bZ  \bone = \bone, \quad \bZ^T \bone \neq \bzero \rfi,
\ees
and $\bZ\in \calF_{L,M}$ such that $\bZ (l,m) = 1$ if $l\in\calS_{m}$ and $\bZ (l,m) = 0$  otherwise, and matrix $\bZ$ does not have zero columns.
It is easy to see that matrix $\bZ^T \bZ$ is diagonal, and $\bW_* = \bZ (\bZ^T \bZ)^{-1/2}$ satisfies
$\bW_*^T\bW_*=\bI$. Here, $\bW_{\star}(l,m) = L_m^{-1/2}$ if $l\in\calS_{m}$ 
and $\bZ (l,m) = 0$ otherwise, where $L_m$ is the number of networks in the layer of type $m$,
$m=1, ...,M$.

Furthermore, we assume that each network can be described by a  Stochastic Block Model (SBM). Specifically, we assume that,
for each $m$  and any $l\in\calS_m$, $\bP_{*l} = \tilde{\bQ}_{*m}$ where
$\tilde{\bQ}_{*m}$ is generated as follows: the nodes $\calV$ are grouped into $K_m$ classes
$G_{m,1}, \cdots, G_{m,K_m}$, and the probability of a connection $\bP_{*l} (i,j)$
is entirely  determined by the  groups to which the nodes $i$ and $j$
belong at $l$. In particular,  if $i \in G_{m,k}$ and $j \in G_{m,k'}$, then
$\bP_{*l} (i,j) = \bB_m(k,k')$, where $\bB_m\in\reals^{K_m\times K_m}$ is
the {\it connectivity matrix} with $\bB_m (k,k') = \bB_m (k',k)$.
In this case, one has 
\be \label{maineq}
\bP_{*l} = \bTe_m \bB_m \bTe_m^T, \quad m = z(l), \quad \bTe_m \in \calF_{n, K_m},
\ee
where $\bTe_m (i,k) =1$ if and only if    node $i$ belongs to the class $G_{m,k}$ and is zero otherwise.

Denote   $\bQ_{*m} = \sqrt{|\calS_m|}\tilde{\bQ}_{*m}$. Denote the three-way tensors with
the $l$-th layer $\bA_l$ and $\bP_{*l}$ by, respectively, $\bbA, \bbP_*\in \reals^{L\times n\times n}$,
and the three-way tensor with the $m$-th layer $\bQ_{*m}$ by $\bbQ_*\in\reals^{M\times n\times n}$.

The objective of this work is to partition  the multilayer network $\calA$ into $M$ similar layers (between-layer clustering)
and, furthermore, for each of these sets of layers, to recover communities $G_{m,1}, \cdots, G_{m,K_m}$, $m = 1, \ldots, M$
(within layer clustering).
Specifically, we focus on  the setting where $M$ and $\{K_m\}_{m=1}^M$ are fixed or grow  slowly, while   $n$ and  $L$
tend to infinity,  since usually   networks are large but have
relatively few similar groups of layers, and the number of communities is also usually small
compared to the number of nodes.



\subsection{Notations}


For any matrix $\bX \in \reals^{n_1 \times n_2}$, denote the Frobenius and the operator norm of any matrix $\bX$ by
$\|\bX\|_F$  and $\|\bbX\|$, respectively, and its $r$-th largest
singular value by $\sig_r(\bX)$. Let $\vect(\bX) \in \reals^{n_1 n_2}$  be  vectorization of matrix $\bX$
obtained by sequentially stacking columns of matrix $\bX$. 
Denote the projection operator onto the nearest orthogonal matrix by   $\Pi_o$:
 \be \label{eq:Pi_o}
 \Pi_o(\bX)= \bX (\bX^T \bX)^{-1/2}.
 \ee
If $n_1\geq n_2$, then $\Pi_o(\bX)$ is an orthogonal matrix that has the same column space as $\bX$.
Specifically, if  the singular value decomposition of $\bX$ is $\bX=\bU\Sigma\bV^T$, 
where $\bU\in\reals^{n_1\times n_2}$ and $\Sigma,\bV\in\reals^{n_2\times n_2}$, then  $\Pi_o(\bX)=\bU\bV^T$. 
\\


For any tensor $\bbX\in\reals^{n_1\times n_2\times n_3}$, its mode 1 matricization 
$\calM_1(\bbX)\in\reals^{n_1\times n_2n_3}$ is a matrix such that $[\calM_1(\bbX)](l,:)=\mathrm{vec}(\bbX(l,:,:))$.
For any tensor $\bbX\in\reals^{n_1\times n_2\times n_3}$ and a matrix $\bA\in\reals^{m\times n_1}$,
their mode-1  product $\bbX\times_1\bA$ is a tensor in $\reals^{m\times n_2\times n_3}$ defined by
\bes
[\bbX\times_1\bA](j,i_2,i_3)=\sum_{i_1=1}^{n_1}  \bbX(i_1,i_2,i_3)  \bA(i_1,j), \quad j=1, ...,m.
\ees
In this product, every mode-1 fiber of tensor $\bbX$ is multiplied by matrix $\bA$:
\be \label{eq:mode1}
\bY = \bX \times_1 \bA \Longleftrightarrow  \widetilde{\bY} = \bA \widetilde{\bX}, \quad \widetilde{\bY} = \calM_1(\bY), \
\widetilde{\bX} = \calM_1(\bX)
\ee 
If $\bbX\in\reals^{n \times n_2\times n_3}$ and $\bbY\in\reals^{m\times n_2\times n_3}$ are two tensors,  
their mode-(2,3) product  denoted by $\bbX\times_{2,3}\bbY$, is a matrix in $\reals^{n \times m}$ with elements
$i_1 = 1, ...,n$,   $i_2 = 1, ...,m$
$$
[\bbX\times_{2,3}\bbY](i_1,i_2)=\sum_{j_2=1}^{n_2}\sum_{j_3=1}^{n_3} \bbX(i_1,j_2,j_3)\bbY(i_2,j_2,j_3) = 
\Tr[X(i_1, :, :)Y (i_2, :, :)^T]
$$
The Frobenius norm $\|\bbX\|_F $ and the largest singular value $\sigma_1(\bbX)$ of
a tensor $\bbX\in\reals^{n_1\times n_2\times n_3}$ are defined by
\begin{align*}
\|\bbX\|_F & =\sqrt{\sum_{i_1,i_2,i_3=1}^{n_1,n_2,n_3}\bbX(i_1,i_2,i_3)^2},\\
\sigma_1(\bbX)=\|\bbX\|& =\max_{\bu_i\in\reals^{n_i}, \|\bu_i\|=1, 1\leq i\leq 3}\bbX\times_1\bu_1\times_2\bu_2\times_3\bu_3.
\end{align*}
Operations above obey the following properties:
\\
1.\ For $\bX  \in \reals^{n_1 \times n_2\times n_3}$, $\bA \in  \reals^{m \times n_1}$, $\bY \in \reals^{n \times n_2\times n_3}$
one has $(\bX \times_1 \bA) \times_{2,3} \bY = \bA^T (\bX \times_{2,3} \bY)$ 
\\
2. \  If $\bW$ is such that $\bW^T \bW = \bI$, then $\| \bQ \times_1 \bW^T \|^2_F = \| \bQ \|^2_F$ 
\\

For a more comprehensive tutorial for tensor algebra, please, see  
the review article of \cite{Kolda09tensordecompositions}.

Next, we introduce some notations  that will be used later in the paper.
Let
$\bK = (K_1, \ldots, K_M),$ and denote
\be \label{eq:K_def}
\Kmax= \max_{m=1,\cdots,M}K_m,\quad \Kmin= \min_{m=1,\cdots,M}K_m,\quad
\dot{K}= \sum_{m=1}^MK_m.
\ee
Denote the size of the smallest cluster in all networks by $g_{\min}$, i.e.,
$\displaystyle{g_{\min}=\min_{\stackrel{1\leq m\leq M}{1\leq k\leq K_m}}\, |G_{m,k}|}$.
Consider the SVDs of matrices $\bbQ_*(m,:,:)$ and the matrices  $\Pi_{\bU_m^\perp}$
orthogonal to the linear spaces of their eigenvectors:
\be \label{eq:svd_Qm}
\bbQ_*(m,:,:) = \bU_m  \bLam_m \bU_m^T, \quad \Pi_{\bU_m^\perp}=\bI-\bU_m\bU_m^T
\ee 
Since $\rank(\bbQ_*(m,:,:))=K_m$,  $\bU_m\in\reals^{n\times K_m}$ is an  
orthogonal matrix that has the same column space as $\bbQ_*(m,:,:)$.
Note that we somewhat abuse notations here: $\Pi_{\bU_m^\perp}$ is a matrix 
and also an operator, so that, for any matrix $\bX$,   $\Pi_{\bU_m^\perp} (\bX)$
is a projection of the matrix $\bX$ on the linear space orthogonal to the column space  of matrix $\bU_m$.

Now, we introduce  operators that will be used later in the paper.
For any tensor $\bbX\in\reals^{M\times n\times n}$, define a projector
$\Pi_\bK: \reals^{M\times n\times n}\rightarrow \reals^{M\times n\times n}$ by
\be \label{eq:Pi_bK}
[\Pi_\bK(\bbX)](m,:,:)=\Pi_{K_m}\big(\bbX(m,:,:)\big), \ m=1, \ldots, M,
 \ee
where $\Pi_{K_m}: \reals^{n\times n}\rightarrow \reals^{n\times n}$ is the projection onto the nearest rank $K_m$ matrix.
Consider  operator $\Pi_{T,\bK}: \reals^{M \times n \times n} \to \reals^{M \times n \times n}$ 
defined as
\begin{align} \label{eq:Pi_T_bK}
[\Pi_{T,\bK}(\bbX)](m,:,:) & =\bbX(m,:,:)-\Pi_{\bU_m^\perp}\bbX(m,:,:)\Pi_{\bU_m^\perp},
\end{align}
where $\Pi_{\bU_m^\perp}$ is defined in \eqref{eq:svd_Qm}.
In addition, let $\Pi_{T,K_m}: \reals^{L\times n\times n}\rightarrow\reals^{L\times n\times n}$ be 
the projection onto the subspace spanned by $\bU_m$ for each ``slice'' of the tensor, i.e.,
\begin{align*}
&[\Pi_{T,K_m}(\bbX)](m',:,:)=\bbX(m,:,:)-\Pi_{\bU_m^\perp}\bbX(m,:,:)\Pi_{\bU_m^\perp},
\quad m'=1, \ldots, M, \nonumber
\end{align*}
%

 
\section{Alternating Minimization Algorithm  (ALMA)}
\label{sec:methodology}

As we observe the multi-layer networks $\{\bbA_l\}_{l=1}^L$, our objectives are
\begin{itemize}
\item {\bf Between-layer clustering:} recover the network classes $\calS_1, \cdots, \calS_M$ such that $[L]=\cup_{m=1}^M\calS_m$.
\item {\bf Within-layer clustering:} recover the community structures for each network class, i.e.,
for any $m\in [M]$, find a partition of the vertices $G_{m,1}, \cdots, G_{m,K_m}$.
\end{itemize}
To achieve these goals, we   start with the estimation of $\bbQ_*\in\reals^{M\times n\times n}$ and
$\bW_*\in\reals^{L\times M}$ based on  $\bbA$. Then, the between-layer clustering can be carried out by
applying $K$-means algorithm  to the rows of  the estimator $\widehat{\bW}$ of  matrix  $\bW_*$. Subsequently,  the within-layer clustering  of
the $m$-th group of networks can be obtained by analyzing  estimators $\hat{\bbQ}_*(m,:,:)$ of $\bbQ_*(m,:,:)$ for every $m\in [M]$.
\\

In order to estimate $\bbQ_*$ and $\bW_*$, note that the tensor $\bbA$ can be considered as a noisy observation of $\bbP_*$,
since $\EE(\bbA)=\bbP_*$, where $\bbP_*=\bbQ_*\times_1\bW_*^T$, and $\bW_*$ is an orthogonal matrix by definition.
For this reason, we propose to find $\bbQ_*$ and $\bW_*$ by solving the following  optimization problem
\begin{align}\label{eq:objective}
& \underset{\bbQ,\bW}{\text{argmin}}\
\|\bbA-\bbQ\times_1\bW^T \|_F\\
&\text{s.t. $\bbQ\in \reals^{M \times n\times n},\bW\in\reals^{L\times M}$, $\bW^T\bW=\bI$, $\rank(\bbQ(m,:,:))\leq K_m$
for all $1\leq m\leq M$.}\nonumber
\end{align}
We solve \eqref{eq:objective} by alternatively minimizing the objective function in \eqref{eq:objective} over $\bQ$ and $\bW$.

When $\bW$ is fixed, the best approximation to $\bbQ$ is given by  $\bbQ= \Pi_{\bK}(\bbA\times_1\bW)$. 
Indeed, by equation~\eqref{eq:mode1}, one has $\|\bbA-\bbQ\times_1\bW\|_F = \|\calM_1(\bA) - \bW \calM_1(\bbQ)\|$. Hence,  
minimization of the last expression over $\bbQ$ yields $\calM_1(\bbQ) = \bW^T \calM_1(\bA)$ which, 
by \eqref{eq:mode1},  leads to $\bbQ = \bA \times_1 \bW$. The latter, due to the rank restrictions, is approximated
by the closest rank projection  $\Pi_{\bK}(\bbA\times_1\bW)$.

When $\bbQ$ is fixed, the problem of minimizing of   $\|\bbA-\bbQ\times_1\bW^T \|_F$   over $\bW$
under the assumption that $\bW^T\bW=\bI$,  is called the orthogonal Procrustes problem   (see, e.g., \cite{oro2736}),
and it has an explicit solution  $\bW = \Pi_o(\bbA\times_{2,3} \bbQ)$,
where $\Pi_o(\bX)$ is defined in \eqref{eq:Pi_o}. 
Combining the two steps, we summarize this  alternating minimization procedure
in Algorithm~\ref{alg:update}.
\\

%
%
\begin{algorithm} [tb]
\caption{{Alternating Minimization Algorithm (ALMA)}}  
\label{alg:update}
\begin{flushleft}
{\bf Input:}  {Adjacency tensor $\bbA\in\reals^{L\times n\times n}$; number of different types of networks $M$;
$\{K_m\}_{m=1}^M$; Initialization clustering matrix  $\bW^{(1)}\in\reals^{L\times M}$ such that 
$(\bW^{(1)})^T\, \bW^{(1)}=\bI$}   \\
{\bf Output:} A clustering matrix $\widehat{\bW}\in\reals^{L\times M}$ such that
$\widehat{\bW}^T\widehat{\bW}=\bI$, and a tensor $\widehat{\bbQ}\in\reals^{M\times n\times n}$.\\
{\bf Steps:}\\
{\bf 1:}  Set  $\iter=1$.\\
{\bf 2:}  Let $\bbQ^{(\iter+1)}=\Pi_{\bK}(\bbA\times_1  \bW^{(\iter)})$, where $\Pi_\bK$ is a projector defined  in \eqref{eq:Pi_bK}.  \\
{\bf 3:} Let $\bW^{(\iter+1)}=\Pi_o(\bbA\times_{2,3} \bbQ^{(\iter+1)})$  
where $\Pi_o(\bX)$ is defined in \eqref{eq:Pi_o}.\\
{\bf 4:} Set $\iter=\iter+1$.\\
{\bf 5:} Repeat steps 2-4 until $\|\bW^{(\iter)}-\bW^{(\iter-1)}\|_F \leq \eps_{n,L}$ where
$\eps_{n,L}$ is a pre-specified threshold, or the number of iterations exceeds the upper limit: $\iter > N^{(iter)}$.\\
{\bf 6:} Set $\widehat{\bW}= \bW^{(\iter-1)}$, $\widehat{\bbQ}= \bbQ^{(\iter-1)}$.
\end{flushleft}
\end{algorithm}

After obtaining $\widehat{\bW}$ and  $\widehat{\bbQ}$, we recover the groups of similar networks $\calS_1, \cdots, \calS_M$
by  clustering the rows of $\widehat{\bW}$ into $M$ groups using the $(1 + \eps)$ approximate $K$-means algorithm.
Finally, for clustering the nodes in each type of networks,
we apply spectral clustering with $\widehat{\bbQ} (m,:,:)$ being treated as the affinity matrix.
Specifically, we first find the orthogonal matrix of size $n\times K_m$ whose columns are the top $K_m$
eigenvectors of $\widehat{\bbQ}(m,:,:)$, and then cluster its rows into $K_m$ groups using the $(1 + \eps)$ approximate $K$-means.
There exist efficient algorithms for solving  the $(1 + \eps)$ approximate $K$-means problem, see, e.g., \cite{1366265}.


\section{Theoretical guarantees}
\label{sec:Theor_guarantees}

\subsection{Convergence of the iterative algorithm for the true probability tensor}
\label{sec:convergence_true}

The purpose of this section is to explain how Algorithm~\ref{alg:update} works.
Indeed, in order this algorithm delivers acceptable solutions   when it is applied the adjacency tensor $\bA$,
it should  guarantee  convergence when  $\bA$ is replaced by $\bP_*$, 
and one starts from an arbitrary matrix $\bW^{(1)}$.
In this case, the associated optimization problem becomes
\begin{align}\label{eq:objective1}
& \underset{\bbQ,\bW}{\text{argmin}}\
\|\bbP_*-\bbQ\times_1\bW^T\|_F\\
&\text{s.t. $\bbQ\in \reals^{M \times n\times n},\bW\in\reals^{L\times M}$, $\bW^T\bW=\bI$, $\rank(\bbQ(m,:,:))\leq K_m$,
for all $1\leq m\leq M$}. \nonumber
\end{align}
Then, Algorithm~\ref{alg:update} yields
\begin{equation}\label{eq:update1}
\text{$\bbQ^{(\iter)}=\Pi_{\bK}(\bbP_*\times_1\bW^{(\iter-1)})$,\ \ \  $\bW^{(\iter)}=\bW_*\Pi_o(\bbQ_*\times_{2,3}\bbQ^{(\iter)})$},
\end{equation}
where the latter formula is obtained using $\bbP_*=\bbQ_*\times_1\bW_*^T$.
Hence,
\[
\bW^{(\iter)}=\Pi_o(\bbP_*\times_{2,3}\bbQ^{(\iter)})=\Pi_o(\bW_* (\bbQ_*\times_{2,3}\bbQ^{(\iter)}))=
\bW_*\Pi_o(\bbQ_*\times_{2,3}\bbQ^{(\iter)}).
\]

As a result, we can reformulate problem \eqref{eq:objective1} by adding an assumption that $\bW=\bW_*\bV$
for some $\bV\in\reals^{M\times M}$. Then \eqref{eq:objective1} is simplified to
\begin{align}\label{eq:objective2}
& \underset{\bbQ,\bV}{\text{argmin}}\
\|\bbQ_*-\bbQ\times_1\bV^T\|_F\\
&\text{s.t. $\bbQ\in \reals^{M \times n\times n},\bV\in\reals^{M\times M}$, $\bV^T\bV=\bI$, $\rank(\bbQ(m,:,:))\leq K_m$
for all $1\leq m\leq M$,}\nonumber
\end{align}
and the iterative relations \eqref{eq:update1} become
\begin{equation}\label{eq:update2}
\text{$\bbQ^{(\iter)}=\Pi_{\bK}(\bbQ_*\times_1\bV^{(\iter-1)})$,\ \ \  $\bV^{(\iter)}=\Pi_o(\bbQ_*\times_{2,3}\bbQ^{(\iter)})$},
\end{equation}
where the first equation follows from the fact that 
$$
\bbP_*\times_1 \bW_*\bV^{(\iter-1)}=(\bbP_*\times_1 \bW_*)\times_1\bV^{(\iter-1)}=\bbQ_*\times_1 \bV^{(\iter-1)}.
$$
The latter implies that, for the sets $\calS_1$ and $\calS_2$ in $\reals^{M\times n\times n}$ defined by
\begin{align*}
\calS_1 & =\{\bbQ:  \rank(\bbQ(m,:,:))\leq K_m\  \mbox{for all}\  1\leq m\leq M\},\\
\calS_2 & =\{\bbQ = \bbQ_*\times_1\bV: \bV\in\reals^{M\times M}, \bV^T\bV=\bI\},
\end{align*}
 $\bbQ^{(\iter)}$ is the nearest point on $\calS_1$ to $\bbQ_*\times_1\bV^{(\iter-1)}$,
and $\bbQ_*\times_{2,3}\bV^{(\iter)}$ is the nearest point on $\calS_2$ to $\bbQ^{(\iter)}$.
Hence,  the update formula \eqref{eq:update2} can be viewed as an alternating projection between $\calS_1$ and $\calS_2$.

Denote the tangent planes to the sets $\calS_1$ and $\calS_2$ at $\bbQ_*$ by $L_1$ and $L_2$, respectively.
Then, the explicit formulas for $L_1$ and $L_2$ are given by
\begin{align} \label{eq:L1L2}
L_1 & =\{\bbQ\in\reals^{M\times n\times n}: \Pi_{\bU_m^\perp}\bbQ(m,:,:)\Pi_{\bU_m^\perp}=\mathbf{0}, \quad m=1, \ldots, M,   \}\\
L_2 & =\{\bbQ = \bbQ_*\times_1\bX: \bX\in \mathrm{Skew}_M\},  \nonumber
\end{align}
where $\mathrm{Skew}_M$  represents the set of skew-symmetric matrices of size $M\times M$. The intuition, the formal definition of tangent space, and the derivations of $L_1$ and $L_2$ are deferred to Section~\ref{sec:tangent}.

Hence,  the ``alternating projection'' viewpoint of \eqref{eq:update2} reveals that it is approximately
an alternating projection procedure between the subspaces $L_1$ and $L_2$. Since the projections onto
$L_1$ and $L_2$ are linear operators, the convergence rate of this alternating projection method
can be described by the operator norm of the composite operator:
\be \label{eq:kappa_H}
\kappa_H=\max_{\bbX\in\reals^{M\times n\times n}}\frac{\|P_{L_2}P_{L_1}\bbX\|_F}{\|\bbX\|_F}.
\ee
Since $P_{L_1}$ and $P_{L_2}$ are projection operators, one has $\|P_{L_1}(\bbX)\|_F\leq \|\bbX\|_F$,
 $\|P_{L_2}(\bbX)\|_F\leq \|\bbX\|_F$ and, therefore, $\kappa_H\leq 1$.

Note that $\kappa_H$ can be expressed via the smallest principal angle $\theta$  between planes $L_1$ and $L_2$: $\kappa_H=\cos(\theta)$.
In particular, $\kappa_H=0$ if $L_1$ and $L_2$ are perpendicular to each other, and $\kappa_H=1$ if intersection  $L_1\cap L_2$ is nontrivial.
As an  example, when $M=2$, one has $\dim(L_2)=1$, and $\kappa_H$ can be  explicitly written as
\[
\kappa_H=\sqrt{\frac{{\|\bbQ_*(1,:,:)-\Pi_{\bU_2^\perp}\bbQ_*(1,:,:)\Pi_{\bU_2^\perp}\|_F^2}
+\|\bbQ_*(2,:,:)-\Pi_{\bU_1^\perp}\bbQ_*(2,:,:)\Pi_{\bU_1^\perp}\|_F^2}{{\|\bbQ_*(1,:,:)\|_F^2+\|\bbQ_*(2,:,:)\|_F^2}}}.
\]
Since the algorithm in \eqref{eq:update2} is approximately
an alternating projection procedure between the subspaces $L_1$ and $L_2$, it converges faster for smaller values
of $\kappa_H$. Note that $\kappa_H=1$ if and only if $P_{L_2}P_{L_1}\bbX=\bbX$ for some $\bbX$,
i.e., when there is a nontrivial intersection  between planes $L_1$ and $L_2$.


\subsection{Assumptions }
\label{sec:assumptions}

In order to guarantee   linear convergence
of Algorithm~\ref{alg:update} when it is applied to the true probability tensor $\bP_*$,
we make the following assumption:\\

\noindent
\textbf{(A1).} The subspaces $L_1$ and $L_2$, defined in \eqref{eq:L1L2},  have only trivial intersection at the origin.
\\

\noindent
While Assumption \textbf{(A1)} is somewhat complicated, it is actually not very restrictive.
Specifically, the statement below provides two very simple sufficient conditions that guarantee  Assumption~\textbf{(A1)}.  In particular, Assumption \textbf{(A1(b))} holds if each clustering pattern is not obtained by  mixing other clustering patterns via combining or intersecting the clusters. For example, Assumption \textbf{(A1(b))} holds with high probability when the $M$ clustering patterns are drawn uniformly at random.

\begin{lem}\label{lem:assum_A1}
Let at least one of the following conditions hold:
\\

\noindent
{\rm {\bf(A1(a))}.} For every $1\leq m\leq M$, the sets of $(M-1)$ matrices
\bes
\lfi \Pi_{\bU_m^\perp}\bbQ_*(m',:,:)\Pi_{\bU_m^\perp}, m'\neq m, 1\leq m'\leq M  \rfi
\ees
are linearly independent.
That is, the vectorized versions of those matrices, a matrix of size $n^2\times M$ with the m-th column given by
$\vect(\Pi_{\bU_m^\perp}\bbQ_*(m',:,:)\Pi_{\bU_m^\perp})$, has rank $(M-1)$.
\\

\noindent
{\rm {\bf(A1(b))}.} For all $1\leq m\leq M$, one has
$$
\Span(\bTe_m)\not\in\bigoplus_{1\leq m'\leq M,m'\neq m}\Span(\bTe_{m'}),
$$
where $\bTe_m\in\reals^{n\times K_m}$ is the membership matrix for the $m$-th network as defined in \eqref{maineq}
and $\oplus$ stands for the direct sum of subspaces.
\\
\end{lem}

\noindent
The proof that these conditions are sufficient is presented in Section~\ref{sec:auxillary_proof}.
While conditions in Lemma~\ref{lem:assum_A1} are sufficient, they are not necessary.
A more detailed discussion of assumption  \textbf{(A1)} is deferred to  Section~\ref{sec:A1_discussion}.
\\

In addition,   we impose  few other natural assumptions as follows:\\

\noindent
\textbf{(A2).} There exist absolute constants $c_0>0$ such that $K_{\max}M\leq c_0\dot{K}$,
where $\dot{K}$ is defined in \eqref{eq:K_def}.
\\

\noindent
\textbf{(A3).} The layers in the network, as well as local communities in each network are balanced, i.e.,
there exist absolute constants $c_1,c_2, c_3, c_4$ such that
\begin{equation}\label{eq:assumptionA3}
c_1\, \frac{L}{M} \leq L_m \leq c_2\, \frac{L}{M}, \quad c_3\, \frac{n}{K_m} \leq |G_{m,k}| \leq c_4\, \frac{n}{K_m} \quad \mbox{for  any}\ \ 
1\leq m\leq M,\ 1\leq k\leq K_m,
\end{equation}
where $|G_{m,k}|$ is the size of the $k$-th community in cluster $m$.
\\

\noindent
\textbf{(A4).} There exist matrices     $\bB^0_m\in\reals^{K_m\times K_m}$, $m=1, \ldots, M$, such that
$\bB_m = p_{\max} \bB^0_m$, where $p_{\max}\in (0,1]$ controls the overall network sparsity 
 and the matrices $\bB^0_m$ are such that, for all $m=1, ...,M$, 
there exists some absolute constants $b_1>0$  and $b_2>0$ such that
\begin{equation}\label{eq:assumptionA4}\sigma_{K_m}(\bB_m^0)\geq b_1, \quad
\|\bB_m^0\|_F\geq b_2\, K_m, \quad
\|\bB_m^0\|_\infty\leq 1.
\end{equation}
We remark that the first two inequalities of  \eqref{eq:assumptionA4} imply that the magnitude of $\bB_m^0$ 
is bounded from  below, while the last inequality of  \eqref{eq:assumptionA4} ensures that the magnitude of $\bB_m^0$ 
is bounded from above. In addition,   $\|\bB_m^0\|_\infty\geq  K_m^{-1}\, \|\bB_m^0\|_F$, and \eqref{eq:assumptionA4} 
guarantees that $\|\bB_m^0\|_\infty$ is bounded from both below and above: $b_2\leq \|\bB_m^0\|_\infty\leq 1$.

\subsection{Convergence of ALMA }
\label{sec:convergence}

Denote by  $\kappa_0$  the condition number of matrix $\bbQ_*\times_{2,3}\bbQ_*$:
$$
\kappa_0=\frac{\sigma_1(\bbQ_*\times_{2,3}\bbQ_*)}{\sigma_M(\bbQ_*\times_{2,3}\bbQ_*)}.
$$
Let
\be \label{eq:betanL}
\beta_{n,L} =  \log^2(n+L)\sqrt{M^3 \, \kappa_0^5}  \left(\sqrt{\frac{1}{p_{\max}n^2}}+
\frac{\dot{K}^2}{p_{\max}n\min(n,L)}\right)
\ee
Then,  the following theorem shows that, with a good initialization that satisfies \eqref{eq:assumption23}, 
as the number of iterations tends to infinity, 
Algorithm~\ref{alg:update}   converges to a fixed point that is close to $\bbQ_*$ and $\bW_*$.

\begin{thm}\label{thm:simple_model}
Let Assumptions \textbf{(A1)}-\textbf{(A4)}  hold and  $\kappa_H$ be uniformly bounded away from one  for any $n$ and $L$ large enough.  Let,
for some positive absolute constant  $C_1$
{\begin{align}\label{eq:assumption13}
p_{\max}\, \geq C_1\max\left(\frac{\kappa_0^{12}\log^6(n+L)M^3\dot{K}^3}{(1-\kappa_H)^4\,n\, \min(n,L)},\frac{\log(n+L)}{n+L} \right)
\end{align}}
and the initialization  $\bW^{(1)}$  has an estimation error bounded above by some positive function $h$ of $M,\dot{K}$ and  $\kappa_0$:
\begin{equation}\label{eq:assumption23}
\|\bW^{(1)}-\bW_*\|_F\leq h(M,\dot{K},\kappa_0).
\end{equation}
Then, for some absolute constant $C_2 >0$  and $\iter\geq 1$,
with probability    $1 - o(1)$  as $n,L\rightarrow\infty$,   one has
\begin{align}
\|{\bW}^{(\iter+1)}-\bW_*\|_F\leq \frac{1+\kappa_H}{2}\|{\bW}^{(\iter)}-\bW_*\|_F +  {C_2}\,  \beta_{n,L}
\label{eq:hatW0}
\end{align}
Moreover, for any $\iter>1$, one has
\bes
\|{\bW}^{(\iter)}-\bW_*\|_F   \leq
\lkr \frac{1+\kappa_H}{2}\rkr^{\iter-1}\, \|{\bW}^{(1)}-\bW_*\|_F
  + \frac{2C_2}{1-\kappa_H}\, \beta_{n,L},
\ees
with probability    $1 - o(1)$  as $n,L\rightarrow\infty$.
Consequently, if
$\displaystyle \widetilde{\bW} = \lim_{\iter\rightarrow\infty}{\bW}^{(\iter)}$ and
$\displaystyle \widetilde{\bbQ} = \lim_{\iter\rightarrow\infty}{\bbQ}^{(\iter)}$, then
\begin{align}
\|\widetilde{\bW}-\bW_*\|_F\leq   \frac{2C_2}{1-\kappa_H}\,  \beta_{n,L}
\label{eq:hatW}
\end{align}
In addition, with probability    $1 - o(1)$  as $n,L\rightarrow\infty$,
 for some absolute constant $C_3 >0$ and any $m=1, \ldots, M$, one has
\begin{align}
 \left\|[\bbQ^{(\iter)} -\bbQ_*](m,:,:)\right \| \leq 2p_{\max}\,n\, \sqrt{L}\|{\bW}^{(\iter)} - \bW_*\|_F
+ C_3\, \sqrt{p_{\max}(n+L)}\, \log(n+L).
\label{eq:hatQ}
\end{align}
\\
\end{thm}

Based on the iterative update formula for the estimation error in \eqref{eq:hatW0}, we can
assess performance of Algorithm~\ref{alg:update} with the stopping criterion $\eps_{n,L}$.
\\

\begin{cor}\label{cor:main}
Under the assumptions and the notations of Theorem~\ref{thm:simple_model},
if
\be \label{eq:iter_cond}
\|{\bW}^{(\iter +1)}-{\bW}^{(\iter)}\|_F   \leq \eps_{n,L},
\ee
then, with probability    $1 - o(1)$  as $n,L\rightarrow\infty$, one has
\be \label{eq:alg1_er_bound}
\|{\bW}^{(\iter)}-\bW_*\|_F   \leq  \frac{2 \eps_{n,L}}{1-\kappa_H}\,  
+ \frac{2C_2}{1-\kappa_H}\, \beta_{n,L}.
\ee
As a result, if  $\epsilon_{n,L} \geq 6\, C_2\, (1-\kappa_H)^{-1}\, \beta_{n,L}$,
then Algorithm~\ref{alg:update} would stop within at most $T_{n,L}$ iterations, where
\bes
 T_{n,L} =
\log \lkr \frac {6 \, \|\bW^{(1)}-\bW_*\|_F} {\epsilon_{n,L}} \rkr \Bigg/ \log\lkr \frac{2}{1+\kappa_H} \rkr
\ees
\\
\end{cor}

\begin{rem} \label{rm:permutations}
{\bf Permutations. \ }
{\rm
We remark that $\bW^{(1)}$ needs to be close to $\bW_*$ up to  a permutation of columns, but this permutation has no impact
on the between-layer and within-layer clustering results, as the outputs of Algorithm~\ref{alg:update},
$\widehat{\bW}$ and $\widehat{\bbQ}$, would also be close to $\bW_*$ and $\bbQ_*$
up to permutations of columns and layers, respectively.
\\
}
\end{rem}

\begin{rem} \label{rm:initialization}
{\bf Initialization. \ }
{\rm
Theorem~\ref{thm:simple_model} and Corollary~\ref{cor:main}  require
initialization $\bW^{(1)}$ that satisfies condition \eqref{eq:assumption23}.
If $M$, $\dot{K}$ and $\kappa_0$ are uniformly bounded above for any values of $n$ and $L$, then \eqref{eq:assumption23}
is satisfied by any matrix  $\bW^{(1)}$ since for any $L \times M$ matrix $\bW$ such that $\bW^T \bW = I_M$ one has
$ \|\bW \|^2_F = M$, i.e. condition \eqref{eq:assumption23} holds with $h(M,\dot{K},\kappa_0) = 2 \sqrt{M}$.
In order to obtain more precise results, one
can   first obtain an  initial between-layer clustering matrix $\bZ^{(1)} \in \{ 0,1 \}^{L \times M}$ using, e.g.,
spectral clustering algorithm on vectorized layer matrices $\{\bA_l\}_{l=1}^L$, and then   set
\begin{equation}\label{eq:initialclustering}
 {\bW}^{(1)} = \bZ^{(1)} ((\bZ^{(1)})^T \bZ^{(1)})^{-1/2}.
\end{equation}
However,   as it follows from Theorem~\ref{thm:simple_model} and Corollary~\ref{cor:main},
the errors are smaller and the convergence of ALMA algorithm is faster when a more accurate initialization $\bW^{(1)}$  is used. 
For this reason, in Section~\ref{sec:initialization} we present a   more involved initialization procedure.
\\
}
\end{rem}


{\bf Sketch of the proof.\ \ }  The proof of Theorem~\ref{thm:simple_model}  is deferred to Section~\ref{sec:proof_main}.
Below we provide some insight into how this theorem can be proved.
The proof of the main inequality \eqref{eq:hatW0} in Theorem~\ref{thm:simple_model} can be divided into four steps.

The first step establishes a deterministic bound  on  $\|{\bW}^{(\iter)}-\bW_*\|$  for any given fixed $\bbA$.
The second   and the third steps establish    probabilistic bounds for a random tensor $\bbA$
under the probabilistic model in Section~\ref{sec:model}.
Finally, the fourth step simplifies this probabilistic bound using  Assumptions \textbf{(A2)}-\textbf{(A4)}.

As for the proof of \eqref{eq:hatQ}, it is based on the following chain of inequalities
\begin{align}
\|[\hat{\bbQ}-\bbQ_*](m,:,:)\| & =\|\Pi_{K_m}[\bbA\times_1\hat{\bW}]-\bbQ_*(m,:,:)\|
\leq 2\|\bbA\times_1\hat{\bW}-\bbQ_*(m,:,:)\| \nonumber  \\
& \leq  2\|\bbP_*\times_1(\hat{\bW}-\bW_*)\|+2\|\bbDelta\times_1\hat{\bW}(:,m)\|
\leq  2\|\bbQ_*\|\|\hat{\bW}-\bW_*\|+2\|\bbDelta\|    \nonumber
\\
&\leq 2\|\bbQ_*\|\|\hat{\bW}-\bW_*\|_F+2\|\bbDelta\|,
\label{eq:hatQ1}
\end{align}
where $\bbDelta=\bbA-\bbP_*$, and the factor 2 in the first inequality follows from
\[
\|\bbA\times_1\hat{\bW}-\Pi_{K_m}[\bbA\times_1\hat{\bW}]\|\leq \|\bbA\times_1\hat{\bW}-\bbQ_*(m,:,:)\|.
\]
Inequality \eqref{eq:hatQ} is then obtained by combining \eqref{eq:hatQ1}
with the upper bounds on   $\|\bbDelta\|$
in Lemma~\ref{lemma:prob_main} (see Appendix) and the fact that $\|\bbQ_*\|\leq \|\bbQ_*\|_F\leq p_{\max}n\sqrt{L}$.


\subsection{Consistency of between-layer   and within-layer clustering}
\label{sec:layer_clust}

This section studies   misclassification error rates of network clustering and   local community detection.
The misclassification error rates are measured by the Hamming distance between clustering partitions.
{Since the clustering is unique only up to a permutation of clusters, denote the set of permutation functions of $[r]=\{1, \cdots, r\}$ by $\aleph(r)$.}

Given the true partition of network layers $[L] = \{1,\cdots, L\}=\cup_{m=1}^{M}\calS_m$ and the estimated partition
$[L]=\cup_{m=1}^{M}\hat{\calS}_m$, the misclassification error rate of between-layer clustering is given by
\be \label{eq:err_between}
{R_{BL} = L^{-1}\ \min_{\tau \in \aleph{(M)}}\ \sum_{m=1}^M|\calS_m\setminus \hat{\calS}_{\tau(m)}|.}
\ee
The misclassification rate of within-layer clustering is defined similarly: given the true partition of vertices
$[n]=\{1,\cdots, n\}=\cup_{k=1}^{K_m}G_{m,k}$ and the estimated partition $\cup_{k=1}^{K_m}\hat{G}_{m,k}$,
the misclassification error rate of within-layer  clustering for the $m$-th group of layers is given by
\be \label{eq:err_wihin}
{R_{WL} (m) = n^{-1} \ \min_{\tau \in \aleph(K_m)}\ \sum_{k=1}^{K_m}|G_{m,k}\setminus \hat{G}_{m,\tau(k)}|.}
\ee
The derivations of both misclassification rates are based on the upper bound  \eqref{eq:hatQ} and Lemma C.1 of \cite{lei2020tail}.
In addition, the analysis of  within-layer clustering also applies the Davis-Kahan theorem.
Our results on the misclassification errors are as follows, with the proof deferred to Section~\ref{sec:proof_classification}.
\\

\begin{thm}\label{thm:classification}
(a)  [Between-layer clustering error] For an $(1+\epsilon)$ approximate solution of the $K$-means problem,
with probability    $1 - o(1)$  as $n,L\rightarrow\infty$, the between-layer clustering error is bounded by
\be \label{eq:RBL}
R_{BL} \leq C_{\epsilon} \, \frac{\beta_{n,L}^2}{M^2 (1-\kappa_H)^2}.
%
\ee
for some constant $C_{\epsilon}$ depending on $\epsilon$.

(b) [Within-layer clustering] With probability    $1 - o(1)$  as $n,L\rightarrow\infty$,
the within-layer clustering error of the $m$-th type of network  is bounded by
\be \label{eq:RWL}
R_{WL} (m)  \leq C_{\epsilon}K_{\max} \lkr \frac{\beta_{n,L}^2}{(1-\kappa_H)^2} +
\frac{\log^2(n+L)}{n\, L\, p_{\max}} \rkr, \quad m=1, \ldots, M.
\ee
Here, $\kappa_H$ and $\beta_{n,L}$ are defined in, respectively, \eqref{eq:kappa_H} and \eqref{eq:betanL}.
\end{thm}


\subsection{Discussion  of theoretical guarantees}
\label{sec:discuss_theor}


\subsubsection{Discussion of Assumption \textbf{(A1)}}
\label{sec:A1_discussion}

This section shows that Assumption \textbf{(A1)} is not restrictive, and is usually satisfied in practice.
In Lemma~\ref{lem:assum_A1}, we have already provided   sufficient conditions that guarantee validity of  Assumption~\textbf{(A1)}.
Below, we continue the discussion of this assumption.

\begin{figure}\label{fig:example}
\caption{Examples of convergence rates of the alternating projection algorithm. Left: $\calS_1$ and $\calS_2$ are ``tangent'' to each other and their tangent planes have nontrivial intersections. Right: the tangent planes of $\calS_1$ and $\calS_2$ only have a trivial intersection.}
\begin{center}
\includegraphics[height=3.6cm]{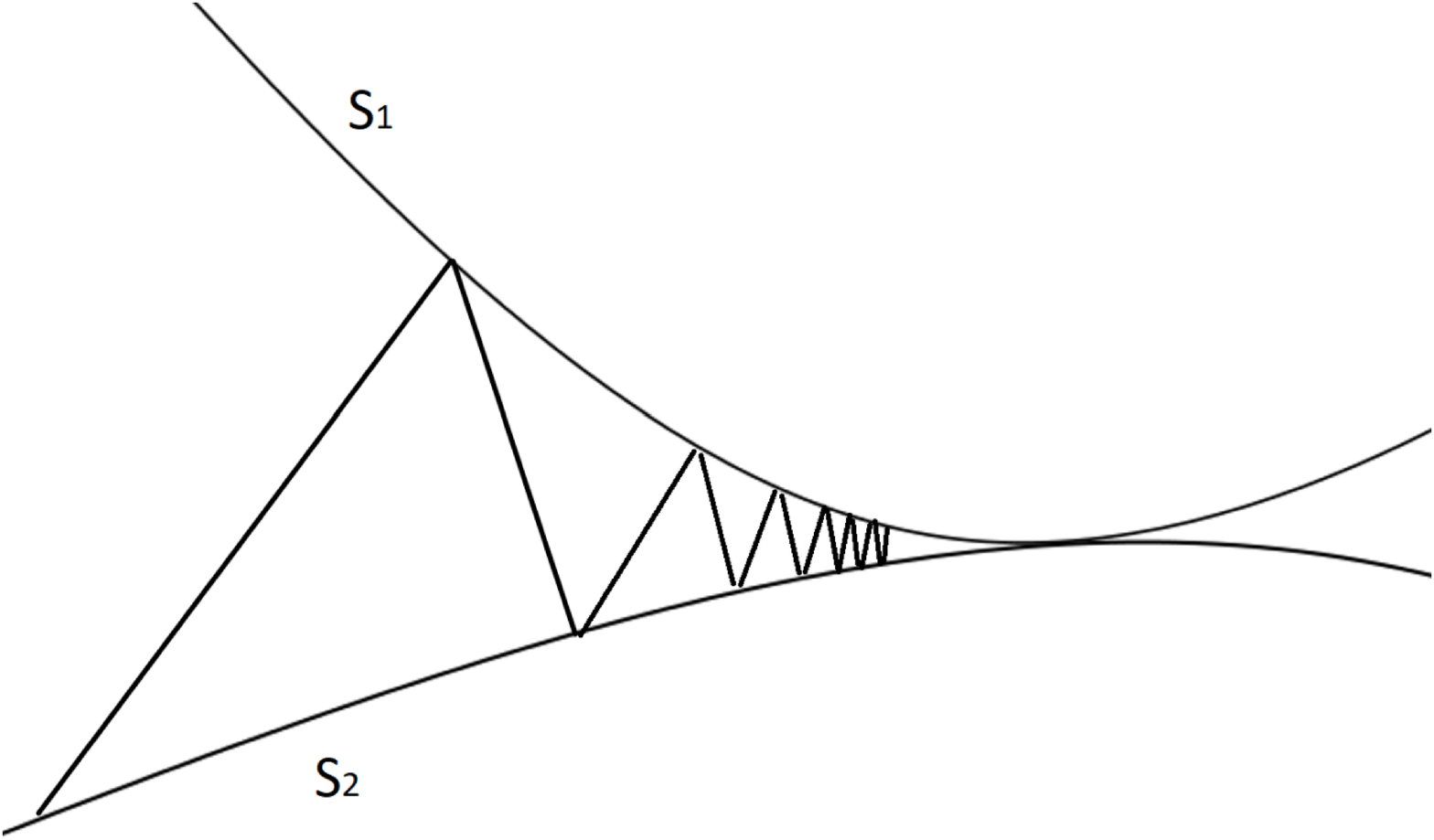}\hspace{0.5in}\includegraphics[height=3cm]{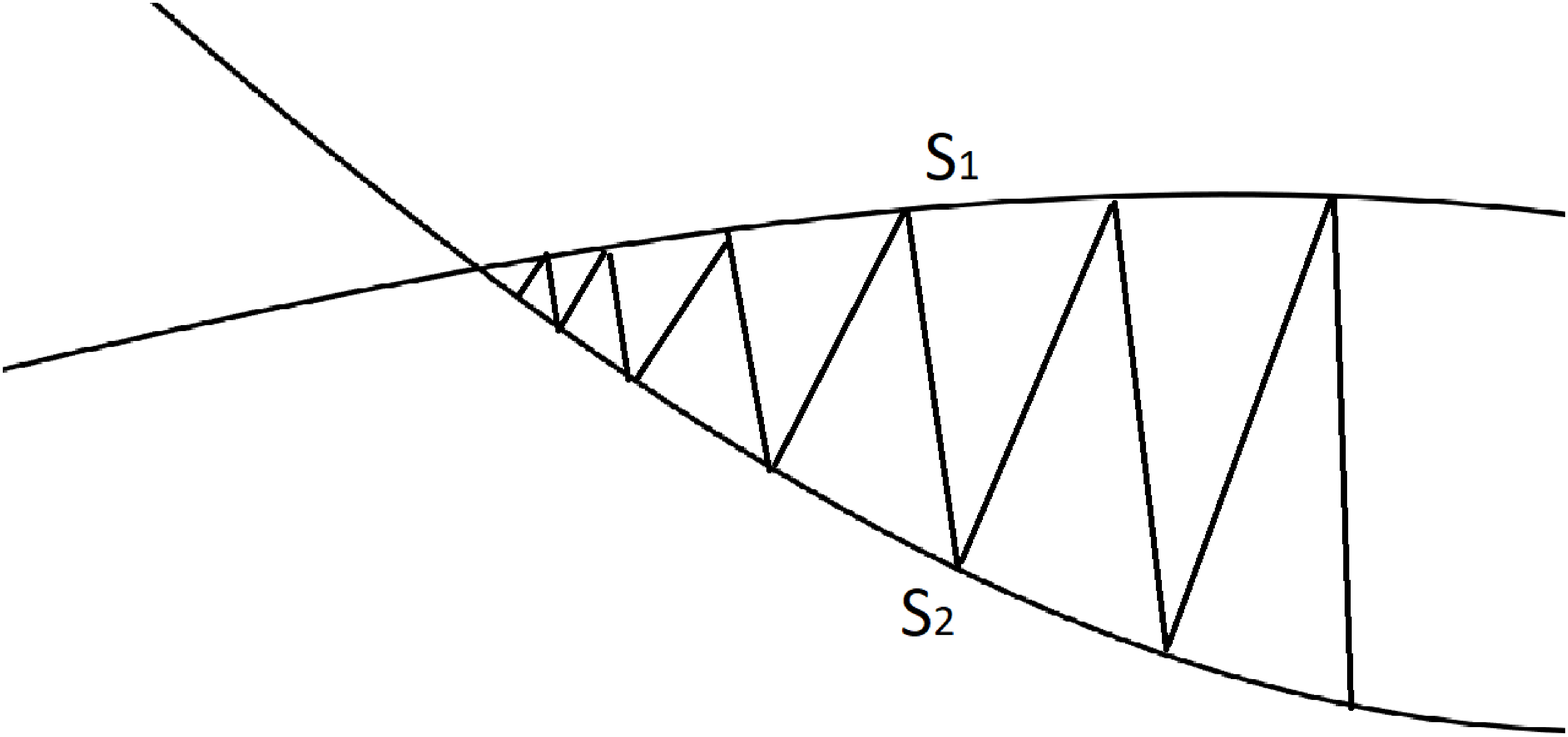}
\end{center}
\end{figure}

The fact that Assumption~\textbf{(A1)} is not restrictive can also be inferred
from counting the dimensions of $L_1$ and $L_2$. Since a symmetric matrix $\bbQ(m,:,:)$ has  $n(n-1)/2$ degrees of freedom,
the set
$$
\{\bbQ(m,:,:): \Pi_{\bU_m^\perp}\bbQ_(m,:,:)\Pi_{\bU_m^\perp}=\mathbf{0}\}
$$
has $n(n-1)/2-(n-K_m)(n-K_m-1)/2= K_m n-  K_m(K_m+1)/2$ degrees of freedom. Summing those up for  $1\leq m\leq M$,
obtain that the dimension of $L_1$
 is $\sum_{m=1}^M (K_m n-  K_m(K_m+1)/2)$. Since the set $\mathrm{Skew}_M$ has $ M(M-1)/2$ degrees of freedom,
$L_2$ has a dimension of $M(M-1)/2$. Since random subspaces of dimensions $d_1$ and $d_2$ in $\reals^D$ 
do not intersect  if $d_1+d_2\leq D$, that is,  if
$$
\frac{M(M-1)}{2} + \sum_{m=1}^M (K_m n- \frac{K_m(K_m+1)}{2})\leq M\, n^2,
$$
(which holds when $n$ is large), then \textbf{(A1(b))} should  usually hold.
\\

We also remark that Assumption \textbf{(A1)} is slightly more restrictive than the local uniqueness of the solution
to the problem \eqref{eq:objective} in the noiseless scenario, which only requires that
$\calS_1$ and $\calS_2$ intersect only at $\bbQ_*$. However, our goal is to prove the linear convergence
of Algorithm~\ref{alg:update}, and, as it is shown in  Figure 1 below, the convergence rate of the alternating method
in  Algorithm~\ref{alg:update} would be slow  and nonlinear if $\calS_1$ and $\calS_2$ are ``tangent'' to each other
and their tangent planes have nontrivial intersections. On the other hand,  the convergence rate is linear
if the tangent planes to $\calS_1$ and $\calS_2$ have only   a trivial intersection.
\\

 We should mention that Assumption \textbf{(A1)} fails in the case of the ``checker board'' model of \cite{JMLR:v21:18-155},
where all networks have the same community structures. As we have indicated, we are not interested in carrying out the inference 
in this case. However, we remark that the ALMA   still succeeds empirically in the checker board model, even though the Assumption \textbf{(A1)} is violated.


\begin{rem} \label{rm:uniqueness}
{\bf Uniqueness of the solution. \ }
{\rm 
Similar to many clustering problems, the solution of the optimization problem  \eqref{eq:objective1} is  unique only up to
permutations of clusters. The non-uniqueness due to permutation of clusters, however, does not cause  difficulty
for Algorithm~\ref{alg:update}. Hence,  Theorem~\ref{thm:simple_model} still applies: if the initialization
$\bW^{(1)}$ is reasonably close to $\bW_*$, then Algorithm~\ref{alg:update} converges to the true solution. 
\\
}
\end{rem}


\subsubsection{Initialization}
\label{sec:initialization}

We remark that   Theorem~\ref{thm:simple_model} and Corollary~\ref{cor:main} require a good initialization
such that \eqref{eq:assumption23}  holds. Also, as Remark~\ref{rm:initialization} states, the errors in
Theorem~\ref{thm:simple_model} and Corollary~\ref{cor:main} are smaller, and the convergence of Algorithm \ref{alg:update}  is faster
when a more accurate initialization $\bW^{(1)}$  is used.
For this reason, in this section we present a  more involved initialization procedure, summarized in Algorithm~\ref{alg:initial},  
which is based on an initial estimator of the between-layer clustering.

%

\begin{algorithm}   [tb]
\caption{Initialization of Algorithm~\ref{alg:update}}
\label{alg:initial}
\begin{flushleft}
{\bf Input:}  {Adjacency matrices $\bA_l\in\reals^{n\times n}$, $1\leq l\leq L$; number of different layers $M$;
$\dot{K}$.}  \\
{\bf Output:} {$\bW^{(1)}$.} \\
{\bf Steps:}\\
{\bf 1:} Set  $\bW^{(0)}\in\reals^{L\times M}$ to be composed of the top $M$ left singular vectors of $\calM_3(\bbA)$.\\
{\bf 2:} Apply the $K$-means algorithms to the rows of $\bW^{(0)}$ to obtain an initial estimator of the between-layer clustering
$[L]=\cup_{m=1}^M\tilde{\calS}_m$.\\
{\bf 3:} Apply \eqref{eq:initialclustering} to obtain $\bW^{(1)}$ from $[L]=\cup_{m=1}^M\tilde{\calS}_m$.
\end{flushleft}
\end{algorithm}

\noindent
Theoretical guarantees on this initialization are given by the statement below. Its proof is deferred to Section~\ref{sec:auxillary_proof}.
\\

\begin{prop} \label{prop:initial}
(Theoretical guarantee of Algorithm~\ref{alg:initial}.)\  Assume that  $p_{\max} \geq c\frac{\log (n^2+L)}{\min(n^2,L)}$  
for some constant $c > 0$. Then for any $r > 0$, there exists
  a constant $C > 0$ depending only on $r, c, k$, such that, with probability at least $1 - n^{-r}$ , 
\[
\|\bW^{(1)}-\bW_*\|_F\leq \sqrt{\frac{C_{\epsilon}M(n^2+L)}{p_{\max}n^2L}}\ 
\frac{\displaystyle \max_{1\leq m\leq M}L_m}{\displaystyle \min_{1\leq m\leq M}L_m}.
\]
\end{prop}

Let us discuss  Proposition~\ref{prop:initial} under the assumptions  that    $M,\dot{K},\kappa_0$ and 
$\frac{\max_{1\leq m\leq M}L_m}{\min_{1\leq m\leq M}L_m}$ are uniformly bounded above, as $n,L \rightarrow\infty$. 
Then, Proposition~\ref{prop:initial}  implies that $\|\bW^{(1)}-\bW_*\|_F\leq O(\sqrt{\frac{1}{p_{\max}\min(n^2,L)}})$,
while    \eqref{eq:assumption23} requires that  $\|\bW^{(1)}-\bW_*\|_F\leq O(1)$. 
As a result, the method satisfies   condition \eqref{eq:assumption23} if $p_{\max}\geq \frac{\log (n^2+L)}{\min(n^2,L)}$.

One can use alternative methods to generate $\bW^{(0)}$, and subsequently follow Steps 2 and 3 in Algorithm~\ref{alg:initial}. 
For example, one can follow  initialization strategy in \cite[Section 5.5]{jing2020community} to obtain $\bW^{(0)}$. 
 This, however, will require additional assumptions to comply with the theory in \cite{jing2020community}.  


\section{Comparison with existing results}
\label{sec:comparison}

To the best of our knowledge, the only paper that studied the model considered in this paper is \cite{jing2020community},
where the authors introduced   algorithm TWIST, based on regularized tensor decomposition.
In this section, we provide theoretical and numerical comparisons with their results.

\subsection{Description of TWIST}
\label{sec:TWIST}

While \cite{jing2020community} consider the model described in this paper, their methodology
and their assumptions are somewhat different. Specifically,   TWIST iterates Tucker decomposition with regularization step on the observation
tensor $\bA$ to obtain a low-rank approximation of $\bA$, where the intention of the
regularization is to dampen the stochastic errors.
The Tucker structure of the approximation is used to cluster the nodes and the layers.  \cite{jing2020community} start with compiling a collection of all
clustering matrices $\bTe_m$, $m=1, \ldots, M$, in \eqref{maineq} into one matrix
$\bTe \in \reals^{n\times\dot{K}}$ defined as $\bTe=[\bTe_1,\cdots,\bTe_M]$. With this notation, they obtain the
Tucker decomposition of the true probability tensor $\bP_*$ as
\be \label{eq:xia_prob_tensor}
\bP_* =  \bbB  \times_2 \bTe \times_3 \bTe \times_1 \bZ, \quad \bC\in\reals^{\dot{K}\times\dot{K}\times M},
\ee
where $\bZ \in \{0,1\}^{L \times M}$ is the clustering matrix of layers such that $\bZ (\bZ^T \bZ)^{-1/2} = \bW_*$
and  $\bbB$ is  defined as
\be \label{eq:tensorB}
\bbB(:,:,m)=\diag(0, \cdots,0,\bB_m,0,\cdots,0), \quad m=1, \ldots, M.
\ee
Furthermore, they obtain the SVD $\bTe = \bar{\bU} \bar{\bD} \bar{\bR}$  of $\bTe$
where matrices $\bar{\bU} \in \real^{n \times r}$ and $\bar{\bR}  \in \real^{\dot{K}\times \times r}$
have orthonormal columns, $r$ is the rank of $\bTe$ and $\bar{\bD}$ is the diagonal matrix of nonzero singular values.
The objective of the technique is to recover matrix $\bW_*$ as well as $\bar{\bU}$.

The TWIST algorithm is based on iterative updates of matrices $\widehat{\bU}^{(iter)}$ and $\widehat{\bW}^{(iter)}$.
Specifically, given $\widehat{\bU}^{(iter)}$ and $\widehat{\bW}^{(iter)}$, TWIST sets
\bes
\tilde{\bU}^{(iter)}  =  \calP_{\del_1,r}  (\widehat{\bU}^{(iter)}), \quad
\tilde{\bW}^{(iter)}  =  \calP_{\del_2,M}  (\widehat{\bW}^{(iter)}),
\ees
where, for matrix $\bV$, any $\del >0$ and positive integer $s$, one has
\be \label{eq:calP_delta}
\calP_{\delta, s} (\bV) = SVD_{s}(\bV_*)  \quad \mbox{with} \quad     \bV_*(i,:)=\bV(i,:)\, \frac{\min(\delta,\|\bV(i,:)\|)}{\|\bV(i,:)\|},
\ee
and $SVD_{s}$ returns the top $s$ left singular vectors. Subsequently, the new iterations
$\widehat{\bU}^{(iter+1)}$ and $\widehat{\bW}^{(iter+1)}$  are obtained as, respectively, the top $r$ left singular vectors of
$\calM_3 (\bA \times_1  (\tilde{\bU}^{(iter)})^T \times_2 (\tilde{\bW}^{(iter)})^T$, and the top $M$ left singular vectors of
$\calM_1 (\bA \times_2  (\tilde{\bU}^{(iter)})^T \times_3 (\tilde{\bU}^{(iter)})^T$.
The process is carried out till the number of iterations reaches the pre-specified value ${\rm iter}_{\max}$.


\subsection{Theoretical comparisons  of TWIST and ALMA}
\label{sec:theor_comparison}

Note that TWIST 
aims at  revealing  both the global and the local memberships of nodes, together with the  memberships of layers.
Since Algorithm~\ref{alg:update} (ALMA) does not deal with the concept of global communities, in
the context of this paper, we use the terms ``within--layer'' and ``between--layer'' clustering to stand for the local memberships
and memberships of layers, respectively. The global communities, defined  in \cite{jing2020community}, are related to,  but not identical, to
the persistence of the local ones in all layers.

Since   \cite{jing2020community} apply a different technique,
their assumptions, theoretical analysis and final results differ from ours.
We start with the comparison of the assumptions.

Specifically,    \cite{jing2020community} impose the following conditions:

\begin{enumerate}

\item
Denote $\sigma_{\min}(\bbB)=\min \lfi \sigma_{\min}(\calM_1(\bbB)),\sigma_{\min}(\calM_2(\bbB)),\sigma_{\min}(\calM_3(\bbB))\rfi$.
Then, for the tensor $\bbB$, defined in \eqref{eq:tensorB}, one has
$\sigma_{\min}(\bbB)\geq \tilde{c}_1 p_{\max}$.
 Note that
\[
\sigma_{\min}(\bbB)\leq \sigma_{\min}(\calM_1(\bbB)) = \min_{m=1,\cdots,M}\sigma_{K_m}(\bB_m) 
= p_{\max}\min_{m=1,\cdots,M} \sigma_{K_m} (\bB_m^0).
\]
Hence, the assumption on $\sigma_{\min}(\calM_1(\bbB))$ is equivalent to the first part of Assumption \textbf{(A4)} in \eqref{eq:assumptionA4}.
 While the assumptions on $\sigma_{\min}(\calM_j(\bbB))$ for $j=2,3$, are not directly comparable to  the second part of Assumption \textbf{(A4)} in
 \eqref{eq:assumptionA4}, both serve a similar purpose that $\bbB$ is not too small.
\\

\item
Matrix $\bTe$ in \eqref{eq:xia_prob_tensor}, for some $\tilde{\kappa}_{0} < \infty$, is assumed to satisfy the condition
$\sigma_{\max}(\bTe)\leq \tilde{\kappa}_{0} \sigma_r(\bTe)$  where $r=\rank(\bTe)$.
This assumption implicitly implies our assumption \textbf{(A2)} and the ``balanced local community'' assumption in \textbf{(A3)},
i.e., the second condition in  \eqref{eq:assumptionA3}.
\\

\item
Layer sizes and community sizes in the layers are  assumed to be similar.
This is equivalent to the   assumption   \textbf{(A3)}.
\\

\item
Network sparsity assumption
$L\, n\, p_{\max}\geq C(\dot{K}+\tilde{\kappa}_0^2\, r^2\, \log^2n)M^{-1}\tilde{\kappa}_0^6\, r^{2}\, \log^2(r\tilde{\kappa}_0)\, \log^{2}n$.
In comparison, we have a similar assumption in \eqref{eq:assumption13}.
\\

\item
Theoretical analysis of TWIST is carried out under the condition that $L\leq n$. We do not impose this assumption.
\\
\end{enumerate}

\noindent
Under these assumptions, the  error  rate of the between-layer clustering of TWIST is
\bes
R_{BL}^{(TWIST)}  = O\lkr \tilde{\kappa}_0^4\frac{r^2\log n}{M\,L\, n\, p_{\max}}\rkr.
\ees
Under the additional assumption  $\sqrt{L}np_{\max}\geq CM^{-1}\tilde{\kappa}_0^5r^{5/2}\log(r\tilde{\kappa}_0)\log^{5/2} n$,
the error rate of within-layer clustering for the $m$-th type of network is
\bes
R_{WL}^{(TWIST)} (m) = O \lkr \frac{\tilde{\kappa}_0^4K_m^2\log n}{M\,L\, n\, p_{\max}} \rkr.
\ees
We remark that   the comparison with $\kappa_0$ and $\tilde{\kappa}_0$ is not straightforward, as they are defined very differently
(even though both measure how ``well-conditioned'' the model is). In order to compare the clustering errors of ALMA and TWIST,
we assume that   $M, K_m,  \kappa_0$ and  $\kappa_H$ are uniformly bounded by constants independent of $n$ and $L$,
so that, as a result, the same is true for $\dot{K},r$ and $\tilde{\kappa}_0$.
Then the clustering error rates  are more comparable, since they  depend only on $n, L$, and $p_{\max}$.
Specifically, the error rates of the between-layer clustering are
\begin{equation}\label{eq:error1}
R_{BL}^{(ALMA)} = O\left(  \frac{\log^4(n+L)}{n^2\, p_{\max}}+\frac{\log^4(n+L)}{n^2 [\min(n,L) \, p_{\max}]^2} \right),
\quad
R_{BL}^{(TWIST)}  = O\lkr  \frac{\log n}{L\, n\, p_{\max}}\rkr.
\end{equation}
The error rates of the within-layer clustering for the $m$-th type of network are
\begin{equation}\label{eq:error2}
R_{WL}^{(ALMA)} (m) = O \lkr \frac{\log^4(n+L)}{n\, \min(n,L)\, p_{\max}} \rkr,
\quad
R_{WL}^{(TWIST)} (m) = O \lkr  \frac{\log n}{L\, n\, p_{\max}}\rkr.
\end{equation}
Recall that, for both the between-layer clustering and  the within-layer clustering,
the error rates of TWIST are derived under the assumption that $L \leq n$.

In comparison, the error  rates of the between-layer clustering of ALMA are  better in two aspects.
First, they  hold  in the case of $L>n$. Second, since both methods require that the quantity $n\, L\, p_{\max}$
grows with $n$ and $L$,   it is easy to see that, up to  the logarithmic factors,
\bes
R_{BL}^{(ALMA)} = o\lkr R_{BL}^{(TWIST)} \rkr \quad \mbox{if} \quad n \to \infty, \ L/n \to 0.
\ees
Also, up to the logarithmic factors, the within-layer clustering error rates   are equivalent.

However,   \cite{jing2020community} do  not have  anything similar to Assumption \textbf{(A1)}
imposed in the present paper. This assumption is due to the fact that Theorem~\ref{thm:simple_model}
attempts to achieve something more than clustering: it aims at recovering $\bbQ_*$ and $\bW_*$ directly.
In addition, we have somewhat stronger assumption on sparsity, requiring that $p_{\max}>O\Big(\frac{\log (L+n)}{n+L}+\frac{\log^6(n+L)}{n\min(n,L)}\Big)$, instead of $p_{\max}>O(\frac{\log n}{nL})$ in \cite{jing2020community}. We suspect that the difference in the assumptions
is due to the technicalities in our analysis, rather than the inherent drawbacks of our algorithm. 
\\

It would also be interesting to compare the computational cost of ALMA and TWIST. In ALMA, the steps $\bbQ^{(\iter+1)}=\Pi_{\bK}(\bbA\times_1  \bW^{(\iter)})$  
and    $\bW^{(\iter+1)}=\Pi_o(\bbA\times_{2,3} \bbQ^{(\iter+1)})$  
 require, respectively,    $O(LMn^2+\dot{K}n^2)$ and  $O(LMn^2+LM^2)$ operations, 
 so that,  each iteration of ALMA requires    $O(LMn^2+LM^2+\dot{K}n^2)$ operations. 
 When $L$ and $n$ are large,  the dominant term is $O(Ln^2M)$. In comparison, each iteration of TWIST has a computational cost of $O(Ln^2(r+M))$, which is larger than that of ALMA, since   $r$ is larger than $M$, due to $[\bTe_1,\cdots,\bTe_M]\in\reals^{n\times \dot{K}}$, where   $\dot{K}=\sum_{m=1}^MK_m>M$.


\subsection{Numerical comparisons}
\label{sec:simul}

As it is evident from the previous section, the theoretical comparison between ALMA and TWIST is very difficult
due to the differences between assumptions.

In order to test the performance of Algorithm~\ref{alg:update} (ALMA) and subsequent within-layer clustering,
and to provide a fair comparison of the clustering precisions with the TWIST technique of   \cite{jing2020community},
we carry out a limited simulation study with various choices of parameters  $p_{\max}$, $L$ and $n$.
We use the misclassification rates as measures of the performance of our algorithm.
Specifically, we characterize the between layer clustering  precision by \eqref{eq:err_between}.
For the error of the within-layer clustering, we average the rates in \eqref{eq:err_wihin}  over the $M$ layers  and use
\be \label{eq:err_ave_wihin}
R_{WL}  = M^{-1} \ \sum_{m=1}^M R_{WL} (m).
\ee
We choose $M$ and fix $K_1=\cdots=K_M=K$, so that in each cluster, network follows SBM with $K$ communities.
The underlying class for each layer, and the membership for each node in every class of layers  are randomly sampled
using the multinomial distributions with equal class probabilities $1/M$ for the layers of the networks, and $1/K$
for the nodes in each of the layer clusters. In each of the layers, we
 use identical connectivity matrices $\bB_m \equiv B$ where the diagonal values are set to $p = p_{max}$
while the off-diagonal entries are equal to $q = \alpha p_{max}$ with $\alpha <1$. The constant $\alpha$ controls
the ratio of the probability of connection of a node outside its own community versus inside it.
Consequently, the within layer clustering is easier when $\alpha$ is small and harder when it is large.

\begin{figure}[t] 
    \begin{center}
    \subfloat[Between-layer clustering]{\includegraphics[width=0.4\linewidth]{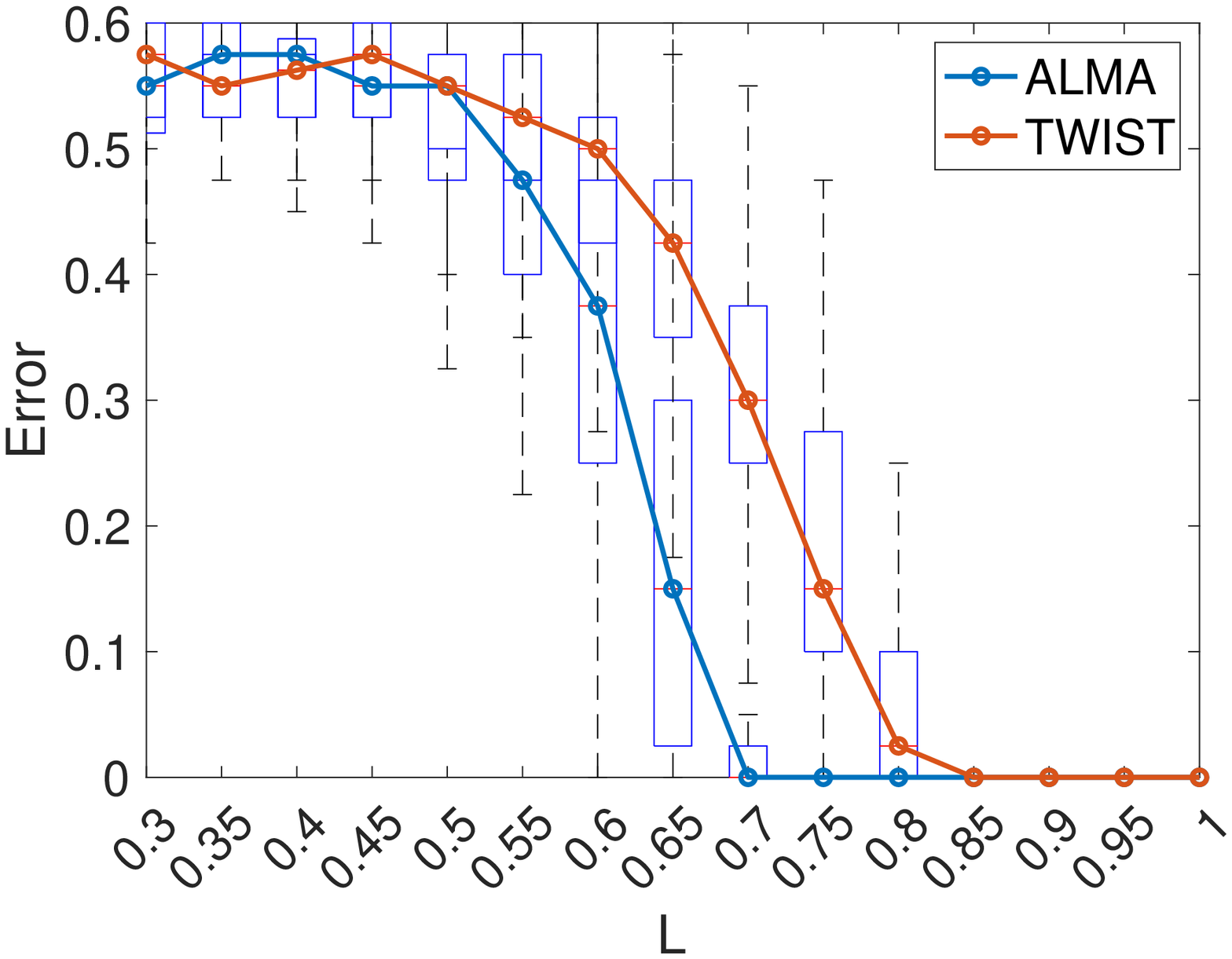} }
    \subfloat[Within-layer clustering]{\includegraphics[width=0.4\linewidth]{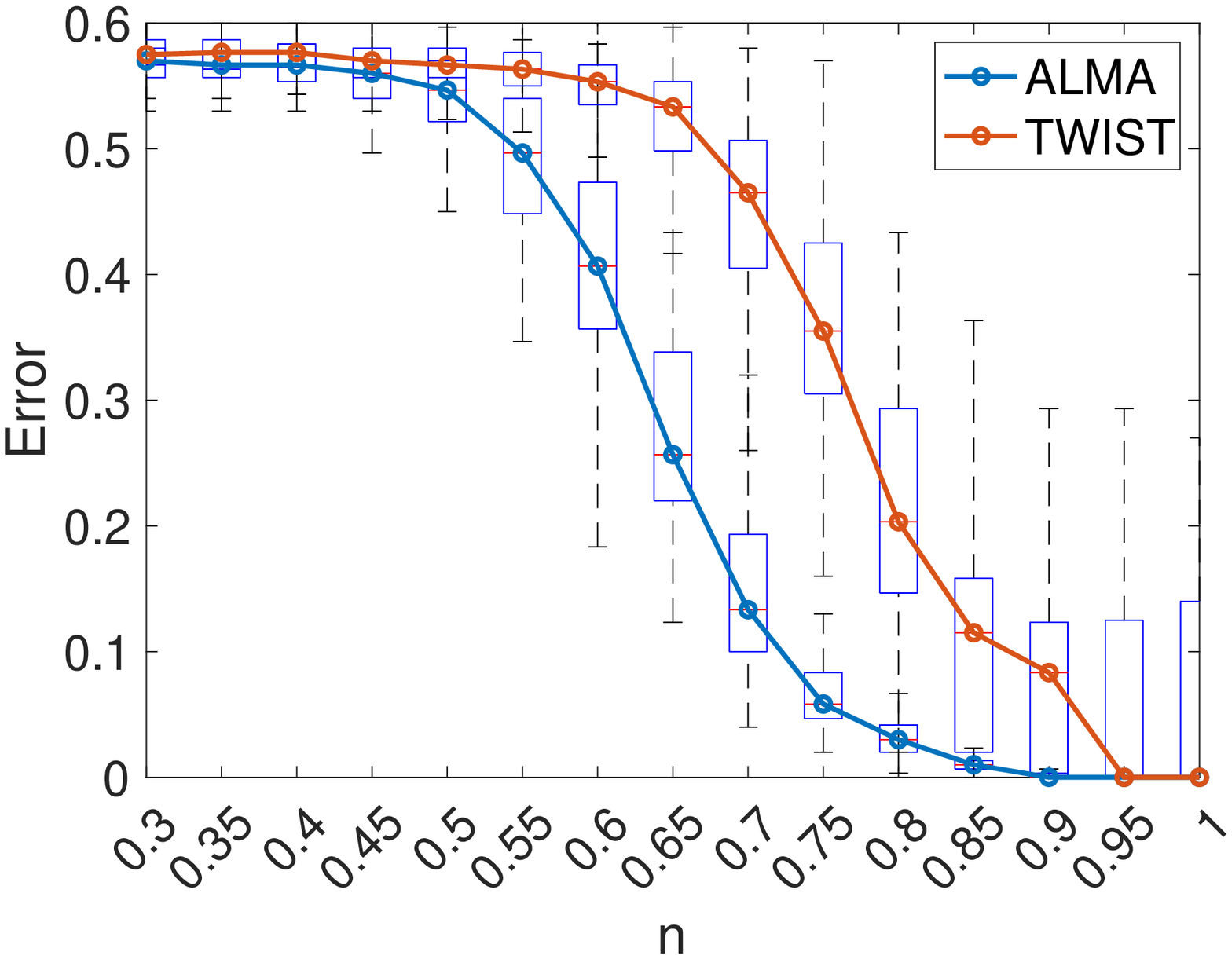} }
    \caption{Simulation Scenario 1: $L=40, n=100, M=3, K=3, \alpha=0.9$.
The between-layer clustering errors and within-layer clustering errors are plotted versus $p_{max}$. 
 The solid lines exhibit the average misclassification errors. } 
    \label{fig:simu-pmax}
    \end{center}
\end{figure}

\begin{figure}[t] 
    \begin{center}
    \subfloat[Between-layer clustering]{\includegraphics[width=0.4\linewidth]{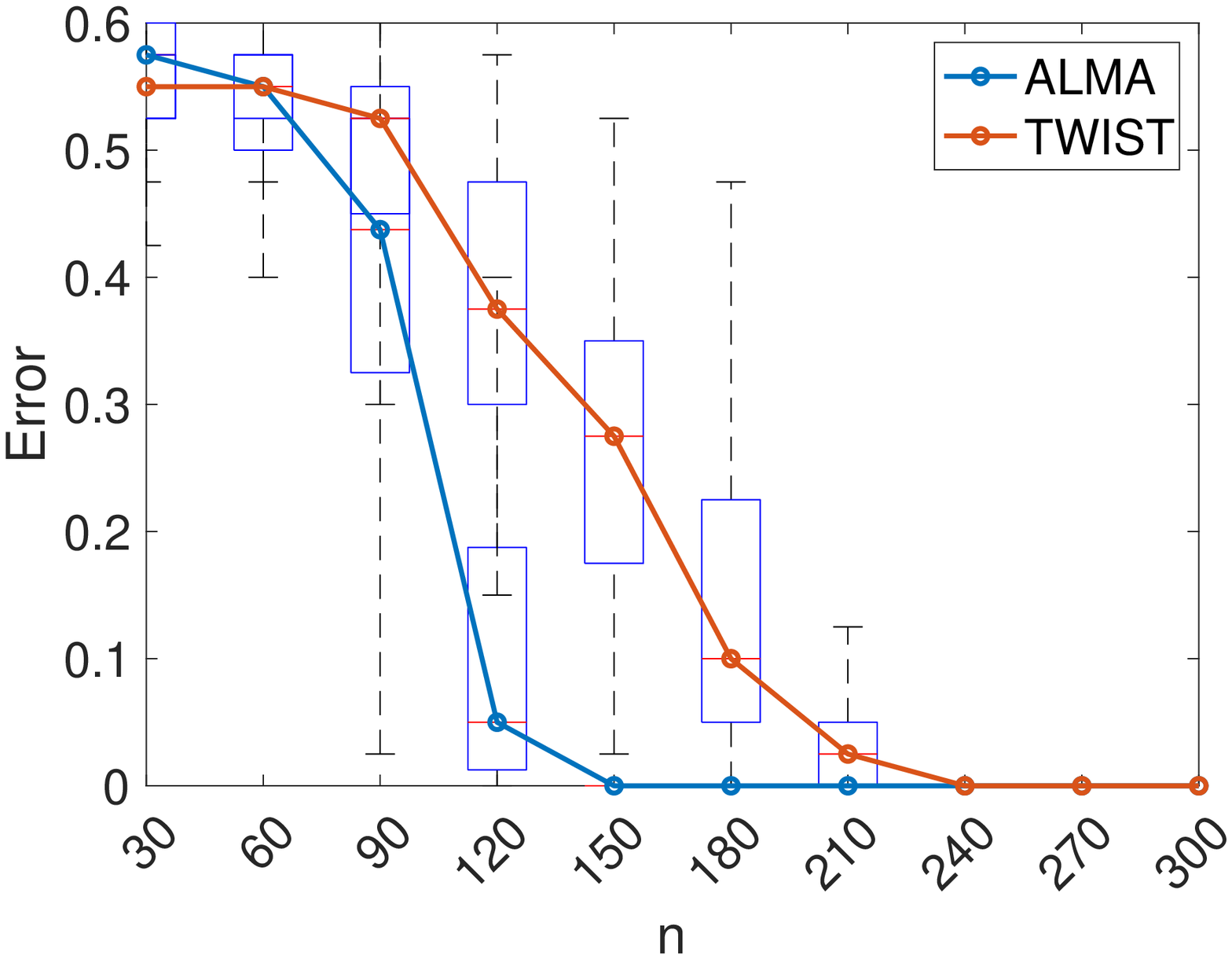} }
    \subfloat[Within-layer clustering]{\includegraphics[width=0.4\linewidth]{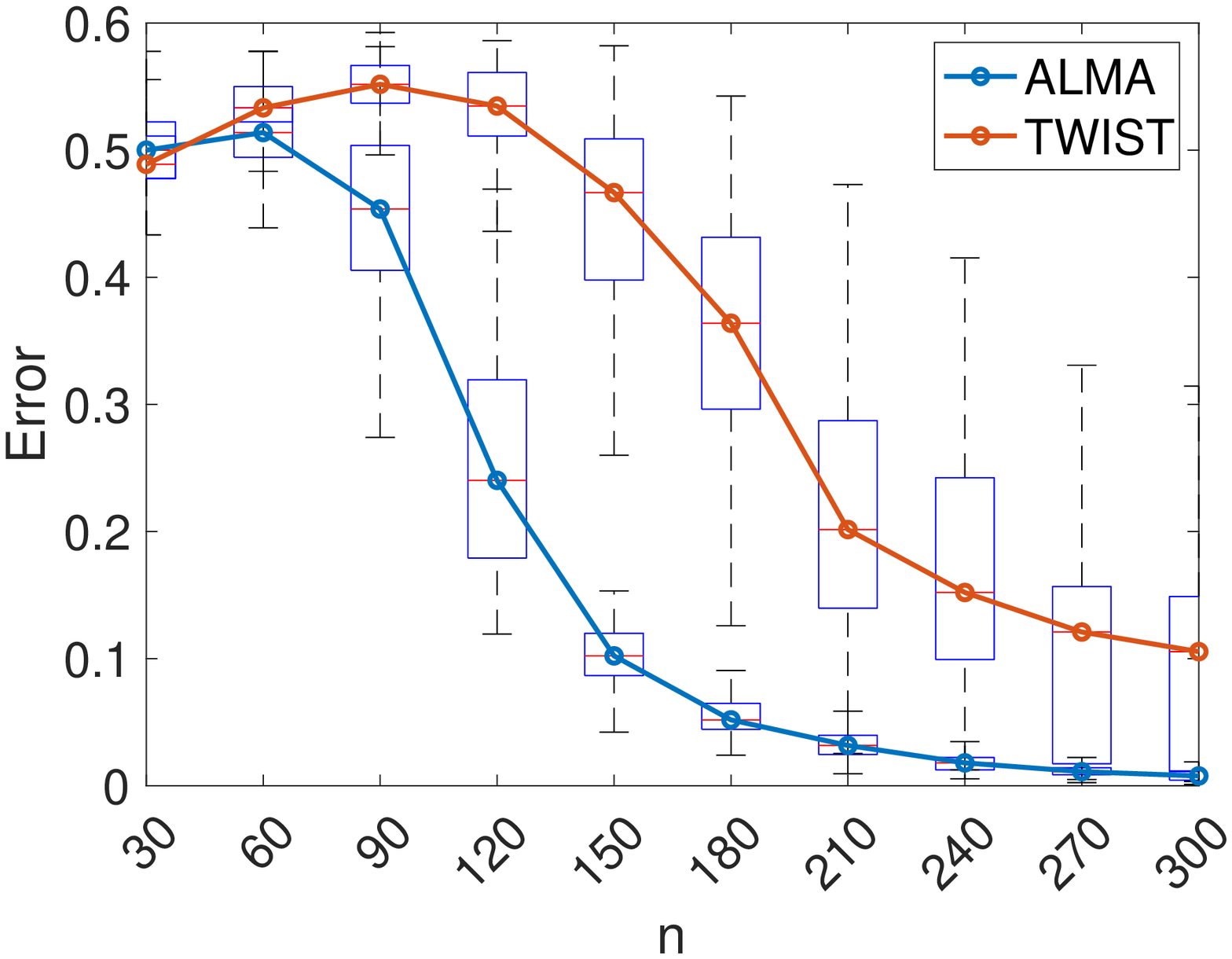} }
    \caption{Simulation Scenario  2: $L=40, M=3, K=3, p_{max}=0.6, \alpha=0.9$.
The between-layer clustering errors and within-layer clustering errors are plotted versus the number of vertices $n$. 
 The solid lines exhibit the average misclassification errors. } 
    \label{fig:simu-n}
    \end{center}
\end{figure}

We investigate the performances  of ALMA (Algorithm~\ref{alg:update}) and compare it with the performances of the TWIST
in four simulation scenarios.
%
In our simulations, we set $M=3$, $K=3$ and $r=7$, since
$r=\rank([\bTe_1,\cdots,\bTe_M])\leq \sum_{m=1}^M K_m-(M-1)$, with inequality  occurring in  degenerate setting.
Since our approach does not involve the concept of global membership, we only compare ALMA with TWIST in terms of ``within--layer'' and
``between--layer clustering''. Furthermore, we choose the stopping criterion   $\|\mathbf{W}^{(\iter)}-\mathbf{W}^{(\iter-1)} \|\leq 10^{-4}$
for both of ALMA and TWIST to make a fair comparison between the algorithms.
Below we describe the simulation schemes.

In Simulation 1, we investigate the effect of the network sparsity on the precision of the algorithms.
For this purpose, we choose   the number of vertices $n=100$, the number of layers $L=40$, the number of network clusters $M=3$,
the number of communities in each cluster of layers $K=3$ and   $\alpha=0.9$.
The variable $p_{max}$, which controls the overall network sparsity, varies from $0.3$~to~$1$.
Fig.~\ref{fig:simu-pmax} shows that both between-layer and within-layer clustering errors decrease as $p_{max}$ is increasing.

\begin{figure}[t] 
    \begin{center}
    \subfloat[Between-layer clustering]{\includegraphics[width=0.4\linewidth]{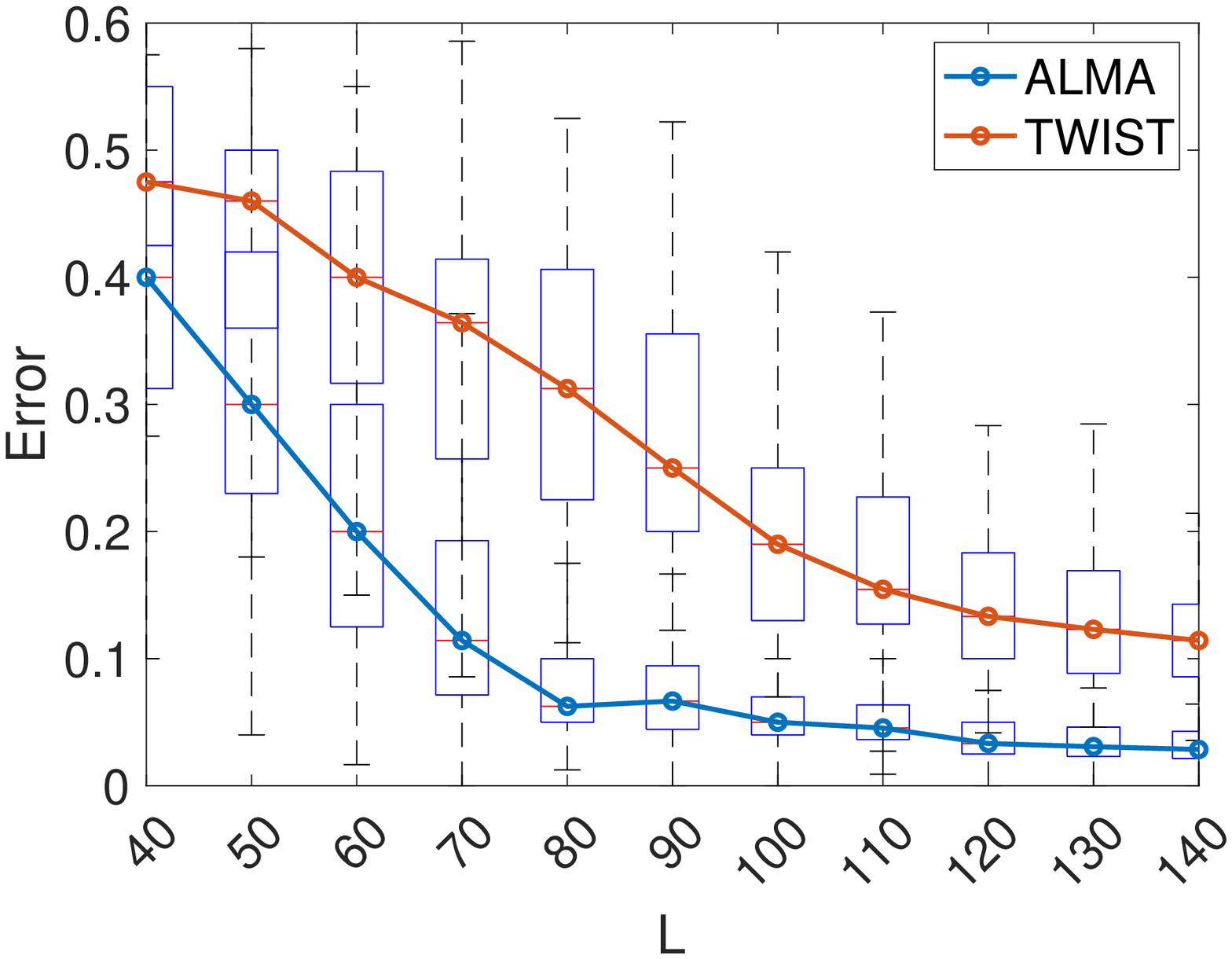} }
    \subfloat[Within-layer clustering]{\includegraphics[width=0.4\linewidth]{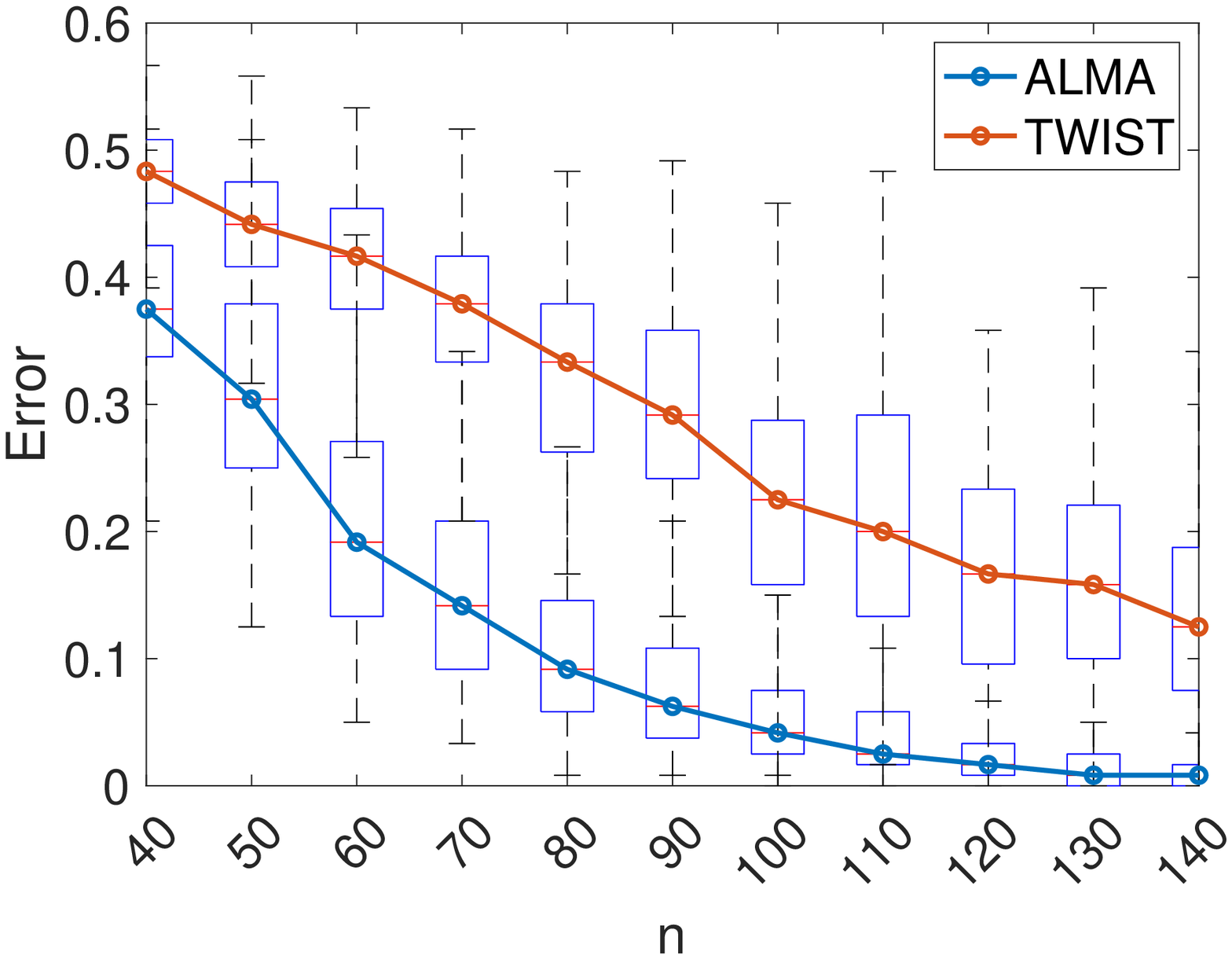} }
    \caption{Simulation Scenario  3: $n=40, M=3, K=3, p_{max}=0.5, \alpha=0.8$.
The between-layer clustering errors and within-layer clustering errors are plotted versus the number of layers $L$ when $L>n$. 
 The solid lines exhibit the average misclassification errors. } 
    \label{fig:simu-L1}
    \end{center}
\end{figure}

\begin{figure}[t] 
    \begin{center}
    \subfloat[Between-layer clustering]{\includegraphics[width=0.4\linewidth]{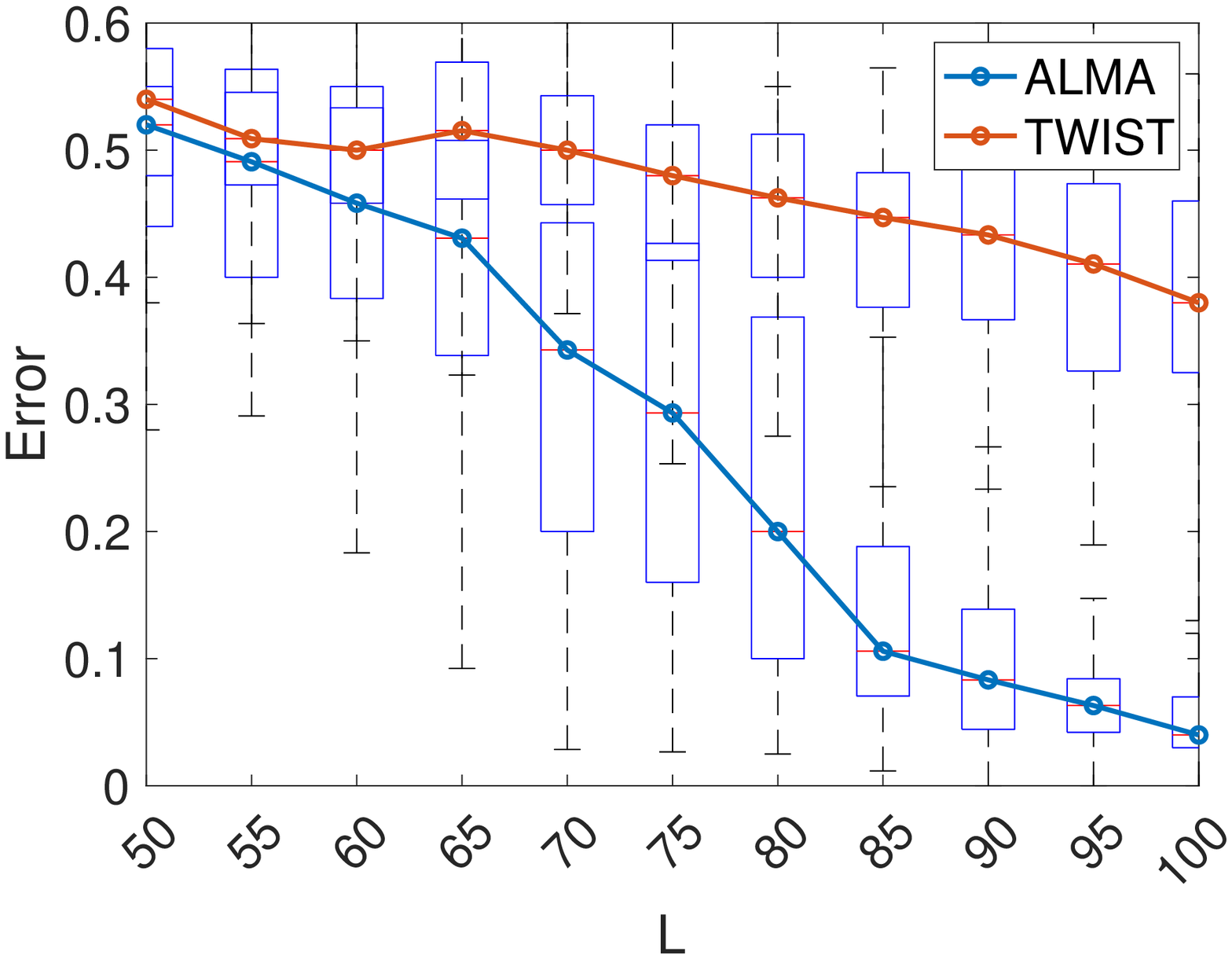} }
    \subfloat[Within-layer clustering]{\includegraphics[width=0.4\linewidth]{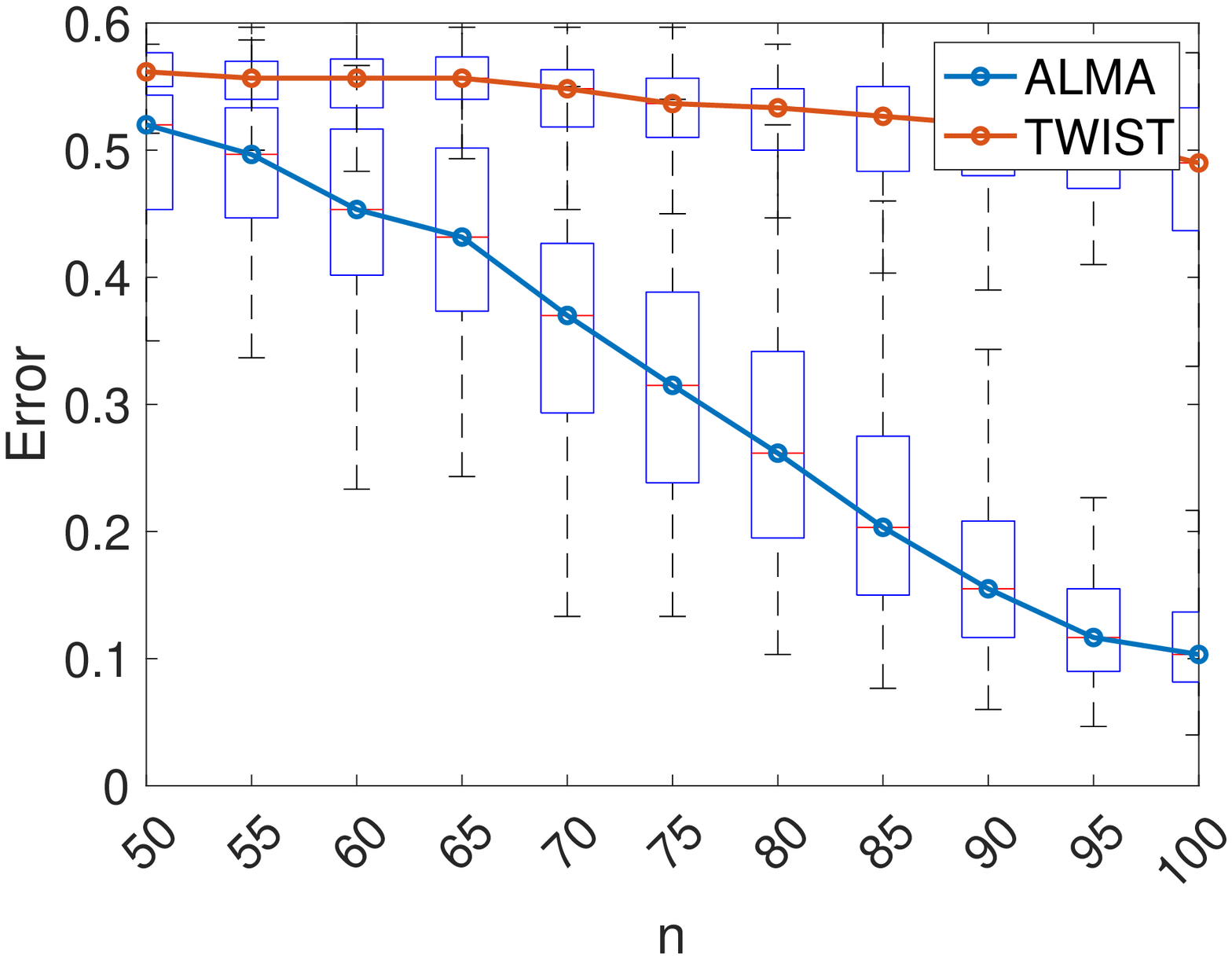} }
    \caption{Simulation Scenario  4: $n=100, M=3, K=3, p_{max}=0.5, \alpha=0.9$.
The between-layer clustering errors and within-layer clustering errors are plotted versus the number of layers $L$ when $L<n$. 
 The solid lines exhibit the average misclassification errors. } 
    \label{fig:simu-L2}
    \end{center}
\end{figure}

In Simulation 2,  the settings are the same as Simulation 1 except that $p_{max}=0.6$ is fixed,
and the number of vertices varies from $30$ to $300$. As $n$ increases,
the between-layer  and within-layer clustering error  rates decrease to zero,
as predicted by  Theorem~\ref{thm:classification}.

In Simulations 3 and 4, we study the effect of the numbers of  layers in the network,
when $L>n$ and $L <n$, respectively. Specifically, in Simulation 3,  we set
 $n=40, M=3,K=3,\alpha=0.8, p_{max}=0.6$ and vary the number of layers $L$ between $40$ and $140$.
The settings  in   Simulation 4 are the same as Simulation 3, except $n=100$ is larger, and $L$ varies from $50$ to $100$.

For both algorithms, in each of the simulation scenarios, we report the between-layer  and the within-layer clustering
errors \eqref{eq:err_between} and \eqref{eq:err_wihin}, respectively, averaged over  100  independent simulation runs.
The results are summarized in  Figures~\ref{fig:simu-pmax}--\ref{fig:simu-L2}.

As it is evident from Figures~\ref{fig:simu-pmax}--\ref{fig:simu-L2}, for all four scenarios, ALMA
has  smaller both the between-layer  and the within-layer clustering errors. Note also that ALMA has better
precision not only in the case of $L > n$, that violates the assumptions of TWIST, but also in the case of $L \leq n$.


\section*{Acknowledgements}

Marianna Pensky was  partially supported by National Science Foundation (NSF)  grants DMS-1712977 and DMS-2014928. Teng Zhang was partially supported by National Science Foundation (NSF)  grant CNS-1818500.


\section{Appendix}
\label{sec:Proofs}

\subsection{Manifold and tangent space}\label{sec:tangent}

The concepts of tangent vector and tangent space to an abstract manifold can be found in, e.g., 
 \cite{boothby2003introduction} and \cite{absil2009optimization}. When $\calM$ is a manifold embedded 
in the Euclidean space $\mathbb{R}^p$, then a smooth function $\gamma: \mathbb{R}\rightarrow\calM$ is called a 
\textit{curve} in $\calM$, and $\gamma'(0)$ is a \textit{tangent vector} to the manifold $\calM$ at the point $\gamma(0)$.
The \textit{tangent space} of $\calM$ at $x$, denoted by $T_x\calM$, is the set of all tangent vectors of $\calM$ at $\bx$, 
that is, $T_x\calM=\{\gamma'(0): \text{$\gamma: \reals\rightarrow\calM$ is a smooth function with $\gamma(0)=x$}\}$ 
\citep{Absil2015}. Intuitively, the tangent plane $T_x\calM$ is the subspace that approximates the manifold $\calM$ 
in a local neighborhood around $x$. For example, if  $\calM$ is the unit sphere $\{x=(x_1,x_2,x_3): x_1^2+x_2^2+x_3^2=1\}$, 
then the tangent space at point $(\hat{x}_1,\hat{x}_2,\hat{x}_3)$ is given by 
$\{(x_1,x_2,x_3): \hat{x}_1 x_1+\hat{x}_2 x_2+\hat{x}_3 x_3=0$\}. 
A visualization of the tangent space is given in Figure~\ref{fig:tangent}.  
\begin{figure}
  \centering
  \includegraphics[width=0.5\linewidth]{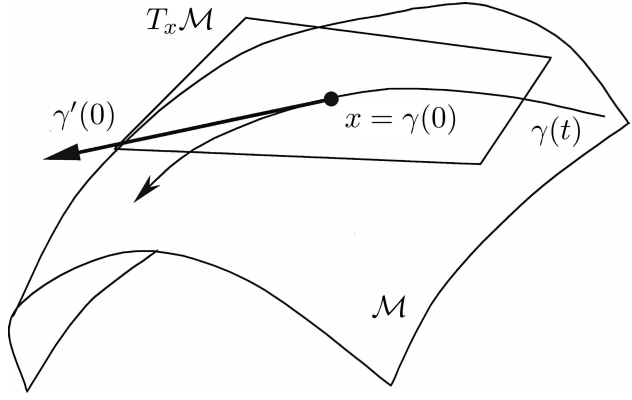}
  \caption{A visualization of the manifold $\calM$, the curve $\gamma$ in $\calM$, the tangent vector $\gamma'(0)$, and the tangent space $T_{\bx}(\calM)$. }
  \label{fig:tangent}
\end{figure}

It remains to derive the tangent planes to the sets $\calS_1$ and $\calS_2$ at $\bbQ_*$ in \eqref{eq:L1L2}.
The expression for $L_1$ follows from the formula for the tangent planes for the manifold of low-rank matrices \citep[equation (13)]{Absil2015}.
Specifically, the  explicit formula for the tangent plane to the manifold of rank $K$ matrices at $\bQ$ is given by
the equation $\Pi_{\bU^\perp}\bbQ \Pi_{\bU^\perp}=\mathbf{0}$, where $\bU$ is an orthogonal matrix that has the same column
space as $\bQ$. Now, the first formula  in \eqref{eq:L1L2} is due to the fact 
that $\calS_1$ is the product of $M$ manifolds of low-rank matrices: $\calS_1=\otimes_{m=1}^M\calM_m$, where 
\[
\calM_m=\{\bX\in\reals^{n\times n}: \rank(\bX)\leq K_m\}.
\]
In order to obtain the second equation in \eqref{eq:L1L2}, note that the tangent plane 
to the set of orthogonal matrices $\calM_0$ at $\bI$ is the set of skew-symmetric matrices: $T_{\bI}\calM_0= \mathrm{Skew}_M$
\citep[Section 2.2.1]{doi:10.1137/S0895479895290954}.
Now, the explicit formula for $L_2$ follows from the facts that   $\calS_2$ is obtained by multiplying 
$\bbQ_*$ with each element from $\calM_0$, where $\calM_0$ is the set of orthogonal matrices of size $M\times M$: 
$\calS_2=\{\bbQ = \bbQ_*\times_1\bV: \bV\in \calM_0\}$.

\subsection{Proof of Theorem~\ref{thm:simple_model}}\label{sec:proof_main}
The organization of this section follows from the sketch of the proof after Theorem~\ref{thm:simple_model} in four steps: the first step establishes a deterministic bound of $\bW^{(\iter)}-\bW_*$,  the second and the third steps establish a probabilistic bound, and the fourth step simplifies the probabilistic bound using Assumptions \textbf{(A2)}-\textbf{(A4)}.

\subsubsection{Step 1: Deterministic analysis of Algorithm~\ref{alg:update}}\label{sec:step1}
In this step, we aim to find a metric $\|\cdot\|_d$ on $\reals^{L\times M}$ such that $\{\|\bW^{(\iter)}-\bW_*\|_d\}_{\iter=1}^{\infty}$ should be monotonically decreasing approximately. While it is natural to consider the Frobenius norm, the previous analysis of the noiseless case in algorithm \eqref{eq:update2} does not support the monotonicity of $\{\|\bW^{(\iter)}-\bW_*\|_F\}_{\iter=1}^{\infty}$. Instead, it establishes the ``approximate'' monotonicity of $\|\bbQ_*\times_1(\bV^{(\iter)}-\bI)\|_F$ of
\begin{equation}\label{eq:converge_clean0}
\|\bbQ_*\times_1(\bV^{(\iter+1)}-\bI)\|_F\lesssim \kappa_H\|\bbQ_*\times_1(\bV^{(\iter)}-\bI)\|_F.\end{equation}
Recall that in the noiseless case, $\bW^{(\iter)}=\bW_*\bV^{(\iter)}$, we expect that $\|\cdot\|_d$ should be defined such that when $\bV$ is orthogonal and close to $\bI$, \begin{equation}\|\bW_*(\bV-\bI)\|_d \approx\|\bbQ_*\times_1(\bV-\bI)\|_F.\label{eq:metric_expect}\end{equation} Since the tangent plane of the set of orthogonal matrices at $\bI$ is the set of skew-symmetric matrices~\cite[Theorem 14.2.2]{Gallier2001}, the tangent space of $\{\bW_*\bV: \bV^T\bV=\bI\}$ at $\bV=\bI$ is $L_0=\{\bW_*\bY: \bY\in\mathrm{Skew}_M\}\subseteq\reals^{L\times M}$, \eqref{eq:metric_expect} implies that for any $\bW_*\bY\in L_0$, $\|\cdot\|_d$ should be defined such that $\|\bW_*\bY\|_d=\lambda\|\bbQ_*\times_1\bY\|_F=\lambda\|\bbP_*\times_1\bW_*\bY\|_F$ for some constant $\lambda>0$.  Combining this metric on $L_0$ and the standard Euclidean/Frobenius metric on the orthogonal subspace $L_0^\perp$, we define the metric $\|\cdot\|_{d}: \reals^{L\times M}\rightarrow \reals$ by
\begin{equation}\label{eq:metric_define}
\|\bW\|_{d}=\sqrt{(\lambda \|\bbP_*\times_1 P_{L_0}\bW\|_F)^2+\|P_{L_0^\perp}\bW\|_F^2},
\end{equation}
Here, $\lambda$ balances the weights from the two components, so that
\bes
\lambda= \lkr { \min_{\bY\in\mathrm{Skew}_M,\,\|\bY\|_F=1}\|\bbQ_*\times_1\bY\|_F}\rkr^{-1},
\ees
and the projection operators  can be explicitly written as 
\begin{align*} 
P_{L_0}\bW & =\bW_*(\bW_*^T\bW^{(\iter)}-\bW^{(\iter)\,T}\bW_*)/2, \\
P_{L_0^\perp}\bW & =\bW_*(\bW_*^T\bW^{(\iter)}+\bW^{(\iter)\,T}\bW_*)/2+
(\bI-\bW_*\bW_*^T)\bW^{(\iter)}. 
\end{align*}
By the definition of the metric $\|\cdot\|_d$ in \eqref{eq:metric_define}, we have the following equivalence between $\|\cdot\|_d$ and $\|\cdot\|_F$:
\[
\|\bW\|_d\geq \sqrt{\|P_{L_0}\bW\|_F^2+\|P_{L_0^\perp}\bW\|_F^2}=\|\bW\|_F
\]
and for  $C_H={{ \displaystyle \max_{{\bY\in\mathrm{Skew}_M,\,\|\bY\|_F=1}}\, \|\bbQ_*\times_1\bY\|_F}/ 
\displaystyle \min_{{\bY\in\mathrm{Skew}_M,\,\|\bY\|_F=1}}\, \|\bbQ_*\times_1\bY\|_F}$,
\[
\|\bW\|_d\leq \sqrt{C_H^2\|P_{L_0}\bW\|_F^2+C_H^2\|P_{L_0^\perp}\bW\|_F^2} =C_H\|\bW\|_F.
\]

Before stating our main result, we introduce two additional parameters:
\[
\kappa_1=\frac{p_{\max}^2n^2{L}}{\|\bbQ_*\|_F^2}, \quad \kappa_2=\frac{\sqrt{\dot{K}\, p_{\max}L}\ n}
{\minM\, \sigma_{K_m}(\bbQ_*(m,:,:))}.
\]

Both parameters are greater than $1$, and describe how ``well-conditioned'' $\bbQ_*$ is, and when $\bbQ_*$ is well-conditioned, then all parameters are close to $1$. In particular, $\kappa_1\geq 1$ because $\|\bbQ_*\|_F=\|\bbP_*\|_F$ and all elements of $\bbP_*$ are bounded by $p_{\max}$, and $\kappa_2\geq 1$ because
\[
\dot{K}\min_{m=1, \cdots, M}\sigma_{K_m}^2(\bbQ_*(m,:,:))\leq \sum_{m=1}^M\sum_{k=1}^{K_m}\sigma_{k}^2(\bbQ_*(m,:,:))=\|\bbQ_*\|_F^2\leq p_{\max}^2n^2L.
\]
When $\bbQ_*(m,:,:)$ is ``degenerate'' in the sense that $\sigma_{K_m}(\bbQ_*(m,:,:))\approx 0$, then $\kappa_2$ is large.

The main result in this step states that, if the noise $\bbDelta$ is small and when  Algorithm~\ref{alg:update} is applied to the observed adjacency tensor
$\bbA$ with a good initialization $\bW^{(0)}$,   the estimations are likely to improve over each iteration,
and Algorithm~\ref{alg:update} converges to $\bW_*$ approximately. The statement is as follows, and its proof is rather complicated and deferred to Section~\ref{sec:proof_lemma1}.

\begin{lem}[Step 1: A deterministic result on Algorithm~\ref{alg:update}]\label{lemma:lemma_main}
For
 \begin{align*}
 a_1=& 6\kappa_0 \frac{\sqrt{M}\lkr  \maxM  \|\Pi_{T,K_m}(\bbDelta)\times_{2,3}\Pi_{T,K_m}(\bbDelta)\|+2\maxM \|\Pi_{T,K_m}(\bbQ_*)\times_{2,3}\Pi_{T,K_m}(\bbDelta)\| \rkr }{\sigma_{M}(\bbQ_*\times_{2,3} \bbQ_*)}\\
 +&\frac{192\kappa_0\dot{K}\|\bbDelta\|^2(\|\bbQ_*\|+\|\bbDelta\|)}{\sqrt{M} \lkr \minM \sigma_{K_m}(\bbQ_*(m,:,:))\rkr \ \sigma_{M}(\bbQ_*\times_{2,3} \bbQ_*)},\\
%
 a_2=& \frac{192\kappa_0\dot{K}\|\bbQ_*\|^2(\|\bbQ_*\|+\|\bbDelta\|)}{\sqrt{M}\minM \sigma_{K_m}(\bbQ_*(m,:,:))\  \sigma_{M}(\bbQ_*\times_{2,3} \bbQ_*)}+6\kappa_0, 
  \end{align*}
if $\|\bbDelta\|\leq \frac{1}{4} \minM \sigma_{K_m}(\bbQ_*(m,:,:))$,
 \begin{equation}\label{eq:assumption1}
 \frac{2C_Ha_1}{1-\kappa_H} \leq \min\Big(\frac{1-\kappa_H}{2C_H(a_2+80\kappa_0^2+32\kappa_0^3)},\frac{\min_{m=1,\cdots,M}\sigma_{K_m}(\bbQ_*(m,:,:))}{4\|\bbQ_*\|}\Big)
  \end{equation}
 and the initialization satisfies
 \begin{equation}\label{eq:assumption2}
\|\bW^{(1)}-\bW_*\|_d\leq \min \lkr \frac{1-\kappa_H}{2C_H(a_2+80\kappa_0^2+32\kappa_0^3)},\frac{\min_{m=1,\cdots,M}\sigma_{K_m}(\bbQ_*(m,:,:))}
{4\|\bbQ_*\|} \rkr,
\end{equation}
 then for all $\iter\geq 1$, \begin{equation}\label{eq:monotone0}
\|\bW^{(\iter+1)}-\bW_*\|_d\leq\frac{1+\kappa_H}{2}\|\bW^{(\iter)}-\bW_*\|_d+C_Ha_1,
\end{equation}
which implies
 \[
\lim_{\iter\rightarrow\infty}\| \bW^{(\iter)}-\bW_*\|_d\leq \frac{2C_Ha_1}{1-\kappa_H}.
 \]
\end{lem}

\subsubsection{Step 2: Probabilistic estimation}
Since $\bbQ_*$ is deterministic in our model, we only need to estimate the terms that depend on $\bbDelta$ in Lemma~\ref{lemma:lemma_main}. The estimations are summarized as follows, and the proof is deferred to the Appendix.\\

\begin{lem}[Step 2: Probabilistic estimation]\label{lemma:prob_main}
(a) [Restatement of \cite[Theorem 1.2]{zhou2019sparse}]
If $p_{\max}\geq \frac{c\log(\max(n,L))}{\max(n,L)}$ for some constant $c>0$, then for any $r>0$, there exists a constant $C>0$ depending only on $r, c$ such that with probability at least $1-n^{-r}$, $\|\bbDelta\|=\sup_{\bu\in\reals^{L}, \bv\in\reals^{n}}\frac{\bbDelta\times_{1}\bu\times_2\bv\times_3\bv}{\|\bu\|\|\bv\|^2}$ satisfies $\|\bbDelta\|\leq C\sqrt{p_{\max}\max(n,L)}\log(\max(n,L))$.\\
(b) For any $t>0$,
\begin{align}\label{eq:prob_b}
\Pr\left(\Big\|\Pi_{T,K_m}(\bbQ_*)\times_{2,3}\Pi_{T,K_m}(\bbDelta)\Big\|_F>3tnL\sqrt{p_{\max}^3}\right)\leq 2\dot{K}L\exp\left(\frac{-\frac{1}{2}t^2}{1+\frac{t}{3\sqrt{p_{\max}ng_{\min}}}}\right).
\end{align}
(c) For any $t>0$,
\begin{align}\label{eq:prob_c}
&\Pr\left(\max_{1\leq m\leq M}\|\Pi_{T,K_m}(\bbDelta)\times_{2,3}\Pi_{T,K_m}(\bbDelta)\|\geq 9K^2t^2{p_{\max} \max(n,L)}\right)\\\leq& 2K(L+n)\exp\left(-\frac{t^2/2}{1+\frac{t}{\sqrt{p_{\max} \max(n,L)g_{\min}}}}\right),\nonumber
\end{align}
where $g_{\min}$  is the size of the smallest community.
\end{lem}

\subsubsection{Step 3: A probabilistic result on Algorithm~\ref{alg:update} without Assumptions \textbf{(A2)}-\textbf{(A4)}}

From the definition of $\kappa_0=\frac{\sigma_1(\bbQ_*\times_{2,3}\bbQ_*)}{\sigma_M(\bbQ_*\times_{2,3}\bbQ_*)}$ and the fact $\bbQ_*\times_{2,3}\bbQ_*=\calM_1(\bbQ_*)\calM_1(\bbQ_*)^T$, we have
\[
\kappa_0=\left(\frac{\sigma_1(\bbQ_*\times_{2,3}\bbQ_*)}{\sigma_M(\bbQ_*\times_{2,3}\bbQ_*)}\right)=\left(\frac{\sigma_1(\calM_1(\bbQ_*))}{\sigma_M(\calM_1(\bbQ_*))}\right)^2,
\]
$
\|\bbQ_*\times_1\! \bY\|_F\!=\!\!\|\calM_1(\bbQ_*)\bY\|_F\!\leq\! \sigma_1\!(\calM_1(\bbQ_*))\|\!\bY\!\|_F,$
 and
$
\|\bbQ_*\times_1\!\bY\!\|_F\!=\!\!\|\calM_1(\bbQ_*\!)\bY\|_F\!\!\geq\! \sigma_M(\calM_1(\bbQ_*\!))\|\bY\|_F$.  As a result, $C_H\leq \sqrt{\kappa_0}$.

Combining $C_H\leq \sqrt{\kappa_0}$ with  $\|\bbQ_*\|\leq \|\bbQ_*\|_F\leq p_{\max}n\sqrt{L}$, one obtains
\[
\sigma_M(\bbQ_*\times_{2,3}\bbQ_*)\geq \frac{1}{\kappa_0}\sigma_1(\bbQ_*\times_{2,3}\bbQ_*)=\frac{1}{\kappa_0}\|\calM_1(\bbQ_*)\|^2\geq \frac{1}{M\kappa_0}\|\calM_1(\bbQ_*)\|_F^2=\frac{1}{M\kappa_0}\|\bbQ_*\|_F^2\geq \frac{p_{\max}^2n^2L}{M\kappa_0\kappa_1},
\]
Then, Lemma~\ref{lemma:lemma_main} and Lemma~\ref{lemma:prob_main} imply the following statement.

\begin{thm}[Step 3: A generic result on Algorithm~\ref{alg:update} without Assumptions \textbf{(A2)}-\textbf{(A4)}]\label{thm:main0}
If \begin{equation}\label{eq:assumption122}\text{$p_{\max}\geq \frac{c\log(\max(n,L))}{\max(n,L)}$ for some constant $c>0$,}\end{equation} then for any $r>0$, there exists $C>0$ that depending only on $r, c$ such that for
\begin{align*}
a_1=&C\kappa_0^2\kappa_1\sqrt{M^{3}}\Big(\frac{t}{n\sqrt{p_{\max}}}+\frac{\dot{K}^2t^2}{p_{\max}n\min(n,L)}\Big)+{C\kappa_0^2\kappa_1\kappa_2\sqrt{M\dot{K}}}\frac{\log^2(\max(n,L))}{{p_{\max}}n\min(n,L)}\\
a_2=&{C\kappa_0^2\kappa_1\kappa_2\sqrt{M\dot{K}}}\Big(1+\frac{\log(\max(n,L))}{\sqrt{p_{\max}n\min(n,L)}}\Big),
\end{align*}
if
 \begin{equation}\label{eq:assumption12}
 C\kappa_2\sqrt{{\dot{K}}} \leq \sqrt{{p_{\max}n\min(n,L)}}, \,\,\,\sqrt{4a_1\kappa_0(a_2+112\kappa_0^3)}\leq 1-\kappa_H,\,\,\, 2\kappa_0\kappa_2a_1\sqrt{\dot{K}}\leq (1-\kappa_H)
  \end{equation}
 and the initialization satisfies
 \begin{equation}\label{eq:assumption22}
\|\bW^{(1)}-\bW_*\|_F\leq \min\Big(\frac{1-\kappa_H}{2\kappa_0(a_2+112\kappa_0^3)},\frac{1}{4\kappa_2\sqrt{\kappa_0\dot{K}}}\Big)\end{equation}
 then  with probability at least
 \[
1-n^{-r}-2\dot{K}L\exp\left(\frac{-\frac{1}{2}t^2}{1+\frac{t}{3\sqrt{p_{\max}ng_{\min}}}}\right)- 2K(L+n)\exp\left(-\frac{t^2/2}{1+\frac{t}{\sqrt{p_{\max} \max(n,L)g_{\min}}}}\right),
 \]
 \begin{equation}\label{eq:monotone01}
\|\bW^{(\iter+1)}-\bW_*\|_d\leq\frac{1+\kappa_H}{2}\|\bW^{(\iter)}-\bW_*\|_d+a_1\sqrt{\kappa_0},
\end{equation}
 holds for all $\iter\geq 1$, which implies
 \[
\lim_{\iter\rightarrow\infty}\| \bW^{(\iter)}-\bW_*\|_F\leq \lim_{\iter\rightarrow\infty}\| \bW^{(\iter)}-\bW_*\|_d\leq\frac{2a_1\sqrt{\kappa_0}}{1-\kappa_H}.
 \]
\end{thm}

\subsubsection{Step 4: Simplification under Assumptions \textbf{(A2)}-\textbf{(A4)}}\label{sec:step4}

We will need to estimate the parameters $\kappa_1, \kappa_2$ under Assumptions \textbf{(A2)}-\textbf{(A4)}. Since
$\bbQ_*(m,:,:) = \sqrt{L_m}\bTe_m \bB_m \bTe_m^T$ and $\bTe_m^T\bTe_m=\diag(|G_{m,1}|,|G_{m,2}|,\cdots,|G_{m,K_m}|)$, we have $\sigma_{K_m}(\bbQ_*(m,:,:))\geq \sqrt{L_m}\sigma_{K_m}(\bB_m)\sigma_{K_m}(\bTe_m)^2\geq {c_1 \sqrt{c_3}}p_{\max}\frac{n\sqrt{L}}{K_m\sqrt{M}}\min_{m=1, \cdots, M}\sigma_{K_m}(\bB_m^0)$,
which suggests \[\kappa_2\leq\frac{K_{\max}\sqrt{M}}{c_1 \sqrt{c_3}b_1}.\]

Similarly, we have the estimation
\[
\kappa_1\leq \frac{1}{c_1c_3^2b_2}.
\]
Let $p^*=p_{\max}n\min(n,L)$ and $t=\log(n+L)$, then we have the estimation that
\begin{align*}
a_1=&\frac{C\kappa_0^2\sqrt{M^3}}{c_1c_3^2b_2}\Big(\frac{t}{n\sqrt{p_{\max}}}+\frac{\dot{K}^2t^2}{p_{\max}n\min(n,L)}\Big)+{C\kappa_0^2\frac{\sqrt{\dot{K}}K_{\max}M}{c_1^{1.5}c_3^3b_1b_2}}\frac{\log^2(\max(n,L))}{p_{\max}n\min(n,L)}\\\leq &C\kappa_0^2\log^2(n+L)\left(\frac{1}{\sqrt{p_{\max}n^2}}\frac{\sqrt{M^3}}{c_1c_3^2b_2}+\frac{1}{p^*}\Big(\frac{\dot{K}^2\sqrt{M^3}}{c_1c_3^2b_2}+\frac{\sqrt{\dot{K}}K_{\max}M}{c_1^{1.5}c_3^3b_1b_2}\Big)\right)\\\leq& C\kappa_0^2\log^2(n+L)\sqrt{M^3}\left(\sqrt{\frac{1}{p_{\max}n^2}}+\frac{\dot{K}^2}{p^*}\right)\\
a_2=&C\kappa_0^2\frac{\sqrt{\dot{K}}K_{\max}M}{c_1^{1.5}c_3^3b_1b_2}\Big(1+\frac{\log(\max(n,L))}{\sqrt{p_{\max}n\min(n,L)}}\Big)\leq C\kappa_0^3\log(n+L)\frac{1}{\sqrt{p^*}}\frac{\sqrt{\dot{K}}K_{\max}M}{c_1^{1.5}c_3^3b_1b_2}\\
\leq &C\kappa_0^2\log(n+L)\sqrt{\frac{\dot{K}^3}{p^*}},
\end{align*}
By calculation, a sufficient condition for the requirement in \eqref{eq:assumption12} becomes \eqref{eq:assumption13}.

With \eqref{eq:assumption13}, we have $\frac{p^*(1-\kappa_H)}{\kappa_0^2\log(n+L)\dot{K}^3}\geq \frac{\sqrt{M}}{\dot{K}^2}$. Combining it with $a_2\geq{C\kappa_0^2\kappa_1\kappa_2\sqrt{M\dot{K}}}\Big(1+\frac{\log(\max(n,L))}{\sqrt{p^*}}\Big)$, the assumption on the initialization \eqref{eq:assumption22} can be guaranteed by \eqref{eq:assumption23}. In addition, \eqref{eq:assumption122} and \eqref{eq:assumption12} follow from \eqref{eq:assumption13}. Then \eqref{eq:hatW0} is proved by applying Theorem~\ref{thm:main0}, and \eqref{eq:hatW} follows from \eqref{eq:hatW0}. The proof of \eqref{eq:hatQ} is presented in \eqref{eq:hatQ1} at the end of Section~\ref{sec:convergence}. 



\subsection{Proof of Theorem~\ref{thm:classification}}
\label{sec:proof_classification}

\begin{proof}[Proof of Theorem~\ref{thm:classification}]
(a) For completeness, we will first write down the statement from \cite[Lemma C.1]{lei2020tail}:
\\

\noindent
\textit{Let $\bU$ be an $n\times d$ matrix with $K$ distinct rows with minimum pairwise Euclidean norm separation $\gamma$. Let $\hat{\bU}$ be another $n\times d$ matrix
and $(\hat{\bTe}, \hat{\bX})$ be an $(1+\epsilon)$-approximate solution to K-means problem with input $\hat{\bU}$, then the
number of errors in $\hat{\bTe}$ as an estimate of the row clusters of $\bU$ is no larger than $C_{\epsilon}\|\bU-\hat{\bU}\|_F^2\gamma^{-2}$
for some constant $C_\epsilon$ depending only on $\epsilon$.}
\\

\noindent
Note that $\bW_*$ is an $L\times M$ matrix with $M$ distinct rows with minimum pairwise Euclidean norm separation larger than $2/\max_{m=1,\cdots,M}\sqrt{L_m}$, the misclassification rate is not larger than
\[
\frac{\max_{m=1,\cdots,M}L_m}{4L} C_{\epsilon}\|\bW_*-\hat{\bW}\|_F^2.
\]
Combining it with the estimation of $\|\bW_*-\hat{\bW}\|_F^2$ and assumption \textbf{(A3)} on $\frac{\max_{m=1,\cdots,M}L_m}{L}$, part (a) is proved.

(b) Denote the orthogonal matrix of size $n\times K_m$ whose columns are the top $K_m$
eigenvectors of $\widehat{\bbQ}(m,:,:)$ by $\hat{\bU}_m$ and the  orthogonal matrix of size $n\times K_m$ whose columns are the top $K_m$
eigenvectors of ${\bbQ}_*(m,:,:)$ by $\bU_m$, then the Davis-Kahan theorem implies that
\[
\|\hat{\bU}_m-{\bU}_m\|_F\leq \frac{\|\widehat{\bbQ}(m,:,:)-\bbQ_*(m,:,:)\|_F}{\sigma_{K_m}(\bbQ_*(m,:,:))}.
\]
In addition, ${\bU}_m$ has $K_m$ distinct rows with minimum pairwise Euclidean norm separation at least $2/\sqrt{g_{m,\max}}$, where $\displaystyle{g_{m,\max}=\max_{1\leq k\leq K_m}\, |G_{m,k}|}$. As a result, \eqref{eq:hatQ} implies that the misclassification rate is bounded by
\begin{align*}
&\frac{g_{m,\max}}{4n}C_{\epsilon}\frac{\|\widehat{\bbQ}(m,:,:)-\bbQ_*(m,:,:)\|_F^2}{\sigma_{K_m}(\bbQ_*(m,:,:))^2}\\
\leq &
\frac{1}{4K_m}C_{\epsilon}\frac{\|\widehat{\bbQ}(m,:,:)-\bbQ_*(m,:,:)\|_F^2}{\sigma_{K_m}(\bbQ_*(m,:,:))^2}\leq C_{\epsilon}\frac{\|\widehat{\bbQ}(m,:,:)-\bbQ_*(m,:,:)\|^2}{\sigma_{K_m}(\bbQ_*(m,:,:))^2}\\
\leq & C_{\epsilon}\frac{(2\|\bbQ_*\|\|\hat{\bW}-\bW_*\|_F+2\|\bbDelta\|)^2}{\sigma_{K_m}(\bbQ_*(m,:,:))^2}\leq C_{\epsilon}\frac{(2p_{\max}n\sqrt{L}\|\hat{\bW}-\bW_*\|_F+2\|\bbDelta\|)^2}{\sigma_{K_m}(\bbQ_*(m,:,:))^2}
\\\leq& C_{\epsilon}(\frac{\kappa_2^2}{\dot{K}})\Bigg(\|\hat{\bW}-\bW_*\|_F^2+\Big(\frac{\sqrt{p_{\max}\max(n,L)}\log(\max(n,L))}{p_{\max}n\sqrt{L}}\Big)^2\Bigg)\\
\leq &C_{\epsilon}(\frac{K_{\max}^2 M}{\dot{K}})\Bigg(\|\hat{\bW}-\bW_*\|_F^2+\Big(\frac{\sqrt{p_{\max}\max(n,L)}\log(\max(n,L))}{p_{\max}n\sqrt{L}}\Big)^2\Bigg)\\
\leq &C_{\epsilon}(\frac{K_{\max}^2 M}{\dot{K}})\Bigg(\Big(\log^2(n+L)\sqrt{M^3}\frac{\sqrt{\kappa_0}}{1-\kappa_H}\Big(\sqrt{\frac{1}{p_{\max}n^2}}+\frac{\dot{K}^2}{p_{\max}n\min(n,L)}\Big)\Big)^2\\&+\Big(\frac{\sqrt{p_{\max}\max(n,L)}\log(\max(n,L))}{p_{\max}n\sqrt{L}}\Big)^2\Bigg)\\
\leq &C_{\epsilon}K_{\max}\Bigg(\log^4(n+L){M^3}\frac{{\kappa_0}}{(1-\kappa_H)^2}\Big({\frac{1}{p_{\max}n^2}}+\frac{\dot{K}^4}{p_{\max}^2n^2\min(n,L)^2}\Big)+\frac{(n+L)\log(n+L)^2}{n^2{p_{\max}L}}\Bigg)
\end{align*}
\end{proof}


\subsection{Proof of Lemma~\ref{lemma:lemma_main}}\label{sec:proof_lemma1}
\begin{proof}[Proof of Lemma~\ref{lemma:lemma_main}]
The main idea of the proof of Lemma~\ref{lemma:lemma_main} is as follows. Assumption \textbf{(A1)} implies that, when the observation is noise-free in the sense that $\bbA=\bbP_*$, Algorithm~\ref{alg:update} converges linearly. As a result, we only need to show that the output of the algorithm does not change much if we replace $\bbA$ with $\bbP_*$ and Algorithm~\ref{alg:update} with its linear approximation $P_{L_2}P_{L_1}$.

Given $\bW^{(\iter)}$, we construct a skew-symmetric matrix $\bY^{(\iter)}\in\reals^{M\times M}$ by
\begin{equation}\label{eq:yiter}
\bY^{(\iter)}=\frac{1}{2}(\bW_*^T\bW^{(\iter)}-\bW^{(\iter)\,T}\bW_*),
\end{equation}
and then the update formula for the ``clean'' version of the algorithm is the solution to the equation
\begin{equation}\label{eq:tildeY}
\bbQ_*\times_1\hat{\bY}^{(\iter+1)}=P_{L_2}P_{L_1}(\bbQ_*\times_1\bY^{(\iter)}).
\end{equation}
Intuitively, $\hat{\bY}^{(\iter+1)}$ is the algorithmic update when $\bbA_*$ is replaced by $\bbP_*$, and  Algorithm~\ref{alg:update} is replaced by its linear approximation $P_{L_2}P_{L_1}$. By the definition of $\kappa_H$ in \eqref{eq:kappa_H}, we have \begin{equation}\|\bW_*\hat{\bY}^{(\iter+1)}\|_d\leq \kappa_H\|\bW_*{\bY}^{(\iter)}\|_d\label{eq:converge_bY}.\end{equation}
We will bound $\|\bW^{(\iter+1)}-\bW_*\|_d$ as a function of $\|\bW^{(\iter)}-\bW_*\|_d$ by \eqref{eq:converge_bY} and the following perturbation bounds  in Lemma~\ref{lemma:pertubationbounds}, with its proof deferred to Section~\ref{sec:lemma:pertubationbounds}.

\begin{lem}\label{lemma:pertubationbounds}
For any $\bW^{\iter}$, let $\bY^{\iter}$ and $\hat{\bY}^{\iter+1}$ be defined as in \eqref{eq:yiter} and  \eqref{eq:tildeY}, then
\begin{enumerate}
\item \begin{equation}\|\bW_*\bY^{(\iter)}\|_d\leq \|\bW^{(\iter)}-\bW_*\|_d.\label{eq:lemma1_bound1}\end{equation}
\item For \begin{equation}\tilde{\bW}^{(\iter+1)}=\Pi_o(\bbP_*\times_{2,3} (\bbQ_*+\Pi_{T,\bK} (\bbQ_*\times_1 \bY^{(\iter)}))),\label{eq:tildeW}\end{equation} we have
\begin{align}\nonumber
&\|\tilde{\bW}^{(\iter+1)}-\bW_* (\hat{\bY}^{(\iter+1)}+\bI)\|_F\leq \frac{32\kappa_0^2\|\bY^{\iter}\|_F^2}{2-8\kappa_0^2\|\bY^{\iter}\|_F-4\kappa_0\|\bY^{\iter}\|_F(2+4\kappa_0\|\bY^{\iter}\|_F)}\\&+\Big(1+\frac{2\kappa_0+8\kappa_0\|\bY^{\iter}\|_F}{2-8\kappa_0\|\bY^{\iter}\|_F-\kappa_0(1+4\kappa_0\|\bY^{\iter}\|_F)4\kappa_0\|\bY^{\iter}\|_F}\Big)(4\kappa_0\|\bY^{\iter}\|_F)^2.\label{eq:lemma1_bound2}
\end{align}
\item If 
\begin{equation}
\|\bW_*-\bW^{(\iter)}\|_F\leq \frac{\min_{m=1,\cdots,M}\sigma_{K_m}(\bbQ_*(m,:,:))}{\|\bbQ_*\|},
\label{eq:recursive}\end{equation} 
then, 
\begin{align}
&\|{\bW}^{(\iter+1)}-\tilde{\bW}^{(\iter+1)}\|_F\leq \frac{\beta_1}{\sigma_{M}(\bbQ_*\times_{2,3} \bbQ_*)-2\|\bbQ_*\times_{2,3}\bbQ_*\|\|\bY^{(\iter)}\|_F-\beta},\label{eq:lemma1_bound3}
\end{align}
where, for $a_1$ and $a_2$ are defined in Lemma~\ref{lemma:lemma_main},
\[
\beta_1=\frac{\sigma_M(\bbQ_*\times_{2,3} \bbQ_*)}{6}(a_1+a_2\|\bW_*-\bW^{(\iter)}\|_F^2)
\]
\end{enumerate}
\end{lem}
With the perturbation bounds \eqref{eq:lemma1_bound2} and \eqref{eq:lemma1_bound3} in Lemma~\ref{lemma:pertubationbounds}, we have
\begin{align*}
&\|\bW^{(\iter+1)}-\bW_*(\hat{\bY}^{(\iter+1)}+\bI)\|_F\\\leq&
\|{\bW}^{(\iter+1)}-\tilde{\bW}^{(\iter+1)}\|_F+\|\tilde{\bW}^{(\iter+1)}-\bW_*(\hat{\bY}^{(\iter+1)}+\bI)\|_F
\\\leq& \frac{ \beta_1}{\sigma_{M}(\bbQ_*\times_{2,3} \bbQ_*)-2\|\bbQ_*\times_{2,3}\bbQ_*\|\|\bY^{(\iter)}\|_F-\beta_1}\\&+\frac{32\kappa_0^2\|\bY^{\iter}\|_F^2}{2-8\kappa_0^2\|\bY^{\iter}\|_F-4\kappa_0\|\bY^{\iter}\|_F(2+4\kappa_0\|\bY^{\iter}\|_F)}\\
+&\Big(1+\frac{2\kappa_0+8\kappa_0\|\bY^{\iter}\|_F}{2-8\kappa_0\|\bY^{\iter}\|_F-\kappa_0(1+4\kappa_0\|\bY^{\iter}\|_F)4\kappa_0\|\bY^{\iter}\|_F}\Big)(4\kappa_0\|\bY^{\iter}\|_F)^2.
\end{align*}

When \begin{equation}\label{eq:cond1}a_1\leq 1,  \|\bW^{(iter)}-\bW_*\|_F\leq \min(\frac{1}{32\kappa_0^3},\frac{1}{\sqrt{a_2}}), \end{equation} we have (using $\|\bY^{(iter)}\|_F\leq \|\bW^{(iter)}-\bW_*\|_F$)
 $4\kappa_0\|\bY^{\iter}\|_F\leq 1$, $8\kappa_0^2\|\bY^{\iter}\|_F+4\kappa_0\|\bY^{\iter}\|_F(2+4\kappa_0\|\bY^{\iter}\|_F)<1$, \ $8\kappa_0\|\bY^{\iter}\|_F+\kappa_0(1+4\kappa_0\|\bY^{\iter}\|_F)4\kappa_0\|\bY^{\iter}\|_F<1$,
 which imply $\|\bbQ_*\times_{2,3}\bbQ_*\|\|\bY^{(\iter)}\|_F\leq \sigma_M(\bbQ_*\times_{2,3}\bbQ_*)/2$ and $\|\bW_*-\bW^{(\iter)}\|_F^2\leq 1$, and
\begin{align}\label{eq:converge_W_Y}
&\|\bW^{(\iter+1)}-\bW_*(\hat{\bY}^{(\iter+1)}+\bI)\|_F\\\nonumber\leq & \frac{ \beta_1}{\sigma_{M}(\bbQ_*\times_{2,3} \bbQ_*)(1-1/2-1/3)}+32\kappa_0^2\|\bY^{\iter}\|_F^2+(1+2\kappa_0+8\kappa_0\|\bY^{\iter}\|_F)(4\kappa_0\|\bY^{\iter}\|_F)^2\\\nonumber
=&\frac{6\beta_1}{\sigma_{M}(\bbQ_*\times_{2,3} \bbQ_*)}+(3+2\kappa_0+8\kappa_0\|\bY^{\iter}\|_F)(4\kappa_0\|\bY^{\iter}\|_F)^2\\\nonumber
\leq &\frac{6\beta_1}{\sigma_{M}(\bbQ_*\times_{2,3} \bbQ_*)}+(5+2\kappa_0)(4\kappa_0\|\bY^{\iter}\|_F)^2\\\nonumber
=&\kappa_0(a_1+a_2\|\bW_*^T\bW^{(\iter)}-\bI\|^2)+(5+2\kappa_0)(4\kappa_0\|\bY^{\iter}\|_F)^2.
\end{align}

Combining \eqref{eq:converge_bY}, \eqref{eq:lemma1_bound1}, and \eqref{eq:converge_W_Y}, we have 
\begin{align*}
& \|\bW^{(\iter+1)}-\bW_*\|_d  \leq \|\bW_*\hat{\bY}^{(\iter+1)}\|_d+\|\bW^{(\iter+1)}-\bW_*(\hat{\bY}^{(\iter+1)}+\bI)\|_d \\
 & \leq  \kappa_H\|\bW_*{\bY}^{(\iter)}\|_d+\|\bW^{(\iter+1)}-\bW_*(\hat{\bY}^{(\iter+1)}+\bI)\|_d\\
 & \leq \kappa_H \|\bW_*-{\bW}^{(\iter)}\|_d+\|\bW^{(\iter+1)}-\bW_*(\hat{\bY}^{(\iter+1)}+\bI)\|_d\\
&\leq \!\kappa_H\|\bW^{(\iter)}\!\!-\!\bW_*\|_d+\!C_H\Big(a_1\!+\!a_2\|\bW_*-\bW^{(\iter)}\!-\!\bI\|_F^2\!+\!(5\!+\!2\kappa_0)(4\kappa_0\|\bY^{\iter}\|_F)^2\!\Big)  \\
& \leq    \kappa_H\|\bW^{(\iter)}-\bW_*\|_d+C_H\Big(a_1+a_2\|\bW_*-\bW^{(\iter)}\|_F^2
+16(5+2\kappa_0)\kappa_0^2\|\bW^{(\iter)}-\bW_*\|_d^2\Big) \\
& \leq    \kappa_H\|\bW^{(\iter)}-\bW_*\|_d+C_H\Big(a_1+(a_2+80\kappa_0^2+32\kappa_0^3)\|\bW^{(\iter)}-\bW_*\|_d^2\Big)
\end{align*}
As a result, if  in addition we have \begin{equation}\label{eq:cond2}C_H(a_2+80\kappa_0^2+32\kappa_0^3)\|\bW^{(\iter)}-\bW_*\|_d\leq (1-\kappa_H)/2,\end{equation} then
\begin{equation}\label{eq:monotone}
\|\bW^{(\iter+1)}-\bW_*\|_d\leq\frac{1+\kappa_H}{2}\|\bW^{(\iter)}-\bW_*\|_d+C_Ha_1.
\end{equation}
By the assumptions in \eqref{eq:assumption1} and \eqref{eq:assumption2}, the argument of induction implies that \eqref{eq:recursive}, \eqref{eq:cond1}, and \eqref{eq:cond2} hold for all $\iter\geq 1$. Therefore, \eqref{eq:monotone} holds for all $\iter\geq 1$ and the theorem is proved.
\\

\subsubsection{Proof of Lemma~\ref{lemma:pertubationbounds}}\label{sec:lemma:pertubationbounds}The proof of Lemma~\ref{lemma:pertubationbounds} is based on  Lemma~\ref{lemma:skew}-\ref{lemma:pertubation_o}. Among these lemmas, the proofs of Lemmas~\ref{lemma:skew}, ~\ref{lemma:perturb1}, ~\ref{lemma:pertubation_o} will be presented in Section~\ref{sec:auxillary_proof}, and Lemma~\ref{lemma:polar_perturbation} is a restatement of Theorem VII.5.1 in \cite{Bhatia1997}. We shall prove the three perturbation bounds \eqref{eq:lemma1_bound1}, \eqref{eq:lemma1_bound2}, and \eqref{eq:lemma1_bound3} separately.
\begin{lem}\label{lemma:skew}
Given any symmetric matrix $\bQ\in\reals^{M\times M}$, if $\bQ\bY-\bX=\bS$ holds for a symmetric matrix $\bS\in\reals^{M\times M}$ and a skew symmetric matrix $\bY\in\mathrm{Skew}_M$, then we have
\[
\max(\|\bQ\bY\|_F, \|\bS\|_F)\leq 2\|\bX\|_F.
\]
\end{lem}

\begin{lem}\label{lemma:polar_perturbation}
For any square matrices $\bA$ and $\bB$,
\[
\|\Pi_o(\bA)-\Pi_o(\bB)\|_F\leq2\frac{\|\bA-\bB\|_F}{\sigma_{\min}(\bA)+\sigma_{\min}(\bB)}.
\]
The inequality also holds if the operator norm is replaced with Frobenius norm.
\end{lem}

\begin{lem}\label{lemma:perturb1}
For any positive definite matrix $\bQ\in\reals^{M\times M}$ and any skew symmetric matrix $\bY\in\mathrm{Skew}_M$ with $\|\bY\|\leq 1$, we have
\[
\|\Pi_o(\bQ(\bI+\bY)-(\bI+\bY)\|_F\leq \Big(1+\frac{2\|\bQ\|}{2\sigma_{\min}(\bQ)-e\|\bQ\|\|\bY\|}\Big)(e-2)\|\bY\|_F^2.
\]
\end{lem}

\begin{lem}\label{lemma:pertubation_o}
Let $\Pi_o$ be defined in \eqref{eq:Pi_o}.  Then
\[
\|\Pi_o(\bX+\bY)-\Pi_o(\bX)\|_F\leq (1+\sqrt{2})\frac{\|\bY\|_F}{\sigma_{M}(\bX)-\|\bY\|}, 
\]
where $\sigma_{M}(\bX)$ represents the $M$-th singular value of $\bX$.
\end{lem}

\begin{lem}\label{lemma:remainder}
For a symmetric matrix $\bX_0\in\reals^{n\times n}$ with rank $r$, let $\Pi_{T}: \reals^{n\times n}\rightarrow\reals^{n\times n}$ be projection onto the tangent space of $\{\bX: \rank(\bX)=r\}$ at $\bX_0$ and $\Pi_{T,\perp}$ be the remainder of the projection, then for any symmetric matrix $\bDelta$,
\[
\|\Pi_{T,\perp}(\bDelta)\|_F\leq \frac{\|\Pi_{T}(\bDelta)\|_F^2}{\sigma_r(\bX_0)-\|\bDelta\|}.
\]
\end{lem}

\noindent
\textbf{Proof of bound 1 in \eqref{eq:lemma1_bound1}}

It follows from the observation that $\bW_*\bY^{(\iter)}=P_{L_0}(\bW^{(\iter)}-\bW_*)$ and from the definition \eqref{eq:metric_define}, $\|\bW_*\bY^{(\iter)}\|_d=\lambda\|\bP_*\times \bW_*\bY^{(\iter)}\|_F$ and
\[
\|\bW^{(\iter)}-\bW_*\|_d=\sqrt{(\lambda\|\bP_*\times \bW_*\bY^{(\iter)}\|_F)^2+\|P_{L_0^\perp}(\bW^{(\iter)}-\bW_*)\|_F^2}.
\]


\noindent
\textbf{Proof of bound 2 in \eqref{eq:lemma1_bound2}}

By the definition of \eqref{eq:tildeY}, we have that for any $\bDelta\in\mathrm{Skew}_M$,
\begin{align*}
&0=\langle \bbQ_*\times_1 \hat{\bY}^{(\iter+1)}-P_{L_1}(\bbQ_*\times_1\bY^{(\iter)}), \bbQ_*\times_1 \bDelta\rangle\\=&\left\langle \bbQ_*\times_{2,3} \big(\bbQ_*\times_1 \hat{\bY}^{(\iter+1)}-P_{L_1}(\bbQ_*\times_1\bY^{(\iter)})\big), \bDelta\right\rangle.
\end{align*}
As a result, $\bbQ_*\times_{2,3}\big(\bbQ_*\times_1 \hat{\bY}^{(\iter+1)}-P_{L_1}(\bbQ_*\times_1\bY^{(\iter)})\big)=(\bbQ_*\times_{2,3}\bbQ_*)\hat{\bY}^{(\iter+1)}-\bbQ_*\times_{2,3} \Pi_{T,\bK} (\bbQ_*\times_1\bY^{(\iter)})$ is a symmetric matrix. Denoting it by $\bS$, Lemma~\ref{lemma:skew} implies that
\begin{align}\nonumber
&\max(\|(\bbQ_*\times_{2,3}\bbQ_*)\hat{\bY}^{(\iter+1)}\|_F, \|\bS\|_F)\leq 2\|\bbQ_*\times_{2,3} \Pi_{T,\bK} (\bbQ_*\times_1\bY^{(\iter)})\|_F\\\leq& 2\max_{1\leq m\leq M}\|\bbQ_*\times_{2,3}\Pi_{T,K_m}(\bbQ_*)\|\|\bY^{(\iter)}\|_F
\leq  4\|\bbQ_*\times_{2,3}\bbQ_*\|\|\bY^{(\iter)}\|_F,\label{eq:temp1}
\end{align}
where the last inequality is due to \[\bbQ_*\times_{2,3}\bbQ_*=\|\calM_1(\bbQ_*)\|^2, \|\bbQ_*\times_{2,3}\Pi_{T,K_m}(\bbQ_*)\|\leq \|\calM_1(\bbQ_*)\|\|\calM_1(\Pi_{T,K_m}(\bbQ_*))\|,\]
\[
\calM_1(\Pi_{T,K_m}(\bbQ_*))=\calM_1(\bbQ_*)(\Pi_{\bU_m}\otimes \bI)+\calM_1(\bbQ_*)(\Pi_{\bU_m^\perp}\otimes \Pi_{\bU_m}),
\]
where $\otimes$ represents the Kronecker product, and $\|\bA\otimes\bB\|\leq \|\bA\|\|\bB\|$.

As a result, \begin{align}\label{eq:temp2}\|\hat{\bY}^{(\iter+1)}\|_F\leq 4\kappa_0\|\bY^{(\iter)}\|_F.\end{align}


By the definitions of $\tilde{\bW}^{(\iter+1)}$ and $\bS$, we have $\bW_*^T \tilde{\bW}^{(\iter+1)} = \Pi_o(\bbQ_*\times_{2,3} (\bbQ_*+\Pi_{T,\bK} (\bbQ_*\times_1\bY^{(\iter)}) ))=\Pi_o((\bbQ_*\times_{2,3}\bbQ_*)(\hat{\bY}^{(\iter+1)}+\bI)-\bS)$. Lemma~\ref{lemma:polar_perturbation}, the upper bounds of $\|\bS\|_F$ in \eqref{eq:temp1}, and $\|\hat{\bY}^{(\iter+1)}\|_F$ in \eqref{eq:temp2} imply
\begin{align*}
&\|\bW_*\tilde{\bW}^{(\iter+1)}-\Pi_o((\bbQ_*\times_{2,3}\bbQ_*-\bS)(\hat{\bY}^{(\iter+1)}+\bI))\|_F\\=
&\|\Pi_o((\bbQ_*\times_{2,3}\bbQ_*)(\hat{\bY}^{(\iter+1)}+\bI)-\bS)-\Pi_o((\bbQ_*\times_{2,3}\bbQ_*-\bS)(\hat{\bY}^{(\iter+1)}+\bI))\|_F\\\leq& \frac{2\|\bS\hat{\bY}^{(\iter+1)}\|_F}{2\sigma_{\min}((\bbQ_*\times_{2,3}\bbQ_*)(\hat{\bY}^{(\iter+1)}+\bI))-\|\bS\|(2+\|\hat{\bY}^{(\iter+1)}\|)}\\
\leq &\frac{2\|\bS\hat{\bY}^{(\iter+1)}\|_F}{2\sigma_{\min}((\bbQ_*\times_{2,3}\bbQ_*)(\hat{\bY}^{(\iter+1)}+\bI))-\|\bS\|_F(2+\|\hat{\bY}^{(\iter+1)}\|)}\\\leq& \frac{2\|\bS\hat{\bY}^{(\iter+1)}\|_F}{2\sigma_{\min}(\bbQ_*\times_{2,3}\bbQ_*)-2\|(\bbQ_*\times_{2,3}\bbQ_*)\hat{\bY}^{(\iter+1)}\|_F-\|\bS\|_F(2+\|\hat{\bY}^{(\iter+1)}\|)}\\\leq& \frac{2(4\|\bbQ_*\times_{2,3}\bbQ_*\|\|\bY^{\iter}\|_F)(4\kappa_0\|\bY^{\iter}\|_F)}{2\sigma_{\min}(\bbQ_*\times_{2,3}\bbQ_*)-2\|\bbQ_*\times_{2,3}\bbQ_*\|(4\kappa_0\|\bY^{\iter}\|_F)-(4\|\bbQ_*\times_{2,3}\bbQ_*\|\|\bY^{\iter}\|_F)(2+4\kappa_0\|\bY^{\iter}\|_F)}\\
\leq &\frac{32\kappa_0^2\|\bY^{\iter}\|_F^2}{2-8\kappa_0^2\|\bY^{\iter}\|_F-4\kappa_0\|\bY^{\iter}\|_F(2+4\kappa_0\|\bY^{\iter}\|_F)}.
\end{align*}
In addition, Lemma~\ref{lemma:perturb1} implies
\begin{align*}
&\|\Pi_o((\bbQ_*\times_{2,3}\bbQ_*-\bS)(\hat{\bY}^{(\iter+1)}+\bI))-(\hat{\bY}^{(\iter+1)}+\bI)\|_F\\\leq&
\Big(1+\frac{2\|\bbQ_*\times_{2,3}\bbQ_*-\bS\|}{2\sigma_{\min}(\bbQ_*\times_{2,3}\bbQ_*-\bS)-e\|\bbQ_*\times_{2,3}\bbQ_*-\bS\|\|\hat{\bY}^{(\iter+1)}\|}\Big)(e-2)\|\hat{\bY}^{(\iter+1)}\|_F^2\\\leq&
\Big(1+\frac{2\|\bbQ_*\times_{2,3}\bbQ_*\|+2\|\bS\|_F}{2\sigma_{\min}(\bbQ_*\times_{2,3}\bbQ_*)-2\|\bS\|_F-(\|\bbQ_*\times_{2,3}\bbQ_*\|+\|\bS\|_F)\|\hat{\bY}^{(\iter+1)}\|}\Big)\|\hat{\bY}^{(\iter+1)}\|_F^2
\\\leq &\Big(1+\frac{2\kappa_0+8\kappa_0\|\bY^{\iter}\|_F}{2-8\kappa_0\|\bY^{\iter}\|_F-\kappa_0(1+4\kappa_0\|\bY^{\iter}\|_F)4\kappa_0\|\bY^{\iter}\|_F}\Big)(4\kappa_0\|\bY^{\iter}\|_F)^2
.\end{align*}
Combining the previous two estimations, part 2 is proved.
\\

\noindent
\textbf{Proof of bound 3 in \eqref{eq:lemma1_bound3}}

By the definition of  $\tilde{\bW}^{(\iter+1)}$ in \eqref{eq:tildeW} and ${\bW}^{(\iter+1)}=\Pi_o(\bbA\times_{2,3}\Pi_{\bK}(\bbA\times_1\bW^{(\iter)}))$, Lemma~\ref{lemma:pertubation_o} implies that
\[
\|\tilde{\bW}^{(\iter+1)}-{\bW}^{(\iter+1)}\|_F\leq \frac{\beta_1}{\sigma_{M}(\bbP_*\times_{2,3} \Pi_{T,\bK} (\bbQ_*\times_1 (\bI+\bY^{(\iter)})))-\beta_1}\leq \frac{\beta_1}{\sigma_{M}(\bbQ_*\times_{2,3} \bbQ_*)-\beta_2-\beta_1}.
\]
for \[\beta_1=\|\bbA\times_{2,3}\Pi_{\bK}(\bbA\times_1\bW^{(\iter)})-\bbP_*\times_{2,3} (\bbQ_*+\Pi_{T,\bK} (\bbQ_*\times_1 \bY^{(\iter)}))\|_F\] and \[\beta_2=\|\bbP_*\times_{2,3} \Pi_{T,\bK} (\bbQ_*\times_1 \bY^{(\iter)})\|.\] By the same calculation as in \eqref{eq:temp1}, we have \begin{equation}\label{eq:beta_2}\beta_2\leq \max_{1\leq m\leq M}\|\bbQ_*\times_{2,3}\Pi_{T,K_m}(\bbQ_*)\| \|\bY^{(\iter)}\|_F\leq 2\|\bbQ_*\times_{2,3}\bbQ_*\|\|\bY^{(\iter)}\|_F.\end{equation}

To estimate the upper bound of $\beta_1$, we note that $\beta_1\leq \beta_3+\beta_4+\beta_5$, where
\begin{align}\label{eq:4_01}
\beta_3=\|\bbA\times_{2,3}\Pi_{\bK}(\bbA\times_1\bW^{(\iter)})-\bbA\times_{2,3} (\bbQ_*+\Pi_{T,\bK}(\bbA\times_1\bW^{(\iter)}-\bbQ_*))\|_F,
\end{align}
\begin{align}\label{eq:4_02}
\beta_4=&\|\bbA\times_{2,3} (\bbQ_*+\Pi_{T,\bK}(\bbA\times_1\bW^{(\iter)}-\bbQ_*))-\bbP_*\times_{2,3}(\bbQ_*+ \Pi_{T,\bK} (\bbP_*\times_1(\bW^{(\iter)}-\bW_*)))\|_F\\
=&\|\bbDelta\times_{2,3}\Pi_{T,\bK}[\bbP_*\times_1\bW^{(\iter)}]+\bbDelta\times_{2,3}\Pi_{T,\bK}[\bbDelta\times_1 \bW^{(\iter)}]+\bbP_*\times_{2,3}\Pi_{T,\bK}[\bbDelta\times_1 \bW^{(\iter)}]\|_F,\nonumber
\end{align}
and
\begin{align}\label{eq:4_03}
\beta_5=&\|\bbP_*\times_{2,3}(\bbQ_*+ \Pi_{T,\bK} (\bbP_*\times_1(\bW^{(\iter)}-\bW_*)))-\bbP_*\times_{2,3} \Pi_{T,\bK} (\bbQ_*\times_1 (\bI+\bY^{(\iter)}))\|_F\\
=&\|\bbP_*\times_{2,3}\Pi_{T,\bK} (\bbQ_*\times_1(\bW_*^T\bW^{(\iter)}-\bI-\bY^{(\iter)}))\|_F\nonumber\\
=&\frac{1}{2}\|\bbQ_*\times_{2,3}\Pi_{T,\bK} (\bbQ_*\times_1(\bW_*^T\bW^{(\iter)}+\bW^{(\iter)\,T}\bW_*-2\bI))\|_F\nonumber\\
=&\frac{1}{2}\|\bbQ_*\times_{2,3}\Pi_{T,\bK} (\bbQ_*\times_1(\bW_*-\bW^{(\iter)})^T(\bW_*-\bW^{(\iter)})\|_F\nonumber
\end{align}


To obtain an upper bound for  $\beta_3$ in \eqref{eq:4_01}, let $\Pi_{T,\bK,\perp}\in\reals^{M\times n\times n}$ be $\Pi_{\bK}(\bbA\times_1\bW^{(\iter)})-(\bbQ_*+\Pi_{T,\bK}(\bbA\times_1\bW^{(\iter)}-\bbQ_*))$, then by Lemma~\ref{lemma:remainder} and \begin{align*}
    &\bbA\times_1\bW^{(\iter) }-\bbQ_*=\bbDelta\times_1\bW^{(\iter)} + \bbP_*\times_1(\bW^{(\iter) }-\bW-*),
\end{align*}
we have 
$\rank(\Pi_{T,\bK,\perp}(m,:,:))\leq {2} K_m$ and
\begin{align}\nonumber
&\|\Pi_{T,\bK,\perp}(m,:,:)\|\leq \frac{\|\Pi_{T,K_m}[\bbA\times_1\bW^{(\iter)}-\bbQ_*](m,:,:)\|^2}{\sigma_{K_m}(\bbQ_*(m,:,:))-\|\bbQ_*\|\|\bW_*-\bW^{(iter)}\|-\|\bbDelta\|}\\\leq&\nonumber \frac{4\|[\bbA\times_1\bW^{(\iter)}-\bbQ_*](m,:,:)\|^2}{\sigma_{K_m}(\bbQ_*(m,:,:))-\|\bbQ_*\|\|\bW_*-\bW^{(iter)}\|-\|\bbDelta\|}\\\leq&
\frac{4(\|\bbQ_*\|\|[\bW_*^T\bW^{(\iter)}-\bI](m,:)\|+\|\bbDelta\|)^2}{\sigma_{K_m}(\bbQ_*(m,:,:))-\|\bbQ_*\|\|\bW_*-\bW^{(iter)}\|-\|\bbDelta\|}\leq \frac{16(\|\bbQ_*\|^2\|[\bW_*^T\bW^{(\iter)}-\bI](m,:)\|^2+\|\bbDelta\|^2)}{\sigma_{K_m}(\bbQ_*(m,:,:))}.\label{eq:temp21}
\end{align}

Since $\rank(\Pi_{T,\bK,\perp}(m,:,:))\leq 2K_m$, we have
\[
\beta_3\leq \|\bbA\times_{2,3}\Pi_{T,\bK,\perp}(m,:,:)\|\leq 2K_m\|\bbA\| \|\Pi_{T,\bK,\perp}(m,:,:)\|.
\]
Summing $1\leq m\leq M$, \eqref{eq:temp21} yields
\begin{align*}
\beta_3=&\|\bbA\times_{2,3}\Pi_{T,\bK,\perp}\|_F\leq 16\|\bbA\| \sqrt{\sum_{m=1}^M \frac{K_m^2(\|\bbQ_*\|^2\|[\bW_*^T\bW^{(\iter)}-\bI](m,:)\|^2+\|\bbDelta\|^2)^2}{\big(\sigma_{K_m}(\bbQ_*(m,:,:))-\|\bbQ_*\|\|\bW_*-\bW^{(iter)}\|-\|\bbDelta\|\big)^2}}\\
\leq &16\|\bbA\|\sum_{m=1}^M {\frac{K_m(\|\bbQ_*\|^2\|[\bW_*^T\bW^{(\iter)}-\bI](m,:)\|^2+\|\bbDelta\|^2)}{\sqrt{M}\big(\sigma_{K_m}(\bbQ_*(m,:,:))-\|\bbQ_*\|\|\bW_*-\bW^{(iter)}\|-\|\bbDelta\|\big)}}
\\
\leq &16\|\bbA\| {\frac{\dot{K}(\|\bbQ_*\|^2\|\bW_*^T\bW^{(\iter)}-\bI\|^2+\|\bbDelta\|^2)}{\sqrt{M}\big(\min_{k=1}^M\sigma_{K_m}(\bbQ_*(m,:,:))-\|\bbQ_*\|\|\bW_*-\bW^{(iter)}\|-\|\bbDelta\|\big)}}\\\leq& 32 {\frac{\dot{K}(\|\bbQ_*\|+\|\bbDelta\|)(\|\bbQ_*\|^2\|\bW_*^T\bW^{(\iter)}-\bI\|^2+\|\bbDelta\|^2)}{\sqrt{M}\big(\min_{k=1}^M\sigma_{K_m}(\bbQ_*(m,:,:))\big)}}\\
\end{align*}

To find an upper bound for $\beta_4$ in \eqref{eq:4_02}, note that \begin{align*}\nonumber
&\|[\bbDelta\times_{2,3}\Pi_{T,\bK}(\bbDelta\times_1 \bW^{(\iter)})](:,m)\|=\|(\bbDelta\times_{2,3}\Pi_{T,K_m}(\bbDelta))[\bW^{(\iter)}](:,m)\|\\=&\|(\Pi_{T,K_m}(\bbDelta)\times_{2,3}\Pi_{T,K_m}(\bbDelta))[\bW^{(\iter)}](:,m)\|\leq \|\Pi_{T,K_m}(\bbDelta)\times_{2,3}\Pi_{T,K_m}(\bbDelta)\|,
\end{align*}
which implies
\begin{align}
\|\bbDelta\times_{2,3}\Pi_{T,\bK}(\bbDelta\times_1 \bW^{(\iter)})\|_F\leq
\max_{1\leq m\leq M}\|\Pi_{T,K_m}(\bbDelta)\times_{2,3}\Pi_{T,K_m}(\bbDelta)\|\sqrt{M}.\label{eq:4_1}
\end{align}

Similarly,
\begin{align}\nonumber
\|\bbP_*\times_{2,3}\Pi_{T,\bK}(\bbDelta\times_1 \bW^{(\iter)})\|_F\leq&
\max_{1\leq m\leq M}\|\Pi_{T,K_m}(\bbP_*)\times_{2,3}\Pi_{T,K_m}(\bbDelta)\|\sqrt{M},\\ \|\bbDelta\times_{2,3}\Pi_{T,\bK}(\bbP_*\times_1 \bW^{(\iter)})\|_F\leq&
\max_{1\leq m\leq M}\|\Pi_{T,K_m}(\bbP_*)\times_{2,3}\Pi_{T,K_m}(\bbDelta)\|\sqrt{M}.\label{eq:4_2}
\end{align}

As a result,
\[
\beta_4\leq \sqrt{M}\Big(\max_{1\leq m\leq M}\|\Pi_{T,K_m}(\bbDelta)\times_{2,3}\Pi_{T,K_m}(\bbDelta)\|+2\max_{1\leq m\leq M}\|\Pi_{T,K_m}(\bbP_*)\times_{2,3}\Pi_{T,K_m}(\bbDelta)\|\Big).
\]

To find an upper bound for  $\beta_5$ in \eqref{eq:4_03}, we use
\begin{align*}
\beta_5=&\frac{1}{2}\|\bbQ_*\times_{2,3}\Pi_{T,\bK}[\bbQ_*\times_1(\bW_*-\bW^{(\iter)})^T(\bW_*-\bW^{(\iter)})]\|_F\\
\leq &\frac{1}{2}\max_{1\leq m\leq M}\|\bbQ_*\times_{2,3}\Pi_{T,K_m}\bbQ_*\|\|(\bW_*-\bW^{(\iter)})^T(\bW_*-\bW^{(\iter)})\|_F\leq \|\bbQ_*\times_{2,3}\bbQ_*\|\|\bW_*-\bW^{(\iter)}\|_F^2,
\end{align*}
where the inequalities follow the same calculation as in \eqref{eq:temp1} and \eqref{eq:beta_2}, and $\|(\bW_*-\bW^{(\iter)})^T(\bW_*-\bW^{(\iter)})\|_F \colred{=}
\|\bW_*-\bW^{(\iter)}\|_F^2$.

Combining the estimations of $\beta_2$, $\beta_3$, $\beta_4$, $\beta_5$ with $\|\bW_*^T\bW^{(\iter)}-\bI\|\leq\|\bW_*^T\bW^{(\iter)}-\bI\|_F\leq \|\bW_*-\bW^{(\iter)}\|_F$, $\beta_1\leq \beta_3+\beta_4+\beta_5$, and $\kappa_0=\frac{\|\bbQ_*\times_{2,3}\bbQ_*\|}{\sigma_M (\bbQ_*\times_{2,3}\bbQ_*)}$, \eqref{eq:lemma1_bound3} is proved.
\end{proof}



\subsection{Proofs of auxiliary Lemmas and Propositions}
\label{sec:auxillary_proof}

\begin{proof}[Proof of Lemma~\ref{lem:assum_A1}]
For any $\bX\in\Skew_m$ such that $\bbQ_*\times_1\bX\in L_1$, due to
$$
[\bbQ_*\times_1\bX](m,:,:)=\sum_{m'=1}^M\bX(m,m')\bbQ_*(m',:,:)=\sum_{m'=1,m'\neq m}^M\bX(m,m')\bbQ_*(m',:,:),
$$
one has $\sum_{m'=1,m'\neq m}^M\bX(m,m')\Pi_{\bU_m^\perp}\bbQ_*(m',:,:)\Pi_{\bU_m^\perp}=0$.
When the first sufficient condition holds, then $\bX(m,m')=0$ for all $1\leq m'\leq M$.
Combining the analysis for all $1\leq m\leq M$, we have $\bX=0$. As a result, $L_1\cap L_2=\{0\}$ and \textbf{(A1)} holds.

The second sufficient condition follows from the first sufficient condition directly.
\end{proof}
\medskip
\begin{proof}[Proof of Lemma~\ref{lemma:prob_main}]
We first summarize a special case of  \cite[Theorem 2.1]{lei2020tail} as follows:

\begin{lem}\label{lemma:bernoulli_matrix}
If $\bX_l\in\reals^{n\times r}$, $l=1,\cdots,L$ are 
{independent}, elementwise sampled from a centered Bernoulli distribution with parameters not larger than $p$, then
\[
\Pr\left(\Big\|\sum_{l=1}^L\bX_l\Big\|\geq t\right)\leq 2(r+n)\exp\left(-\frac{t^2/2}{p L\max(n,r)+t}\right)
\]
\end{lem}

Since $\bU_m$ is a matrix of size $n\times K_m$ such that the $i$-th column is the normalized indicator vector of the set $G_{m,i}$, i.e., the indicator vector with scale $1/\sqrt{|G_{m,i}|}$.
As a result,
\begin{align*}
&\big\langle\bbQ_*(m_1,:,:),\bbDelta(l,:,:)\bU_m(k,:)\bU_m(k,:)^T\big\rangle=\sum_{j_1=1}^{n}\sum_{j_2\in G_{m,i}}\bbDelta(l,j_1,j_2')\frac{\sum_{j_2'\in G(m,i)}\bbQ_*(m_1,j_1,j_2')}{|G(m,i)|},
\end{align*}
and, by Bernstein's inequality, since  each term is no  larger than $\sqrt{L_m}p_{\max}$ and 
\[
\EE \left[\left(\bbDelta(l,j_1,j_2')\ \frac{\sum_{j_2'\in G(m,k)}\bbQ_*(m_1,j_1,j_2')}{|G(m,k)|}\right)^2\right]\leq L_mp_{\max}^3,
\]
\begin{align*}
\Pr\left(\Big|\big\langle\bbQ_*(m_1,:,:),\bbDelta(l,:,:)\bU_m(k,:)\bU_m(k,:)^T\big\rangle\Big|>t\sqrt{L_mp_{\max}^3n|G_{m,k}|}\right)\leq 2\exp\left(\frac{-\frac{1}{2}t^2}{1+\frac{t}{3\sqrt{p_{\max}n|G_{m,k}|}}}\right).
\end{align*}
Summing it over $1\leq i\leq K_m$, $1\leq l\leq L$, and $1\leq m_1\leq M$,  we proved \eqref{eq:prob_b}.

By definition, $\|\calM_1(\Pi_{T,K_m}(\bbDelta))\|\leq 3\sum_{k=1}^{K_m}\|\frac{1}{\sqrt{G_{m,k}}}\sum_{i\in G_{m,k}}\bbDelta(:,:,i)\|$, and Lemma~\ref{lemma:bernoulli_matrix} implies that
\[
\Pr\left(\|\sum_{i\in G_{m,k}}\bbDelta(:,:,i)\|\geq t\sqrt{p_{\max} \max(n,L)|G_{m,k}|}\right)\leq 2(L+n)\exp\left(-\frac{t^2/2}{1+\frac{t}{\sqrt{p_{\max} \max(n,L)|G_{m,k}|}}}\right).
\]
As a result,
\[
\Pr\left(\|\calM_1(\Pi_{T,K_m}(\bbDelta))\|\leq 3K_mt\sqrt{p_{\max} \max(n,L)}\right)\geq 2K_m(L+n)\exp\left(-\frac{t^2/2}{1+\frac{t}{\sqrt{p_{\max} \max(n,L)g_{\min}}}}\right)
\]
and \eqref{eq:prob_c} is then proved.
\end{proof}

\medskip
\begin{proof}[Proof of Lemma~\ref{lemma:skew}]
Without loss of generality,  assume that $\bQ$ is a diagonal matrix with $\bQ=\diag(q_1, \cdots, q_M)$. Then for each $1\leq i,j\leq M$, $q_i\bY_{i,j}-\bX_{i,j}=\bS_{i,j}$ and $q_j\bY_{j,i}-\bX_{j,i}=\bS_{j,i}$. Since $\bS_{i,j}=\bS_{j,i}$ and $\bY_{i,j}=-\bY_{j,i}$, we have
\[
\bY_{i,j}=\frac{\bX_{i,j}-\bX_{j,i}}{q_i+q_j},\,\,\,\bS_{i,j}=\frac{q_j\bX_{i,j}-q_i\bX_{j,i}}{q_i+q_j}.
\]
Since $\{q_i\}_{i=1}^M$ are positive, the lemma is proved.
\end{proof}

\begin{proof}[Proof of Lemma~\ref{lemma:perturb1}]
Let $\bU=\exp(\bY)$, then it is an orthogonal matrix. In addition, $\|\bU-\bI\|=\|\sum_{k=1}^{\infty}\frac{1}{k!}\bY^k\|\leq \sum_{k=1}^{\infty}\frac{1}{k!}\|\bY\|^k\leq (e-1)\|\bY\|$ and $\|\bU-(\bI+\bY)\|_F=\|\sum_{k=2}^{\infty}\frac{1}{k!}\bY^k\|_F\leq$\\$\|\bY\|_F\sum_{k=2}^{\infty}\frac{1}{k!}\|\bY\|^{k-2}\leq (e-2)\|\bY\|_F^2$.
As a result, Lemma~\ref{lemma:polar_perturbation} implies
\begin{align*}&
\|\Pi_o(\bQ(\bI+\bY))-(\bI+\bY)\|_F\leq \|\Pi_o(\bQ(\bI+\bY))-\Pi_o(\bQ\bU)\|_F+\|\Pi_o(\bQ\bU)-(\bI+\bY)\|_F\\\leq &\frac{2}{\sigma_{\min}(\bQ(\bI+\bY))+\sigma_{\min}(\bQ\bU)}\|\bQ(\bU-\bI-\bY)\|_F+\|\bU-\bI-\bY\|_F\\\leq & \big(1+\frac{2\|\bQ\|}{\sigma_{\min}(\bQ(\bI+\bY))+\sigma_{\min}(\bQ\bU)}\big)(e-2)\|\bY\|_F\\
\leq &\Big(1+\frac{2\|\bQ\|}{2\sigma_{\min}(\bQ)-e\|\bQ\|\|\bY\|}\Big)(e-2)\|\bY\|_F^2.
\end{align*}
\end{proof}
\begin{proof}[Proof of Lemma~\ref{lemma:pertubation_o}]
\begin{align*}
&\|\Pi_o(\bX+\bY)-\Pi_o(\bX)\|_F=\|(\bX+\bY)[(\bX+\bY)^T(\bX+\bY)]^{-0.5}-\bX[\bX^T\bX]^{-0.5}\|_F\\=&
\|\bY[(\bX+\bY)^T(\bX+\bY)]^{-0.5}\|_F+\|\bX\left([(\bX+\bY)^T(\bX+\bY)]^{-0.5}-[\bX^T\bX]^{-0.5}\right)\|_F\\\leq &\frac{\|\bY\|_F}{\sigma_{M}(\bX)-\|\bY\|}+\|\bX[\bX^T\bX]^{-0.5}[[(\bX+\bY)^T(\bX+\bY)]^{0.5}-[\bX^T\bX]^{0.5}][(\bX+\bY)^T(\bX+\bY)]^{-0.5}\|_F\\
\leq &\frac{\|\bY\|_F}{\sigma_{M}(\bX)-\|\bY\|}+\|[[(\bX+\bY)^T(\bX+\bY)]^{0.5}-[\bX^T\bX]^{0.5}]\|_F\|[(\bX+\bY)^T(\bX+\bY)]^{-0.5}\|\\
\leq &\frac{\|\bY\|_F}{\sigma_{M}(\bX)-\|\bY\|}+\sqrt{2}\|\bY\|_F\|[(\bX+\bY)^T(\bX+\bY)]^{-0.5}\|\\
\leq &\frac{(1+\sqrt{2})\|\bY\|_F}{\sigma_{M}(\bX)-\|\bY\|}.
\end{align*}
Here the first and the second inequalities follow from $\|\bA\bB\|_F\leq \|\bA\|_F\|\bB\|$, and the third inequality follows from \cite[(VII.39)]{Bhatia1997}.\\
\end{proof}

\begin{proof}[Proof of Lemma~\ref{lemma:remainder}]
Let $\bU\in\reals^{n\times r}$ be the orthogonal matrix that has the same column space as $\bX_0$, then
\[
\|\Pi_{T}\bDelta\|_F=\|\Pi_{\bU^\perp}\bDelta\Pi_{\bU}+\Pi_{\bU}\bDelta\|_F\geq \|\Pi_{\bU^\perp}\bDelta\Pi_{\bU}\|_F=\|\Pi_{\bU^\perp}\bDelta\bU\|_F,
\]
and
\[
\Pi_{T,\perp}\bDelta=\Pi_{\bU^\perp}\bDelta\bU(\bU^T(\bX_0+\bDelta)\bU)^{-1}\bU^T\bDelta^T\Pi_{\bU^\perp}.
\]
As a result,
\[
\|\Pi_{T,\perp}\bDelta\|_F\leq \frac{\|\Pi_{\bU^\perp}\bDelta\bU\|_F^2}{\sigma_r(\bU^T(\bX_0+\bDelta)\bU)}
\leq\frac{\|\Pi_{T}\bDelta\|_F^2}{\sigma_r(\bX_0)-\|\bDelta\|}
\]
\end{proof}

\medskip

\begin{proof}[Proof of Proposition~\ref{prop:initial}]
For $\bW^{(0)}$, Lemma 5 of \cite{jing2020community} gives the following theoretical guarantees.
Let $\bbP_*$ have Tucker ranks $(M,r,r)$ with decomposition
$\bbP_*=\bar{\bbC}\times_1\bar{\bW}\times_2\bar{\bU}\times_3\bar{\bU}$, where
$\bar{\bbC}\in\reals^{M\times r\times r}$, $\bar{\bW}\in\reals^{L\times M}$
and $\bar{\bU}\in\reals^{n\times r}$ are orthogonal matrices,
and $\bL=(L_1,\cdots,L_M)\in\reals^M$. Then, if $\delta=\max_{1\leq j\leq L}(\bfe_j\bar{\bU})$
is such that  $\delta=O(\sqrt{r/n})$, and $\sigma_r(\bar{\bbC}\times_3(\bL/L)^{1/2})\geq C\sqrt{np_{\max}}\log^2n$,
then, with probability at least $1-3n^{-2}$,
\begin{equation}\label{eq:init1}
\min_{\bO\in\reals^{M\times M}, \bO\bO^T=\bI}\|\bW^{(0)}-\bW_*\bO\|\leq
\min\left(\frac{C\sqrt{Mrnp_{\max}}\log^2n\log^2r}{\sigma_{\min}(\bar{\bbC})},2\right),
\end{equation}
where $\sigma_{\min}(\bar{\bbC})=\min\{\sigma_{\min}(\calM_j(\bar{\bbC})),j=1,2,3\}$.
In addition, \cite[Lemma C.1]{lei2020tail} shows that if $(1+\epsilon)$ $K$-means is applied,
the number of misclassified layers in the initial between-class clustering
$\cup_{m=1}^M\tilde{\calS}_m$ is bounded by
$C_\epsilon \min_{\bO\in\reals^{M\times M}, \bO\bO^T=\bI}\|\bW^{(0)}-\bW_*\bO\|_F^2\ \max_{1\leq m\leq M}|\calS_m|$,
which implies that with a permutation of the columns of $\bW^{(1)}$,
\begin{equation}\label{eq:init2}
\|\bW^{(1)}-\bW_*\|_F^2\leq C_\epsilon\min_{\bO\in\reals^{M\times M}, \quad
\bO\bO^T=\bI} \|\bW^{(0)}-\bW_*\bO\|_F^2\frac{\max_{1\leq m\leq M}L_m}{\min_{1\leq m\leq M}L_m}.
\end{equation}
Here $C_\epsilon$ represents a constant that depends on $\epsilon$ that might be different in different equations.
Combining \eqref{eq:init2} with the upper bound on $\|\bW^{(0)}-\bW_*\|$ in \eqref{eq:init1}, we obtain
\begin{equation}\label{eq:init3}
\|\bW^{(1)}-\bW_*\|_F\leq C_\epsilon \sqrt{M}\min\left(\frac{C\sqrt{Mrnp_{\max}}\,
\log^2n\log^2r}{\sigma_{\min}(\bar{\bbC})},2\right)\frac{\max_{1\leq m\leq M}L_m}{\min_{1\leq m\leq M}L_m}.
\end{equation}
\end{proof}


 \bibliography{citations}

\begin{thebibliography}{41}
\providecommand{\natexlab}[1]{#1}
\providecommand{\url}[1]{\texttt{#1}}
\expandafter\ifx\csname urlstyle\endcsname\relax
  \providecommand{\doi}[1]{doi: #1}\else
  \providecommand{\doi}{doi: \begingroup \urlstyle{rm}\Url}\fi

\bibitem[Absil and Oseledets(2015)]{Absil2015}
P.-A. Absil and I.~V. Oseledets.
\newblock Low-rank retractions: a survey and new results.
\newblock \emph{Computational Optimization and Applications}, 62\penalty0
  (1):\penalty0 5--29, Sep 2015.
\newblock ISSN 1573-2894.
\newblock \doi{10.1007/s10589-014-9714-4}.
\newblock URL \url{https://doi.org/10.1007/s10589-014-9714-4}.

\bibitem[Absil et~al.(2009)Absil, Mahony, and Sepulchre]{absil2009optimization}
P.A. Absil, R.~Mahony, and R.~Sepulchre.
\newblock \emph{Optimization Algorithms on Matrix Manifolds}.
\newblock Princeton University Press, 2009.
\newblock ISBN 9781400830244.
\newblock URL \url{https://books.google.com/books?id=NSQGQeLN3NcC}.

\bibitem[Aleta and Moreno(2019)]{Aleta_2019}
Alberto Aleta and Yamir Moreno.
\newblock Multilayer networks in a nutshell.
\newblock \emph{Annual Review of Condensed Matter Physics}, 10\penalty0
  (1):\penalty0 45--62, Mar 2019.
\newblock ISSN 1947-5462.
\newblock \doi{10.1146/annurev-conmatphys-031218-013259}.
\newblock URL \url{http://dx.doi.org/10.1146/annurev-conmatphys-031218-013259}.

\bibitem[Bhatia(1997)]{Bhatia1997}
Rajendra Bhatia.
\newblock \emph{Matrix Analysis}.
\newblock Number 169 in Graduate Texts in Mathematics. Springer, New York,
  1997.

\bibitem[Bhattacharjee et~al.(2018)Bhattacharjee, Banerjee, and
  Michailidis]{bhattacharjee2018change}
Monika Bhattacharjee, Moulinath Banerjee, and George Michailidis.
\newblock Change point estimation in a dynamic stochastic block model.
\newblock \emph{ArXiv:1812.03090}, 2018.

\bibitem[Bhattacharyya and Chatterjee(2020)]{bhattacharyya2020general}
Sharmodeep Bhattacharyya and Shirshendu Chatterjee.
\newblock General community detection with optimal recovery conditions for
  multi-relational sparse networks with dependent layers, 2020.

\bibitem[Boothby and Boothby(2003)]{boothby2003introduction}
W.M. Boothby and W.M. Boothby.
\newblock \emph{An Introduction to Differentiable Manifolds and Riemannian
  Geometry, Revised}.
\newblock Pure and Applied Mathematics. Elsevier Science, 2003.
\newblock ISBN 9780121160517.
\newblock URL \url{https://books.google.com/books?id=DFYs99E-IFYC}.

\bibitem[Brodka et~al.(2018)Brodka, Chmiel, Magnani, and
  Ragozini]{doi:10.1098/rsos.171747}
Piotr Brodka, Anna Chmiel, Matteo Magnani, and Giancarlo Ragozini.
\newblock Quantifying layer similarity in multiplex networks: a systematic
  study.
\newblock \emph{Royal Society Open Science}, 5\penalty0 (8):\penalty0 171747,
  2018.
\newblock \doi{10.1098/rsos.171747}.
\newblock URL
  \url{https://royalsocietypublishing.org/doi/abs/10.1098/rsos.171747}.

\bibitem[Buckner and DiNicola(2019)]{Buckner2019TheBD}
Randy~L. Buckner and Lauren~M. DiNicola.
\newblock The brain’s default network: updated anatomy, physiology and
  evolving insights.
\newblock \emph{Nature Reviews Neuroscience}, pages 1--16, 2019.

\bibitem[Chen et~al.(2016)Chen, Zhang, Gao, Wee, Li, Shen, and the Alzheimer's
  Disease Neuroimaging~Initiative]{doi:10.1002/hbm.23240}
Xiaobo Chen, Han Zhang, Yue Gao, Chong-Yaw Wee, Gang Li, Dinggang Shen, and the
  Alzheimer's Disease Neuroimaging~Initiative.
\newblock High-order resting-state functional connectivity network for mci
  classification.
\newblock \emph{Human Brain Mapping}, 37\penalty0 (9):\penalty0 3282--3296,
  2016.
\newblock \doi{10.1002/hbm.23240}.
\newblock URL \url{https://onlinelibrary.wiley.com/doi/abs/10.1002/hbm.23240}.

\bibitem[Chi et~al.(2020)Chi, Gaines, Sun, Zhou, and Yang]{JMLR:v21:18-155}
Eric~C. Chi, Brian~J. Gaines, Will~Wei Sun, Hua Zhou, and Jian Yang.
\newblock Provable convex co-clustering of tensors.
\newblock \emph{Journal of Machine Learning Research}, 21\penalty0
  (214):\penalty0 1--58, 2020.
\newblock URL \url{http://jmlr.org/papers/v21/18-155.html}.

\bibitem[Durante et~al.(2017)Durante, Mukherjee, and Steorts]{JMLR:v18:16-391}
Daniele Durante, Nabanita Mukherjee, and Rebecca~C. Steorts.
\newblock Bayesian learning of dynamic multilayer networks.
\newblock \emph{Journal of Machine Learning Research}, 18\penalty0
  (43):\penalty0 1--29, 2017.
\newblock URL \url{http://jmlr.org/papers/v18/16-391.html}.

\bibitem[Edelman et~al.(1998)Edelman, Arias, and
  Smith]{doi:10.1137/S0895479895290954}
Alan Edelman, Tomás~A. Arias, and Steven~T. Smith.
\newblock The geometry of algorithms with orthogonality constraints.
\newblock \emph{SIAM Journal on Matrix Analysis and Applications}, 20\penalty0
  (2):\penalty0 303--353, 1998.
\newblock \doi{10.1137/S0895479895290954}.
\newblock URL \url{https://doi.org/10.1137/S0895479895290954}.

\bibitem[Gallier(2001)]{Gallier2001}
Jean Gallier.
\newblock \emph{Basics of Classical Lie Groups: The Exponential Map, Lie
  Groups, and Lie Algebras}, pages 367--414.
\newblock Springer New York, New York, NY, 2001.
\newblock ISBN 978-1-4613-0137-0.
\newblock \doi{10.1007/978-1-4613-0137-0_14}.
\newblock URL \url{https://doi.org/10.1007/978-1-4613-0137-0_14}.

\bibitem[Gangrade et~al.(2018)Gangrade, Venkatesh, Nazer, and
  Saligrama]{gangrade2018testing}
Aditya Gangrade, Praveen Venkatesh, Bobak Nazer, and Venkatesh Saligrama.
\newblock Testing changes in communities for the stochastic block model.
\newblock \emph{ArXiv:1812.00769}, 2018.

\bibitem[Gower and Dijksterhuis(2004)]{oro2736}
John~C. Gower and Garmt~B. Dijksterhuis.
\newblock \emph{Procrustes problems}, volume~30 of \emph{Oxford Statistical
  Science Series}.
\newblock Oxford University Press, Oxford, UK, January 2004.
\newblock URL \url{http://oro.open.ac.uk/2736/}.

\bibitem[Han et~al.(2021)Han, Luo, Wang, and Zhang]{han2021exact}
Rungang Han, Yuetian Luo, Miaoyan Wang, and Anru~R. Zhang.
\newblock Exact clustering in tensor block model: Statistical optimality and
  computational limit, 2021.

\bibitem[Han and Dunson(2018)]{han2018multiresolution}
Shaobo Han and David~B. Dunson.
\newblock Multiresolution tensor decomposition for multiple spatial passing
  networks.
\newblock \emph{ArXiv:1803.01203}, 2018.

\bibitem[Hu and Xu(2003)]{10.1007/978-3-540-45080-1_27}
Xuelei Hu and Lei Xu.
\newblock A comparative study of several cluster number selection criteria.
\newblock In Jiming Liu, Yiu-ming Cheung, and Hujun Yin, editors,
  \emph{Intelligent Data Engineering and Automated Learning}, pages 195--202,
  Berlin, Heidelberg, 2003. Springer Berlin Heidelberg.
\newblock ISBN 978-3-540-45080-1.

\bibitem[Jing et~al.(2020)Jing, Li, Lyu, and Xia]{jing2020community}
Bing-Yi Jing, Ting Li, Zhongyuan Lyu, and Dong Xia.
\newblock Community detection on mixture multi-layer networks via regularized
  tensor decomposition, 2020.

\bibitem[Kao and Porter(2017)]{Kao_2017}
Ta-Chu Kao and Mason~A. Porter.
\newblock Layer communities in multiplex networks.
\newblock \emph{Journal of Statistical Physics}, 173\penalty0 (3-4):\penalty0
  1286--1302, Aug 2017.
\newblock ISSN 1572-9613.
\newblock \doi{10.1007/s10955-017-1858-z}.
\newblock URL \url{http://dx.doi.org/10.1007/s10955-017-1858-z}.

\bibitem[Kivela et~al.(2014)Kivela, Arenas, Barthelemy, Gleeson, Moreno, and
  Porter]{10.1093/comnet/cnu016}
Mikko Kivela, Alex Arenas, Marc Barthelemy, James~P. Gleeson, Yamir Moreno, and
  Mason~A. Porter.
\newblock {Multilayer networks}.
\newblock \emph{Journal of Complex Networks}, 2\penalty0 (3):\penalty0
  203--271, 07 2014.
\newblock ISSN 2051-1329.
\newblock \doi{10.1093/comnet/cnu016}.
\newblock URL \url{https://doi.org/10.1093/comnet/cnu016}.

\bibitem[Kolda and Bader(2009)]{Kolda09tensordecompositions}
Tamara~G. Kolda and Brett~W. Bader.
\newblock Tensor decompositions and applications.
\newblock \emph{SIAM REVIEW}, 51\penalty0 (3):\penalty0 455--500, 2009.

\bibitem[Kumar et~al.(2004)Kumar, Sabharwal, and Sen]{1366265}
A.~Kumar, Y.~Sabharwal, and S.~Sen.
\newblock A simple linear time (1 + epsiv;)-approximation algorithm for k-means
  clustering in any dimensions.
\newblock In \emph{45th Annual IEEE Symposium on Foundations of Computer
  Science}, pages 454--462, Oct 2004.
\newblock \doi{10.1109/FOCS.2004.7}.

\bibitem[Le and Levina(2015)]{Le2015}
Can~M. Le and Elizaveta Levina.
\newblock {Estimating the number of communities in networks by spectral
  methods}.
\newblock jul 2015.
\newblock URL \url{http://arxiv.org/abs/1507.00827}.

\bibitem[Lei(2020)]{lei2020tail}
Jing Lei.
\newblock Tail bounds for matrix quadratic forms and bias adjusted spectral
  clustering in multi-layer stochastic block models, 2020.

\bibitem[Lei and Rinaldo(2015)]{lei2015}
Jing Lei and Alessandro Rinaldo.
\newblock Consistency of spectral clustering in stochastic block models.
\newblock \emph{Ann. Statist.}, 43\penalty0 (1):\penalty0 215--237, 02 2015.
\newblock \doi{10.1214/14-AOS1274}.
\newblock URL \url{http://dx.doi.org/10.1214/14-AOS1274}.

\bibitem[Lei et~al.(2019)Lei, Chen, and Lynch]{10.1093/biomet/asz068}
Jing Lei, Kehui Chen, and Brian Lynch.
\newblock {Consistent community detection in multi-layer network data}.
\newblock \emph{Biometrika}, 107\penalty0 (1):\penalty0 61--73, 12 2019.
\newblock ISSN 0006-3444.
\newblock \doi{10.1093/biomet/asz068}.
\newblock URL \url{https://doi.org/10.1093/biomet/asz068}.

\bibitem[MacDonald et~al.(2021)MacDonald, Levina, and Zhu]{macdonald2021latent}
Peter~W. MacDonald, Elizaveta Levina, and Ji~Zhu.
\newblock Latent space models for multiplex networks with shared structure,
  2021.

\bibitem[Mercado et~al.(2018)Mercado, Gautier, Tudisco, and
  Hein]{mercado2018power}
Pedro Mercado, Antoine Gautier, Francesco Tudisco, and Matthias Hein.
\newblock The power mean laplacian for multilayer graph clustering.
\newblock \emph{ArXiv:1803.00491}, 2018.

\bibitem[Munsell et~al.(2015)Munsell, Wee, Keller, Weber, Elger, da~Silva,
  Nesland, Styner, Shen, and Bonilha]{munsell_2015}
B.C. Munsell, C.-Y. Wee, S.S. Keller, B.~Weber, C.~Elger, L.A.T. da~Silva,
  T.~Nesland, M.~Styner, D.~Shen, and L.~Bonilha.
\newblock Evaluation of machine learning algorithms for treatment outcome
  prediction in patients with epilepsy based on structural connectome data.
\newblock \emph{NeuroImage}, 118:\penalty0 219--230, 2015.

\bibitem[Paul and Chen(2016)]{paul2016}
Subhadeep Paul and Yuguo Chen.
\newblock Consistent community detection in multi-relational data through
  restricted multi-layer stochastic blockmodel.
\newblock \emph{Electron. J. Statist.}, 10\penalty0 (2):\penalty0 3807--3870,
  2016.
\newblock \doi{10.1214/16-EJS1211}.
\newblock URL \url{https://doi.org/10.1214/16-EJS1211}.

\bibitem[Paul and Chen(2020)]{paul2020}
Subhadeep Paul and Yuguo Chen.
\newblock Spectral and matrix factorization methods for consistent community
  detection in multi-layer networks.
\newblock \emph{Ann. Statist.}, 48\penalty0 (1):\penalty0 230--250, 02 2020.
\newblock \doi{10.1214/18-AOS1800}.
\newblock URL \url{https://doi.org/10.1214/18-AOS1800}.

\bibitem[Pensky and Zhang(2019)]{10.1214/19-EJS1533}
Marianna Pensky and Teng Zhang.
\newblock {Spectral clustering in the dynamic stochastic block model}.
\newblock \emph{Electronic Journal of Statistics}, 13\penalty0 (1):\penalty0
  678 -- 709, 2019.
\newblock \doi{10.1214/19-EJS1533}.
\newblock URL \url{https://doi.org/10.1214/19-EJS1533}.

\bibitem[Stam(2014)]{pub.1037745277}
Cornelis~J. Stam.
\newblock Modern network science of neurological disorders.
\newblock \emph{Nature Reviews Neuroscience}, 15\penalty0 (10):\penalty0
  683--695, 2014.
\newblock \doi{10.1038/nrn3801}.
\newblock URL
  \url{https://app.dimensions.ai/details/publication/pub.1037745277}.

\bibitem[Tibshirani et~al.(2001)Tibshirani, Walther, and
  Hastie]{doi.org/10.1111/1467-9868.00293}
Robert Tibshirani, Guenther Walther, and Trevor Hastie.
\newblock Estimating the number of clusters in a data set via the gap
  statistic.
\newblock \emph{Journal of the Royal Statistical Society: Series B (Statistical
  Methodology)}, 63\penalty0 (2):\penalty0 411--423, 2001.
\newblock \doi{https://doi.org/10.1111/1467-9868.00293}.
\newblock URL
  \url{https://rss.onlinelibrary.wiley.com/doi/abs/10.1111/1467-9868.00293}.

\bibitem[Wang et~al.(2017)Wang, Yu, and Rinaldo]{wang2017optimal}
Daren Wang, Yi~Yu, and Alessandro Rinaldo.
\newblock Optimal covariance change point localization in high dimension.
\newblock \emph{ArXiv:1712.09912}, 2017.

\bibitem[Wang(2010)]{10.1093/biomet/asq061}
Junhui Wang.
\newblock {Consistent selection of the number of clusters via crossvalidation}.
\newblock \emph{Biometrika}, 97\penalty0 (4):\penalty0 893--904, 12 2010.
\newblock ISSN 0006-3444.
\newblock \doi{10.1093/biomet/asq061}.
\newblock URL \url{https://doi.org/10.1093/biomet/asq061}.

\bibitem[Wang and Zeng(2019)]{NEURIPS2019_9be40cee}
Miaoyan Wang and Yuchen Zeng.
\newblock Multiway clustering via tensor block models.
\newblock In H.~Wallach, H.~Larochelle, A.~Beygelzimer, F.~Alch\'{e}-Buc,
  E.~Fox, and R.~Garnett, editors, \emph{Advances in Neural Information
  Processing Systems}, volume~32. Curran Associates, Inc., 2019.
\newblock URL
  \url{https://proceedings.neurips.cc/paper/2019/file/9be40cee5b0eee1462c82c6964087ff9-Paper.pdf}.

\bibitem[Zhang et~al.(2012)Zhang, Szlam, Wang, and Lerman]{Zhang2012}
Teng Zhang, Arthur Szlam, Yi~Wang, and Gilad Lerman.
\newblock Hybrid linear modeling via local best-fit flats.
\newblock \emph{International Journal of Computer Vision}, 100\penalty0
  (3):\penalty0 217--240, Dec 2012.
\newblock ISSN 1573-1405.
\newblock \doi{10.1007/s11263-012-0535-6}.
\newblock URL \url{https://doi.org/10.1007/s11263-012-0535-6}.

\bibitem[Zhou and Zhu(2019)]{zhou2019sparse}
Zhixin Zhou and Yizhe Zhu.
\newblock Sparse random tensors: concentration, regularization and
  applications, 2019.

\end{thebibliography}

\end{document}